\documentclass[final,onefignum,onetabnum,letterpaper]{siamonline190516}
\usepackage[scale=0.8]{geometry} 

\usepackage{soul}
\usepackage{lipsum}
\usepackage{amsfonts}
\usepackage{epstopdf}
\ifpdf
  \DeclareGraphicsExtensions{.eps,.pdf,.png,.jpg}
\else
  \DeclareGraphicsExtensions{.eps}
\fi

\usepackage{datetime}
\newdateformat{monthyeardate}{%
	\monthname[\THEMONTH] \THEDAY, \THEYEAR}

\usepackage{academicons}
\usepackage{xcolor}
\renewcommand{\orcid}[1]{\href{https://orcid.org/#1}{\textcolor[HTML]{A6CE39}{orcid.org/#1}}}

\usepackage{amsmath}
\allowdisplaybreaks
\usepackage{amssymb}
\usepackage{commath}
\usepackage{mathtools}
\usepackage{bbm}
\usepackage{bm}

\usepackage{color}
\usepackage{graphicx}
\usepackage[small]{caption}
\usepackage{subcaption}

\usepackage{relsize}
\usepackage{adjustbox}
\usepackage{algorithmic}
\usepackage{booktabs}
\usepackage{tikz}
\usepackage{pifont}
\usetikzlibrary{bayesnet} 

\usepackage{enumitem}
\setlist[enumerate]{leftmargin=.5in}
\setlist[itemize]{leftmargin=.5in}

\newsiamremark{remark}{Remark}
\newsiamremark{example}{Example}
\newsiamremark{hypothesis}{Hypothesis}
\crefname{hypothesis}{Hypothesis}{Hypotheses}
\newsiamthm{claim}{Claim}

\headers{Leveraging joint sparsity in hierarchical Bayesian learning}{J.\ Glaubitz and A.\ Gelb}

\title{Leveraging joint sparsity in hierarchical Bayesian learning
\thanks{\monthyeardate\today 
\corresponding{Jan Glaubitz} 
}}
\author{ 
Jan Glaubitz\thanks{Department of Aeronautics and Astronautics, MIT, Cambridge, MA 02139, USA (\email{glaubitz@mit.edu}, \orcid{0000-0002-3434-5563})}
\and 
Anne Gelb\thanks{Department of Mathematics, Dartmouth College, Hanover, NH 03755, USA (\email{Anne.E.Gelb@Dartmouth.edu}, \orcid{0000-0002-9219-4572})}
}

\usepackage{amsopn}

\usepackage{comment}

\usepackage[normalem]{ulem}

\usepackage[normalem]{ulem}

\newcommand{\MAP}{\mathrm{MAP}}

\DeclareMathOperator{\diag}{diag}

\DeclareMathOperator*{\argmin}{arg\,min}

\DeclareMathOperator*{\kernel}{kernel}

\newcommand{\intd}{\, \mathrm{d}} 

\newcommand{\N}{\mathbb{N}}

\newcommand{\R}{\mathbb{R}} 
\newcommand{\C}{\mathbb{C}}

\usepackage{lineno}
\usepackage{comment}

\ifpdf
\hypersetup{
	pdftitle={Glaubitz2023Leveraging},
 	pdfauthor={Jan Glaubitz}
}
\fi

\begin{document}

\maketitle


\begin{abstract}
We present a hierarchical Bayesian learning approach to infer jointly sparse parameter vectors from multiple measurement vectors. 
Our model uses separate conditionally Gaussian priors for each parameter vector and common gamma-distributed hyper-parameters to enforce joint sparsity. 
The resulting joint-sparsity-promoting priors are combined with existing Bayesian inference methods to generate a new family of algorithms. 
Our numerical experiments, which include a multi-coil magnetic resonance imaging application, demonstrate that our new approach consistently outperforms commonly used hierarchical Bayesian methods. 

\end{abstract}



\begin{Code}
    \url{https://github.com/jglaubitz/LeveragingJointSparsity}
\end{Code}

\begin{DOI}
    \url{https://doi.org/10.1137/23M156255X}
\end{DOI}

\section{Introduction} 
\label{sec:introduction} 

Parameter estimation from observable measurements is of fundamental importance in science and engineering applications.
Multiple measurement vectors (MMVs) can often be obtained from various sources, each having distinct underlying parameter vectors due to differences in spatial or temporal conditions \cite{cotter2005sparse,wipf2007empirical,adcock2019joint}. 
This situation can be modeled as a set of linear inverse problems given by
\begin{equation}\label{eq:MMV_IP} 
	\mathbf{y}_l = F_l \mathbf{x}_l + \mathbf{e}_l, \quad l=1,\dots,L,
\end{equation}  
where $\mathbf{y}_1,\dots,\mathbf{y}_L$ are the available MMVs, $\mathbf{x}_1,\dots,\mathbf{x}_L$ represent the sought-after parameter vectors, $F_{1},\dots,F_{L}$ are explicitly known linear forward operators, and $\mathbf{e}_1,\dots,\mathbf{e}_L$ denote the unknown noise component. 
The linear forward operators are often poorly conditioned, and the measurements may be limited in number or resolution, and contaminated by noise, causing the set of inverse problems \cref{eq:MMV_IP} to be ill-posed.

A well-known effective strategy used to mitigate ill-posedness is to incorporate prior information regarding the unknown parameter vectors, 
and in this study, we assume that these parameter vectors exhibit {\em joint sparsity}. 
Specifically, we assume there exists a linear operator $R$ (e.g., a discrete gradient or wavelet transform) such that $R \mathbf{x}_1,\dots,R \mathbf{x}_L$ are sparse and have common support. 
For example, the parameter vectors could correspond to piecewise constant signals with the same interior edge locations but different values. 
\cref{fig:deb_signal1_intro,fig:deb_signal2_intro} depict this scenario for the first two of four jointly sparse piecewise constant signals.
Joint sparsity arises in various applications, including signal processing, source location, neuro-electromagnetic imaging, parallel MRI, hyper-spectral imaging, and SAR imaging. 
For further reading on this topic, see \cite{cotter2005sparse,wipf2007empirical,adcock2019joint,zhang2022empirical} and related references.

\begin{figure}[tb]
	\centering
  	\begin{subfigure}[b]{0.45\textwidth}
		\includegraphics[width=\textwidth]{%
      		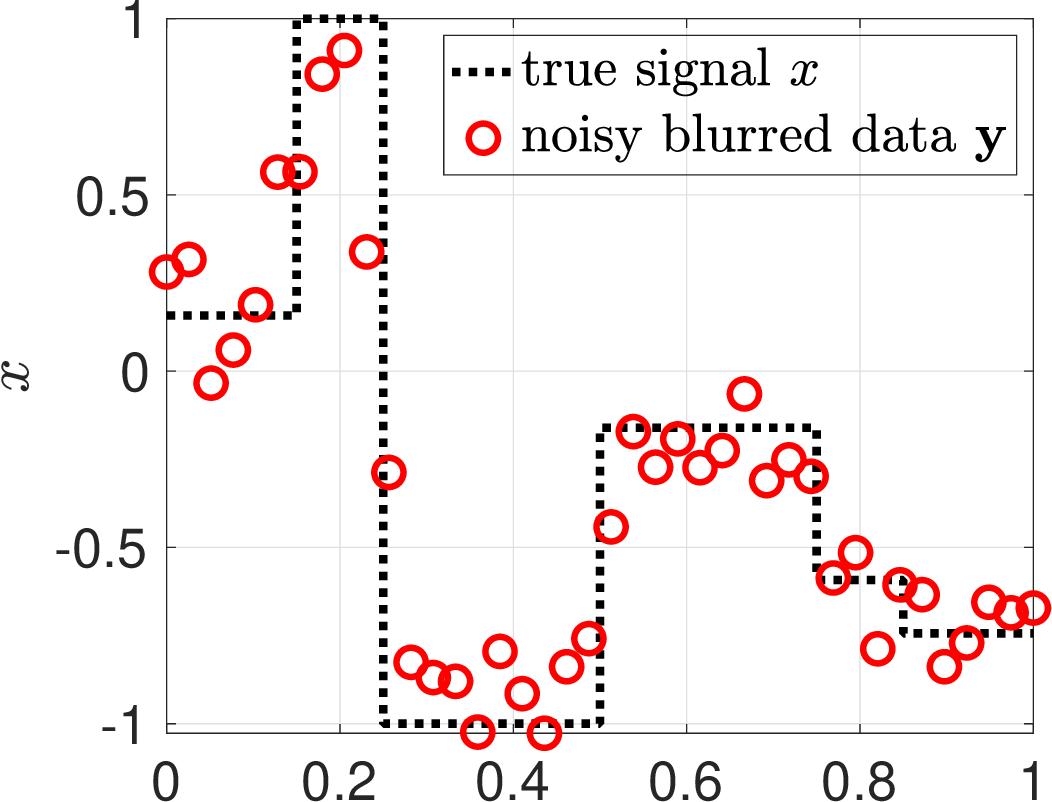} 
    	\caption{First signal and measurements}
    	\label{fig:deb_signal1_intro}
  	\end{subfigure}%
  	\begin{subfigure}[b]{0.45\textwidth}
		\includegraphics[width=\textwidth]{%
      		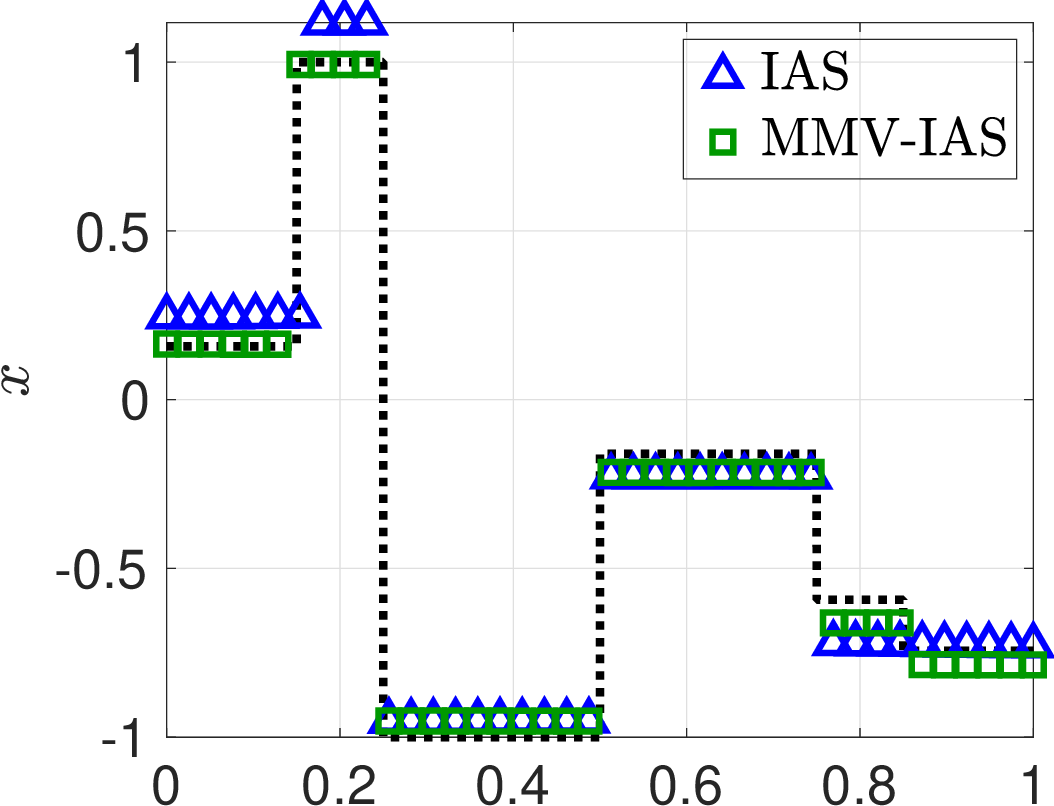} 
    	\caption{Recovered first signal}
    	\label{fig:deb_signal1_IAS_rm1_L4_intro}
  	\end{subfigure}%
	\\
	\begin{subfigure}[b]{0.45\textwidth}
		\includegraphics[width=\textwidth]{%
      		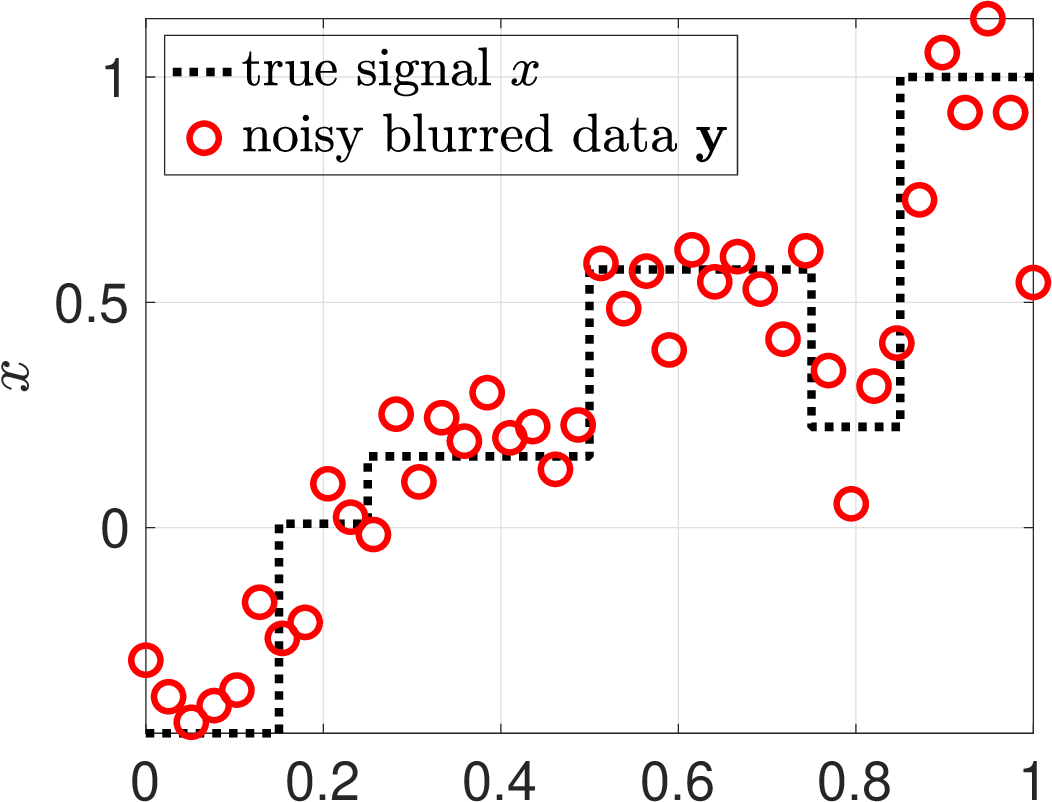} 
    	\caption{Second signal and measurements}
    	\label{fig:deb_signal2_intro}
  	\end{subfigure}%
  	\begin{subfigure}[b]{0.45\textwidth}
		\includegraphics[width=\textwidth]{%
      		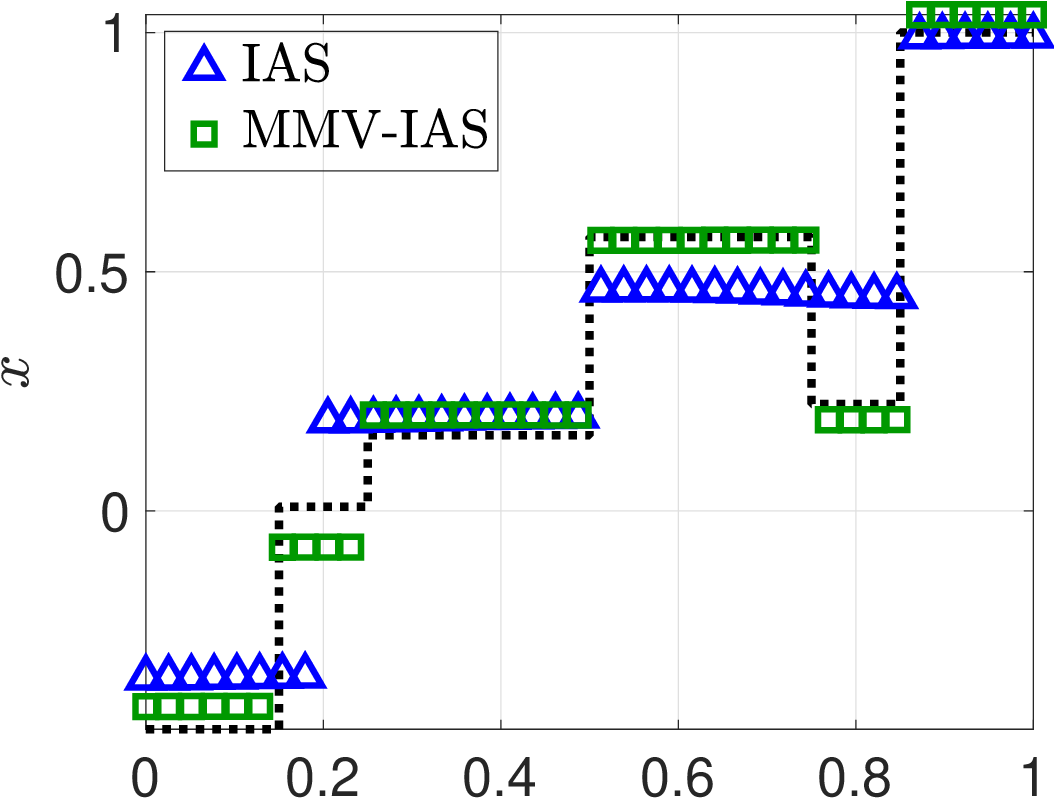} 
    	\caption{Recovered second signal}
    	\label{fig:deb_signal2_IAS_rm1_L4_intro}
  	\end{subfigure}%
  	\caption{ 
  	First column: The first two of four piecewise constant signals with a common edge profile and noisy blurred measurements. 
	Second column: Reconstructions of the signals using the existing IAS algorithm to separately recover them (blue triangles) and the proposed MMV-IAS algorithm to jointly recover them (green squares). 
	See \cref{sub:deblurring} for more details. 
  	}
  	\label{fig:deb_signal_intro}
\end{figure}  

\subsection*{Current methodology}

Various deterministic methods address the ill-posedness in the set of linear inverse problems \cref{eq:MMV_IP} by transforming it into a set of nearby regularized optimization problems. 
Under the joint sparsity assumption, established compressive sensing methods \cite{donoho2006compressed,eldar2012compressed,foucart2017mathematical} can be used to {\em individually} recover the desired parameter vectors. 
By leveraging their joint sparsity structure, the compressive sensing methods in \cite{cotter2005sparse,eldar2009robust,adcock2019joint}  {\em jointly} recover these vectors.
The approach in \cite{adcock2019joint} significantly enhanced the recovery process's robustness and accuracy. 
For more recent works in this area, see \cite{gelb2019reducing,scarnati2020accurate,xiao2022sequential,xiao2023sequential} and their references.

However, regularized inverse problems often face two significant challenges: (i) determining the appropriate regularization parameters and (ii) quantifying uncertainty in the recovered solution. 
Because the available measurements in \eqref{eq:MMV_IP} are often insufficient and noisy, it is essential to quantify the subsequent uncertainty in the parameters of interest. 
In particular, uncertainty in the parameters leads to uncertainty in predictions and decision-making. 

In this work we employ a hierarchical Bayesian approach \cite{kaipio2006statistical,calvetti2007introduction,stuart2010inverse} to solve the MMV inverse problem \cref{eq:MMV_IP}, 
with the parameters of interest and the measurements modeled as random variables. 
The sought-after posterior distribution for the parameters of interest is characterized using Bayes' theorem, which connects the posterior density to the prior and likelihood densities.
The prior encodes information available on the parameters of interest before any data are observed, while the likelihood density incorporates the data model and a stochastic description of measurements. 
A primary benefit of this framework is that it enables uncertainty quantification while avoiding the need for fine-tuning regularization parameters. 

One particularly effective class of priors for promoting sparsity is the conditional Gaussian prior. 
This choice has proven successful in various applications such as sparse basis selection \cite{tipping2001sparse,wipf2004sparse}, signal and image recovery \cite{chantas2006bayesian,babacan2010sparse,glaubitz2022generalized,xiao2022sequential}, and edge detection \cite{churchill2019detecting,xiao2023sequential2}. 
Additionally, conditional Gaussian priors are computationally convenient and lead to highly efficient inference algorithms, 
see \cite{calvetti2019hierachical,calvetti2020sparse,calvetti2020sparsity,vono2022high,glaubitz2022generalized} and references therein. 
Although existing sparsity-promoting hierarchical Bayesian algorithms can be used to infer the parameter vectors separately, such approaches do not exploit the \emph{joint sparsity} of the parameter vectors.

\subsection*{Our contribution}

We present a hierarchical Bayesian learning approach that leverages joint sparsity in multiple parameter vectors described by the MMV data model \cref{eq:MMV_IP}. 
Our approach utilizes separate priors for each parameter vector while sharing common hyper-parameters drawn from (generalized) gamma distributions, which capture the sparsity profile of the parameter vectors.
The advantage of using joint-sparsity-promoting priors is demonstrated in comparison to well-established sparsity-promoting algorithms, such as the generalized sparse Bayesian learning (GSBL) \cite{tipping2001sparse,wipf2004sparse,glaubitz2022generalized} and iterative alternating sequential (IAS) \cite{calvetti2019hierachical,calvetti2020sparse,calvetti2020sparsity} algorithms. 
Our results indicate that the proposed MMV-GSBL and MMV-IAS algorithms, which by design use joint-sparsity-promoting priors, outperform the existing methods.
\cref{fig:deb_signal_intro} compares the existing IAS algorithm with the proposed MMV-IAS algorithm to recover the first two out of four piecewise-constant signals with a common edge profile from noisy blurred data. 
Observe that the joint sparsity enhancement in either approach consistently results in superior signal recovery compared to reconstructing signals individually. 
In particular,  while the performance of the IAS and GSBL algorithms may vary depending on the specific problem and model parameters, incorporating joint sparsity consistently offers advantageous outcomes. 
Further numerical experiments demonstrate that the proposed method increases the robustness and accuracy of the recovered parameter vectors, better catches the sparsity profile encoded in the hyper-parameters, and reduces uncertainty. 
The findings of this research highlight the significant improvement in performance that can be achieved by exploiting joint sparsity in hierarchical Bayesian models. 
In particular, we demonstrate its potential application to parallel MRI. 
Additionally, we note that utilizing separate priors with shared hyper-parameters is not limited to conditionally Gaussian priors and that our approach can be adapted to other hierarchical prior models. 
Finally, our approach shares similarities with the one proposed in \cite{wipf2007empirical} for classical SBL that we discuss in \cref{sub:GSBL_connection}.

\subsection*{Outline} 

We present the joint-sparsity-promoting conditionally Gaussian priors and the resulting hierarchical Bayesian model in \Cref{sec:model}, with the Bayesian MAP estimation discussed in \Cref{sec:BI}. 
In \Cref{sec:analysis}, we analyze the new MMV-IAS algorithm and compare it to the existing IAS algorithm. 
\Cref{sec:GSBL} extends the idea of joint-sparsity-promoting priors to the GSBL framework to form the MMV-GSBL algorithm. 
Numerical experiments are showcased in \Cref{sec:numerics}, which include applications of the proposed MMV-IAS and -GSBL algorithms to parallel MRI. 
We summarize the work in \Cref{sec:summary}.

\subsection*{Notation} 

We use normal and boldface capital letters, such as $X$ and $\mathbf{X}$, to denote scalar- and vector-valued random variables, respectively. 
For a density $\pi$, we write $X \sim \pi$ when $X$ is distributed according to $\pi$. 
If $L \in \N$ and $\mathbf{X}_1,\dots,\mathbf{X}_L$ are random variables, then we denote their collection by $\mathbf{X}_{1:L} = (\mathbf{X}_1,\dots,\mathbf{X}_L)$. 
The same notation applies to dummy variables $\boldsymbol{x} \in \R^n$. 
 
\section{The joint hierarchical Bayesian model} 
\label{sec:model} 

We present the joint-sparsity promoting Bayesian model considered in this investigation. 
The conditionally Gaussian prior in \cref{sub:prior} is particularly important for developing our new method.

\subsection{The likelihood function} 
\label{sub:likelihood}

The likelihood density function models the connection between the parameter and measurement vectors. 
Consider the linear MMV data model \cref{eq:MMV_IP} with MMVs $\mathbf{y}_l \in \R^{M_l}$, known forward operators $F_l \in \R^{M_l \times N}$, desired parameter vectors $\mathbf{x}_l \in \R^{N}$, and additive Gaussian noise $\mathbf{e}_l \sim \mathcal{N}(\mathbf{0},\Sigma_l)$. 
Since $\Sigma_l$ is a symmetric positive definite (SPD) covariance matrix, there exists a Cholesky decomposition of the form $\Sigma_l = C_l C_l^T$ with invertible $C_l$. 
The noise can then be whitened by multiplying both sides of \cref{eq:MMV_IP} with $C_l^{-1}$ from the left-hand side so that we can assume $\Sigma_l = I$, where $I$ is the $M_l \times M_l$ identity matrix. 
The \emph{$l$th likelihood function} is then 
\begin{equation}\label{eq:likelihood} 
	\pi_{\mathbf{Y}_l | \mathbf{X}_l}( \mathbf{y}_l | \mathbf{x}_l ) 
        \propto \exp\left( -\frac{1}{2} \| F_l \mathbf{x}_l - \mathbf{y}_l \|_2^2 \right), 
        \quad l=1,\dots,L.
\end{equation} 
Assuming that $\mathbf{Y}_{1:L}$ are jointly independent conditioned on $\mathbf{X}_{1:L}$, the \emph{joint likelihood function} is   
\begin{equation}\label{eq:joint_likelihood} 
	\pi_{\mathbf{Y}_{1:L}|\mathbf{X}_{1:L}}(\mathbf{y}_{1:L}|\mathbf{x}_{1:L}) 
		= \prod_{l=1}^L \pi_{\mathbf{Y}_l | \mathbf{X}_l}( \mathbf{y}_l | \mathbf{x}_l ) 
		\propto \exp\left( -\frac{1}{2} \sum_{l=1}^L \| F_l \mathbf{x}_l - \mathbf{y}_l \|_2^2 \right).
\end{equation}
Note that \cref{eq:joint_likelihood} is a conditionally Gaussian density function, which is convenient for Bayesian inference.  A couple of remarks are in order.

\begin{remark}[Complex-valued forward operators] This framework also allows for complex-valued forward operators and observations. Specifically for $F \in \C^{M \times N}$, we can use the equivalent real-valued forward operator $[\operatorname{Re}(F);\operatorname{Im}(F)] \in \R^{2M \times N}$, where $\operatorname{Re}(F)$ and $\operatorname{Im}(F)$ denote the real and imaginary part of $F$.
\end{remark} 

\begin{remark}[Non-linear data models]
	For simplicity, we restrict our attention to linear data models. 
	However, the proposed approach can be extended to non-linear models using methods such as Kalman filtering \cite{evensen2009data,spantini2022coupling,kim2022hierarchical}. 
\end{remark}

\begin{remark}[Dependent measurement vectors]
	For simplicity, we have assumed that the MMVs $\mathbf{Y}_{1:L}$ are jointly independent conditioned on the parameter vectors $\mathbf{X}_{1:L}$. 
	This assumption facilitated the expression of the joint likelihood function as detailed in \cref{eq:joint_likelihood}. 
	However, it is important to recognize that in practical scenarios, the assumption of independence among measurement vectors might not hold due to inherent interdependencies. 
	To illustrate, consider the case where $\mathbf{y}_{1:L}|\mathbf{x}_{1:L} \sim \mathcal{N}( \mathbf{0} | \Sigma )$. 
	In this more general scenario, the joint likelihood function is 
	\begin{equation}
		\pi_{\mathbf{Y}_{1:L}|\mathbf{X}_{1:L}}(\mathbf{y}_{1:L}|\mathbf{x}_{1:L})
			\propto \exp\left( -\frac{1}{2} \norm{ \Sigma^{-1} \left( F \mathbf{x} - \mathbf{y} \right) }_2^2 \right),
	\end{equation}
	where $F = \diag( F_1, \dots, F_L )$, $\mathbf{x} = [\mathbf{x}_1, \dots, \mathbf{x}_L]^T$, and $\mathbf{y} = [\mathbf{y}_1, \dots, \mathbf{y}_L]^T$.
	We can again whiten the noise by computing the Cholesky decomposition $\Sigma = C C^T$ and multiplying $F$ and $\mathbf{y}$ by $C^{-1}$ from the left. 
	It is noteworthy that the resulting forward operator matrix might not maintain a block-diagonal form. 
	Consequently, the parameter vector updates, as discussed in \cref{sub:x_update}, no longer decouple. 
	In this case, the optimization problems in \cref{eq:x_update} transforms into 
	\begin{equation}\label{eq:x_update_general}
		\mathbf{x}_{1:L}
			= \argmin_{\mathbf{x}_1,\dots,\mathbf{x}L} \left\{ \, \norm{ F \mathbf{x} - \mathbf{y} }_2^2 + \sum_{l=1}^L  \| D{\boldsymbol{\theta}}^{-1/2} R \mathbf{x}_l \|_2^2 \, \right\}
	\end{equation}
	with $D_{\boldsymbol{\theta}} = \diag(\boldsymbol{\theta})$. 
	Notably, \cref{eq:x_update_general} poses a more computationally demanding problem compared to the original parallelizable optimization problems \cref{eq:x_update}.
\end{remark}

\subsection{The joint-sparsity promoting conditionally Gaussian prior}
\label{sub:prior}

The prior density models our prior belief about the desired parameter vectors $\mathbf{x}_{1:L}$. 
Here we assume that they are {\em jointly} sparse, i.e., there {exist a linear transform $R \in \R^{K \times N}$ such that $R \mathbf{x}_1, \dots, R \mathbf{x}_L$} are sparse and have the same support (the indices of their non-zero values are the same). 
We start by modeling the sparsity of $R \mathbf{x}_l$ in a probabilistic setting by choosing the \emph{$l$th prior} as the conditionally Gaussian density  
\begin{equation}\label{eq:prior_sparsity}
    \pi_{\mathbf{X}_l | \boldsymbol{\Theta}_l}( \mathbf{x}_l | \boldsymbol{\theta}_l ) 
        \propto \det( D_{\boldsymbol{\theta}_l} )^{-1/2} \exp\left( -\frac{1}{2} \| D_{\boldsymbol{\theta}_l}^{-1/2} {R} \mathbf{x}_l \|_2^2 \right), \quad l=1,\dots,L, 
\end{equation}
with hyper-parameter vector $\boldsymbol{\theta}_l = [(\theta_l)_1,\dots,(\theta_l)_K]$, covariance matrix $D_{\boldsymbol{\theta}_l} = \diag(\boldsymbol{\theta}_l)$, and unknown variance parameters $(\theta_l)_k > 0$ for $k=1,\dots,K$ and $l=1,\dots,L$.

\begin{remark}\label{rem:motivation_condGaussian}
	Following \cite{calvetti2007gaussian,glaubitz2022generalized}, the conditional Gaussian prior \cref{eq:prior_sparsity} can be motivated by its asymptotic behavior: 
	Assume that $(\theta_l)_1=\dots=(\theta_l)_K$, then \cref{eq:prior_sparsity} favors $\mathbf{x}_l$ for which $\| R \mathbf{x}_l \|_2$ is close to zero, since such an $\mathbf{x}_l$ has a higher probability. 
	For instance, when $R \mathbf{x}_l$ corresponds to the increments of $\mathbf{x}_l$, i.e., $\left[ R \mathbf{x}_l \right]_k = (x_l)_{k+1} - (x_l)_{k}$, then \cref{eq:prior_sparsity} with $(\theta_l)_1=\dots=(\theta_l)_K$ favors $\mathbf{x}_l$ to have little variation. 
	However, if one of the hyper-parameters, say $(\theta_l)_k$, is significantly larger than the others, a jump between $(x_l)_{k+1}$ and $(x_l)_{k}$ becomes more likely. 
	In this way, \cref{eq:prior_sparsity} promotes sparsity of $R \mathbf{x}_l$. 
	Furthermore, we can connect the support of $R \mathbf{x}_l$ to the hyper-parameters $(\theta_l)_1,\dots,(\theta_l)_K$. 
	In particular, we expect the support of $R \mathbf{x}_l$ to coincide with the hyper-parameters significantly larger than most others. 
\end{remark}

We next model $R \mathbf{x}_1,\dots,R \mathbf{x}_L$ having the same support. 
To this end, motivated by \cref{rem:motivation_condGaussian}, we connect the supports of $R \mathbf{x}_1,\dots,R \mathbf{x}_L$ to the hyper-parameter vectors, $\boldsymbol{\theta}_1,\dots,\boldsymbol{\theta}_L$, by assuming that $\boldsymbol{\theta}_1 = \dots = \boldsymbol{\theta}_L$. 
Denoting the common hyper-parameter vector as $\boldsymbol{\theta}$, \cref{eq:prior_sparsity} then reduces to  
\begin{equation}\label{eq:prior_joint_sparsity}
    \pi_{\mathbf{X}_l | \boldsymbol{\Theta}}( \mathbf{x}_l | \boldsymbol{\theta} ) 
        \propto \det( D_{\boldsymbol{\theta}} )^{-1/2} \exp\left( -\frac{1}{2} \| D_{\boldsymbol{\theta}}^{-1/2} R \mathbf{x}_l \|_2^2 \right), \quad l=1,\dots,L. 
\end{equation}
That is, the priors are now all conditioned on the {\em same} hyper-parameters. 
Finally, assuming the $\mathbf{X}_{1:L}$ are jointly independent conditioned on $\boldsymbol{\Theta}$, the \emph{joint prior} is 
\begin{equation}\label{eq:joint_prior} 
	\pi_{\mathbf{X}_{1:L} | \boldsymbol{\Theta}}( \mathbf{x}_{1:L} | \boldsymbol{\theta} ) 
		= \prod_{l=1}^L \pi_{\mathbf{X}_l | \boldsymbol{\Theta}}( \mathbf{x}_l | \boldsymbol{\theta} ) 
		\propto \det( D_{\boldsymbol{\theta}} )^{-L/2} \exp\left( -\frac{1}{2} \sum_{l=1}^L \| D_{\boldsymbol{\theta}}^{-1/2} R \mathbf{x}_l \|_2^2 \right). 
\end{equation}

\begin{remark}[Other sparsity-promoting hierarchical priors] 
	The conditionally Gaussian prior in \cref{eq:joint_prior} not only enforces joint sparsity but also enables convenient Bayesian inference due to its compatibility with the Gaussian likelihood \cref{eq:joint_likelihood}. 
	However, the joint-sparsity-promoting approach using a common hyper-parameter vector $\boldsymbol{\theta}$ can be extended to other hierarchical sparsity-promoting priors, such as horseshoe \cite{carvalho2009handling,uribe2022horseshoe} and neural network priors \cite{neal1996priors,asim2020invertible,li2021bayesian}. 
\end{remark}

\subsection{The generalized gamma hyper-prior} 
\label{sub:hyper}

The price to pay for the hierarchical joint prior model \cref{eq:joint_prior} is that we now need to estimate not only the parameter vectors $\mathbf{x}_{1:L}$ but also the common hyper-parameter vector $\boldsymbol{\theta}$. 
By Bayes' theorem, the \emph{joint posterior density} of $(\mathbf{X}_{1:L},\boldsymbol{\Theta})$ given $\mathbf{Y}_{1:L}$ is 
\begin{equation}\label{eq:posterior} 
	\pi_{ \mathbf{X}_{1:L}, \boldsymbol{\Theta} | \mathbf{Y}_{1:L} }( \mathbf{x}_{1:L}, \boldsymbol{\theta} | \mathbf{y}_{1:L} ) 
		\propto \pi_{ \mathbf{Y}_{1:L} | \mathbf{X}_{1:L} }( \mathbf{y}_{1:L} | \mathbf{x}_{1:L} ) \, 
			\pi_{ \mathbf{X}_{1:L} | \boldsymbol{\Theta} }( \mathbf{x}_{1:L} | \boldsymbol{\theta} ) \, 
			\pi_{ \boldsymbol{\Theta} }( \boldsymbol{\theta} ). 
\end{equation}
From \cref{rem:motivation_condGaussian}, it is evident that to promote sparsity of $R \mathbf{x}_1,\dots,R \mathbf{x}_L$ the hyper-prior $\pi_{ \boldsymbol{\Theta} }$ should favor small values of $\theta_1,\dots,\theta_K$ while allowing occasional large outliers for the conditionally Gaussian prior \cref{eq:prior_joint_sparsity}.
Following \cite{calvetti2020sparse,calvetti2020sparsity}, this can be achieved by treating $\theta_1,\dots,\theta_K$ as random variables with an uninformative generalized gamma density function
\begin{equation}\label{eq:hyper_priors} 
	\pi_{\boldsymbol{\Theta}}(\boldsymbol{\theta}) 
		= \prod_{k=1}^K \mathcal{GG}( \theta_k | r, \beta, \vartheta_k ) 
		\propto \det(D_{\boldsymbol{\theta}})^{r \beta - 1} \exp\left( - \sum_{k=1}^K ( \theta_k/\vartheta_k )^r \right). 
\end{equation} 
Here, $\mathcal{GG}$ is the generalized gamma distribution  
\begin{equation}\label{eq:pdf_gamma} 
\begin{aligned}
    \mathcal{GG}( \theta_k | r, \beta, \vartheta_k ) 
    		\propto \theta_k^{r \beta - 1} \exp\left( - ( \theta_k/\vartheta_k )^r \right),
\end{aligned}
\end{equation} 
where $r \in \R\setminus\{0\}$, $\beta > 0$, and $\vartheta_k > 0$ for $k=1,\dots,K$.
\Cref{fig:graphical_model} provides a graphical illustration and summary of our joint-sparsity-promoting hierarchical Bayesian model. 

\begin{figure}[tb]
\centering
\resizebox{0.4\textwidth}{!}{%
\begin{tikzpicture}
    \node[obs] (y1) {$\mathbf{y}_1$}; %
	\node[latent, right=1.5 of y1] (x1) {$\mathbf{x}_1$} ; %
	%
	\node[const, right=1.75 of x1] (theta_aux) {}; %
	\node[latent, below=.725 of theta_aux] (theta) {$\boldsymbol{\theta}$} ; %
	\node[obs, below=1.5 of y1] (y2) {$\mathbf{y}_2$}; %
	\node[latent, right=1.5 of y2] (x2) {$\mathbf{x}_2$} ; %
	%
	\edge {x1} {y1}; %
	\edge {x2} {y2}; %
	\edge {theta} {x1}; %
	\edge {theta} {x2}; %
	\plate[inner sep=0.3cm, xshift=0cm, yshift=0cm] {plate_y1} {(y1)} {$M_1$}; %
	\plate[inner sep=0.3cm, xshift=0cm, yshift=0cm] {plate_x} {(x1)} {$N$}; %
	\plate[inner sep=0.3cm, xshift=0cm, yshift=0cm] {plate_theta} {(theta)} {$K$}; %
	\plate[inner sep=0.3cm, xshift=0cm, yshift=0cm] {plate_y2} {(y2)} {$M_2$}; %
	\plate[inner sep=0.3cm, xshift=0cm, yshift=0cm] {plate_x} {(x2)} {$N$}; %
\end{tikzpicture} 
}%
\caption{
    Graphical representation of the hierarchical Bayesian model promoting joint sparsity for two ($L=2$) measurement and parameters vectors, $\mathbf{y}_1, \mathbf{y}_2$ and $\mathbf{x}_1, \mathbf{x}_2$, respectively. 
  	Shaded and plain circles represent observed and unobserved (hidden) random variables, respectively. 
	The arrows indicate how the random variables influence each other: 
	The parameter vectors $\mathbf{x}_1,\mathbf{x}_2$ are connected to the measurement vectors $\mathbf{y}_1,\mathbf{y}_2$, respectively, via the likelihood \cref{eq:joint_likelihood}; 
	The common hyper-parameters $\boldsymbol{\theta}$ are connected to $\mathbf{x}_1, \mathbf{x}_2$ via the joint-sparsity-promoting prior \cref{eq:joint_prior}. 
	Using common gamma hyper-parameters $\boldsymbol{\theta}$ (instead of separate ones for $\mathbf{x}_1, \mathbf{x}_2$) results in $R \mathbf{x}_1$ and $R \mathbf{x}_2$ having the same support. 
	}
\label{fig:graphical_model}
\end{figure}
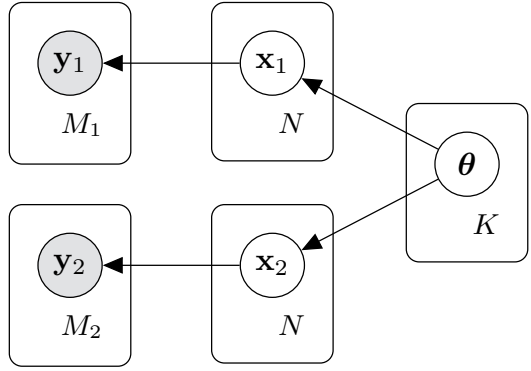

\section{Bayesian inference} 
\label{sec:BI} 

We now address Bayesian inference for the joint-sparsity-promoting hierarchical Bayesian model proposed in \Cref{sec:model}. 
To this end, 
for given MMVs $\mathbf{y}_{1:L}$, 
we solve for the \emph{maximum a posterior (MAP) estimate} $(\mathbf{x}_{1:L}^{\MAP},\boldsymbol{\theta}^{\MAP})$,
which is the maximizer of the posterior density \cref{eq:posterior}.
Equivalently, the MAP estimate is the minimizer of the negative logarithm of the posterior, i.e.,  
\begin{equation}\label{eq:MAP_estimate}
	(\mathbf{x}_{1:L}^{\MAP},\boldsymbol{\theta}^{\MAP}) 
		= \argmin_{ \mathbf{x}_{1:L}, \boldsymbol{\theta} } \left\{ \mathcal{G}( \mathbf{x}_{1:L}, \boldsymbol{\theta} \right\}, 
\end{equation} 
where the objective function $\mathcal{G}$ is 
\begin{equation}\label{eq:obj_fun} 
	\mathcal{G}( \mathbf{x}_{1:L}, \boldsymbol{\theta} ) 
		= - \log \pi_{ \mathbf{X}_{1:L}, \boldsymbol{\Theta} | \mathbf{Y}_{1:L} }( \mathbf{x}_{1:L}, \boldsymbol{\theta} | \mathbf{y}_{1:L} ).
\end{equation} 
Substituting \cref{eq:joint_likelihood,eq:joint_prior,eq:hyper_priors} into \cref{eq:obj_fun}, we obtain
\begin{equation}\label{eq:G}
\begin{aligned} 
	\mathcal{G}( \mathbf{x}_{1:L}, \boldsymbol{\theta} ) 
		= \frac{1}{2} \left( \sum_{l=1}^L \| F_l \mathbf{x}_l - \mathbf{y}_l \|_2^2 + \| D_{\boldsymbol{\theta}}^{-1/2} R \mathbf{x}_l  \|_2^2 \right) + \sum_{k=1}^K \left( \frac{\theta_k}{\vartheta_k} \right)^r - \eta \sum_{k=1}^K \log( \theta_k ) 
\end{aligned}	
\end{equation}
up to constants that neither depend on $\mathbf{x}_{1:L}$ nor $\boldsymbol{\theta}$, where $\eta = r \beta - (L/2 + 1)$. 
In what follows we discuss how the minimizer of $\mathcal{G}$ --- and therefore the MAP estimate $(\mathbf{x}_{1:L}^{\MAP},\boldsymbol{\theta}^{\MAP})$ --- can be approximated.

\subsection{The iterative alternating sequential algorithm} 
\label{sub:IAS}

We use a block-coordinate descent approach \cite{wright2015coordinate,beck2017first} to approximate the MAP estimate. 
In the context of conditionally Gaussian priors for which the covariance is assumed to follow a (generalized) gamma distribution, a prevalent block-coordinate descent method is the so-called iterative alternating sequential (IAS) algorithm \cite{calvetti2007gaussian,calvetti2015hierarchical,calvetti2019hierachical,calvetti2020sparse}.
The IAS algorithm computes the minimizer of the objective function $\mathcal{G}$ by alternatingly (i) minimizing $\mathcal{G}$ w.r.t.\ $\mathbf{x}_{1:L}$ for fixed $\boldsymbol{\theta}$ and (ii) minimizing $\mathcal{G}$ w.r.t.\ $\boldsymbol{\theta}$ for fixed $\mathbf{x}_{1:L}$. 
Given an initial guess for the hyper-parameter vector $\boldsymbol{\theta}$, the IAS algorithm proceeds through a sequence of updates of the form 
\begin{equation}\label{eq:IAS}
	\mathbf{x}_{1:L} = \argmin_{\mathbf{x}_{1:L}} \left\{ \mathcal{G}(\mathbf{x}_{1:L},\boldsymbol{\theta}) \right\}, \quad 
	\boldsymbol{\theta} = \argmin_{\boldsymbol{\theta}} \left\{ \mathcal{G}(\mathbf{x}_{1:L},\boldsymbol{\theta}) \right\},
\end{equation}
until a convergence criterion is met.\footnote{
In our implementation, we stop if the relative change in the $\mathbf{x}_{l}$ variables falls below a given threshold. 
For simplicity, we initialize the hyper-parameter vector as $\boldsymbol{\theta} = [1,\dots,1]$.
} 
Efficient implementation of the two update steps in \cref{eq:IAS} is discussed below.

\subsection{Updating the parameter vectors} 
\label{sub:x_update}

Updating $\mathbf{x}_{1:L}$ given $\boldsymbol{\theta}$ reduces to solving the quadratic optimization problems 
\begin{equation}\label{eq:x_update} 
	\mathbf{x}_l 
		= \argmin_{\mathbf{x}} \left\{ \| F_l \mathbf{x} - \mathbf{y}_l \|_2^2 + \| D_{\boldsymbol{\theta}}^{-1/2} R \mathbf{x} \|_2^2 \right\}, \quad 
		l=1,\dots,L,
\end{equation}
with $D_{\boldsymbol{\theta}} = \diag(\boldsymbol{\theta})$. 
Note that the optimization problems \cref{eq:x_update} are decoupled and can thus be solved efficiently in parallel. 
Furthermore, assuming the \emph{common kernel condition} (also see \cite{glaubitz2022generalized,xiao2023sequential})
\begin{equation}\label{eq:common_kernel}
    \kernel(F_l) \cap \kernel(R) = \{ \mathbf{0} \}, \quad l=1,\dots,L,
\end{equation} 
holds, each optimization problem in \cref{eq:x_update} has a unique solution. 
Here, $\kernel(G) = \{ \, \mathbf{x} \in \R^N \mid G \mathbf{x} = \mathbf{0} \, \}$ is the kernel of an operator $G: \R^N \to \R^M$, i.e., the set of vectors that are mapped to zero by $G$. 
The common kernel condition \cref{eq:common_kernel} guarantees that the combination of prior information and the given measurements will result in a well-posed problem, which is a commonly accepted assumption in regularized inverse problems \cite{kaipio2006statistical,tikhonov2013numerical}. 
Finally, we can efficiently solve the quadratic optimization problems \cref{eq:x_update} using various existing methods, including the fast iterative shrinkage-thresholding (FISTA) algorithm \cite{beck2009fast}, the preconditioned conjugate gradient (PCG) method \cite{saad2003iterative}, potentially combined with an early stopping based on Morozov's discrepancy principle \cite{calvetti2015hierarchical,calvetti2018bayes,calvetti2020sparse}, and the gradient descent approach \cite{glaubitz2022generalized}.  
There is no general advantage of one method over another, and the choice should be made based on the specific problem (and the structure of $F_l$ and $R$) at hand.

\subsection{Updating the hyper-parameters} 
\label{sub:beta_update}

We next address the update for the hyper-parameters $\boldsymbol{\theta}$, for which we must solve 
\begin{equation}\label{eq:update_beta1} 
	\boldsymbol{\theta} = \argmin_{\boldsymbol{\theta}} \left\{ \mathcal{G}(\mathbf{x}_{1:L},\boldsymbol{\theta}) \right\} 
\end{equation}
for fixed parameter vectors $\mathbf{x}_{1:L}$. 
Substituting \cref{eq:G} into \cref{eq:update_beta1} and ignoring all terms that do not depend on $\boldsymbol{\theta}$, \cref{eq:update_beta1} is equivalent to 
\begin{equation}\label{eq:update_beta2} 
	\theta_k = \argmin_{\theta_k} \left\{ \theta_k^{-1} \left( \sum_{l=1}^L [R \mathbf{x}_l]_k^2/2 \right) + \left( \frac{\theta_k}{\vartheta_k} \right)^r - \eta \log( \theta_k ) \right\}, 
\end{equation} 
where $\eta = r \beta - (L/2 + 1)$ and $[R \mathbf{x}_l]_k$ denotes the $k$-th entry of the vector $R \mathbf{x}_l \in \R^K$. 
Differentiating the objective function in \cref{eq:update_beta2} w.r.t.\ $\theta_k$ and setting this derivative to zero yields 
\begin{equation}\label{eq:update_beta3} 
	0 = - \theta_k^{-2} \left( \sum_{l=1}^L [R \mathbf{x}_l]_k^2/2 \right) + \theta_k^{r-1} \left( \frac{r}{\vartheta_k^r} \right) - \theta_k^{-1} \eta.
\end{equation} 
For some values of $r$, \cref{eq:update_beta3} admits an analytical solution. 
For instance, if $r=1$, \cref{eq:update_beta3} is equivalent to 
\begin{subequations}\label{eq:update_beta4}
\begin{equation}\label{eq:update_beta_rp1} 
	\theta_k = \frac{ \vartheta_k }{2} \left( \eta + \sqrt{ \eta^2 + 2 \vartheta^{-1} \sum_{l=1}^L [R \mathbf{x}_l]_k^2 } \right), \quad k=1,\dots,K,
\end{equation}  
where $\eta = r \beta - (L/2 + 1)$, 
and respectively for $r = -1$, we have
\begin{equation}\label{eq:update_beta_rm1} 
	\theta_k = \frac{ \sum_{l=1}^L [R \mathbf{x}_l]_k^2/2 + \vartheta_k}{ - \eta }, \quad k=1,\dots,K.
\end{equation}
\end{subequations}
By assuming $\eta > 0$ in \cref{eq:update_beta_rp1} and $\eta < 0$ in \cref{eq:update_beta_rm1}, the hyper-parameters are ensured to be positive.
We refer to \cite{calvetti2020sparse,calvetti2020sparsity} for details on how \cref{eq:update_beta3} can be solved numerically in the general case.

\subsection{Proposed algorithm and its relationship to current methodology}
\label{sub:comparison_IAS} 

\cref{algo:MMV_IAS} summarizes the above procedure to approximate the MAP estimate $(\mathbf{x}_{1:L}^{\MAP},\boldsymbol{\theta}^{\MAP})$ of our joint-sparsity-promoting hierarchical Bayesian model proposed in \Cref{sec:model}. 
We will refer to this method as the \emph{MMV-IAS algorithm}.

\begin{algorithm}[h!]
\caption{The MMV-IAS algorithm}\label{algo:MMV_IAS} 
\begin{algorithmic}[1]
    \STATE{Choose model parameters $(r,\beta,\boldsymbol{\vartheta})$ and initialize $\boldsymbol{\theta}$} 
    \REPEAT
		\STATE{Update the parameter vectors $\mathbf{x}_{1:L}$ (in parallel) according to \cref{eq:x_update}}
		\STATE{Update the hyper-parameters $\boldsymbol{\theta}$ according to \cref{eq:update_beta4}} 
    \UNTIL{convergence or the maximum number of iterations is reached}
\end{algorithmic}
\end{algorithm}

\subsubsection*{Relationship to the IAS algorithm}

Our MMV-IAS algorithm builds upon the standard IAS algorithm \cite{calvetti2020sparse,calvetti2020sparsity} and reduces to it when handling a single measurement and parameter vector ($L=1$). 
However, it is important to note that using the standard IAS to recover each $\mathbf{x}_{1:L}$ separately from the MMVs $\mathbf{y}_{1:L}$ is \emph{not} equivalent to the MMV-IAS algorithm as it does not consider joint sparsity. 
Our numerical examples in \Cref{sec:numerics} demonstrate that this can lead to suboptimal results.

\begin{remark}[Extensions of the IAS algorithm]
	Several advancements to the IAS algorithm have recently been made. 
	In \cite{calvetti2020sparsity}, hybrid versions were proposed to balance convex and non-convex models to enhance sparsity while avoiding stopping at non-global minima. 
	In \cite{si2022path}, path-following methods were used to smoothly transition from convex to non-convex models. 
	Additionally, in \cite{kim2022hierarchical}, the IAS algorithm was generalized for non-linear data models with ensemble Kalman methods. 
	While beyond the scope of this current investigation, it would be beneficial to integrate our joint-sparsity-promoting approach with these advancements as deemed appropriate for a particular application.
\end{remark}

\subsubsection*{Relationship to iteratively re-weighted least squares} 

The update steps \cref{eq:update_beta4,eq:x_update} of \cref{algo:MMV_IAS} can be understood as an iteratively re-weighted least squares (IRLS) algorithm \cite{chartrand2008iteratively,daubechies2010iteratively} with automatic inter-signal coupling. 
IRLS algorithms aim to recover sparse signals by assigning individual weights to the components of $\mathbf{x}$ and updating these weights iteratively. 
This concept is also applied in iteratively re-weighted $\ell^1$-regularization methods \cite{candes2008enhancing}. 
The MMV-IAS framework provides a Bayesian interpretation for weighting strategies and can be used to tailor these weights based on statistical assumptions of the problem.

\subsubsection*{Relationship to group sparsity}

We next address a possible generalization of our joint-sparsity-promoting hierarchical Bayesian model, and in the process, reveal its connection to group-sparsity-promoting models, as discussed in \cite{calvetti2015hierarchical}. 
More precisely, we show that our joint-sparsity-promoting approach can be re-interpreted as a group-sparsity-promoting model, and generalizes the one in \cite{calvetti2015hierarchical} in several ways.
Observe that the joint prior \cref{eq:joint_prior} can be rewritten in product form as 
\begin{equation}\label{eq:group_prior1} 
	\pi_{\mathbf{X}_{1:L} | \boldsymbol{\Theta}}( \mathbf{x}_{1:L} | \boldsymbol{\theta} ) 
		\propto \prod_{k=1}^K \left\{ \theta_k^{-L/2} \exp\left( -\frac{1}{2 \theta_k} \sum_{l=1}^L  \left[ R \mathbf{x}_l \right]_k^2 \right) \right\}. 
\end{equation}
Furthermore, denoting by $[ R \mathbf{x}_{\bullet} ]_k = \left( [ R \mathbf{x}_{1} ]_k, \dots, [ R \mathbf{x}_{L} ]_k \right) \in \R^L$ the vector that contains the $k$th components of the vectors $R \mathbf{x}_1,\dots,R \mathbf{x}_L$, \cref{eq:group_prior1} becomes 
\begin{equation}\label{eq:group_prior2} 
	\pi_{\mathbf{X}_{1:L} | \boldsymbol{\Theta}}( \mathbf{x}_{1:L} | \boldsymbol{\theta} ) 
		\propto \prod_{k=1}^K \left\{ \theta_k^{-L/2} \exp\left( -\frac{1}{2 \theta_k} \| [ R \mathbf{x}_{\bullet} ]_k \|_2^2 \right) \right\}. 
\end{equation}
Note that, in combination with a suitable generalized gamma hyper-prior, \cref{eq:group_prior2} promotes only a few of the groups $[ R \mathbf{x}_{\bullet} ]_1, \dots, [ R \mathbf{x}_{\bullet} ]_K$ not to be zero, where the $k$th group $[ R \mathbf{x}_{\bullet} ]_k$ is zero if and only if $\| [ R \mathbf{x}_{\bullet} ]_k \|_2 = 0$.
This allows us to re-interpret the ``joint-sparsity-promoting" prior \cref{eq:joint_prior} as a ``group-sparsity-promoting" prior. 
Furthermore, observe that that is straightforward to replace $\| \cdot \|_2$ with other weighted norms. 
The product form \cref{eq:group_prior2} of the joint prior density implies that  
\begin{equation}\label{eq:group_prior3}
	[ R \mathbf{x}_{\bullet} ]_k \sim \mathcal{N}( \mathbf{0}, \theta_k I ), 
\end{equation} 
where $I \in \R^{L \times L}$ is the usual identity matrix.
Or, in other words, the $L$ components of the vector $[ R \mathbf{x}_{\bullet} ]_k \in \R^L$ are independent and identically distributed. 
However, we can relax this assumption to allow non-trivial correlations between the components to be modeled, which is necessary in some applications. 
For instance, see \cite{calvetti2015hierarchical}, which considered a hierarchical Bayesian model for the inverse problem of magnetoencephalography (MEG)---aiming at estimating electromagnetic cerebral activity from measurements of the magnetic fields outside the head. 
To this end, we can replace \cref{eq:group_prior3} by 
\begin{equation}\label{eq:group_prior4}
	[ R \mathbf{x}_{\bullet} ]_k \sim \mathcal{N}( \mathbf{0}, \theta_k C_K ), 
\end{equation}
where $C_k \in \R^{L \times L}$ is an arbitrary covariance matrix. 
In this case, the joint prior \cref{eq:group_prior2} becomes 
\begin{equation}\label{eq:group_prior5} 
	\pi_{\mathbf{X}_{1:L} | \boldsymbol{\Theta}}( \mathbf{x}_{1:L} | \boldsymbol{\theta} ) 
		\propto \prod_{k=1}^K \left\{ \theta_k^{-L/2} \exp\left( -\frac{1}{2 \theta_k} \| [ R \mathbf{x}_{\bullet} ]_k \|_{C_k}^2 \right) \right\} 
\end{equation}
with norm $\| \mathbf{b} \|_{C_k}^2 = \mathbf{b}^T C_k^{-1} \mathbf{b}$. 
The usual Euclidean norm $\| \cdot \|_2$ corresponds to the special case $C_k = I$. 
Moreover, the objective function $\mathcal{G}$ in \cref{eq:G}, which is the negative logarithm of the posterior, is then 
\begin{equation}\label{eq:group_G}
\begin{aligned} 
	\mathcal{G}( \mathbf{x}_{1:L}, \boldsymbol{\theta} ) 
		= \frac{1}{2} \sum_{l=1}^L \| F_l \mathbf{x}_l - \mathbf{y}_l \|_2^2 
			+ \frac{1}{2} \sum_{k=1}^K \theta_k^{-1} \| [ R \mathbf{x}_{\bullet} ]_k \|_{C_k}^2 
			+ \sum_{k=1}^K \left( \frac{\theta_k}{\vartheta_k} \right)^r 
			- \eta \sum_{k=1}^K \log( \theta_k ), 
\end{aligned}	
\end{equation}
up to constants that neither depend on $\mathbf{x}_{1:L}$ nor $\boldsymbol{\theta}$, where $\eta = r \beta - (L/2 + 1)$. 
While the $\mathbf{x}_{l}$-updates of the IAS algorithm in \cref{sub:IAS} are still the same as in \cref{sub:x_update}, the $\boldsymbol{\theta}$-update \cref{eq:update_beta2} in \cref{sub:beta_update} transforms into 
\begin{equation}\label{eq:group_theta} 
	\theta_k = \argmin_{\theta_k} \left\{ \frac{1}{2 \theta_k} \| [ R \mathbf{x}_{\bullet} ]_k \|_{C_k}^2 + \left( \frac{\theta_k}{\vartheta_k} \right)^r - \eta \log( \theta_k ) \right\}. 
\end{equation} 
Hence
we recover the hierarchical Bayesian model and the update rules in \cite{calvetti2015hierarchical} as the special case of $L=3$, $R = I$, and a usual gamma hyper-prior ($r=1$).

\subsection{Uncertainty quantification}
\label{sub:UQ_IAS} 

Although we only solve for the MAP estimate of the posterior density $\pi_{ \mathbf{X}_{1:L}, \boldsymbol{\Theta} | \mathbf{Y}_{1:L} = \mathbf{y}_{1:L} }$ for given MMV data $\mathbf{y}_{1:L}$, we can still partially quantify uncertainty in the recovered parameter vectors. 
Specifically, for fixed hyper-parameters $\boldsymbol{\theta}$, Bayes' theorem yields 
\begin{equation}\label{eq:posterior_x}
	\pi_{ \mathbf{X}_{1:L} | \boldsymbol{\Theta} = \boldsymbol{\theta}, \mathbf{Y}_{1:L} = \mathbf{y}_{1:L} }( \mathbf{x}_{1:L} )
		\propto \pi_{ \mathbf{Y}_{1:L} | \mathbf{X}_{1:L}}(  \mathbf{y}_{1:L} | \mathbf{x}_{1:L} ) \, \pi_{ \mathbf{X}_{1:L} | \boldsymbol{\Theta}}( \mathbf{x}_{1:L} | \boldsymbol{\theta} )
\end{equation}
for the fully conditional posterior for the parameter vectors $\mathbf{X}_{1:L}$. 
Here, $\pi_{ \mathbf{Y}_{1:L} | \mathbf{X}_{1:L}}$ is the likelihood density \cref{eq:joint_likelihood} and $\pi_{ \mathbf{X}_{1:L} | \boldsymbol{\Theta}}$ is the prior \cref{eq:joint_prior}.
Substituting \cref{eq:joint_likelihood,eq:joint_prior} into \cref{eq:posterior_x} yields 
\begin{equation}\label{eq:posterior_x2} 
	\pi_{ \mathbf{X}_{1:L} | \boldsymbol{\Theta} = \boldsymbol{\theta}, \mathbf{Y}_{1:L} = \mathbf{y}_{1:L} }( \mathbf{x}_{1:L} ) 
		\propto \exp\left( -\frac{1}{2} \sum_{l=1}^L \| F_l \mathbf{x}_l - \mathbf{y}_l \|_2^2 + \| D_{\boldsymbol{\theta}}^{-1/2} R \mathbf{x}_l \|_2^2 \right).
\end{equation}
For covariance matrix $\Gamma_l = ( F_l^T F_l + {R^T D_{\boldsymbol{\theta}}^{-1} R} )^{-1}$ and mean $\boldsymbol{\mu}_l = \Gamma_l F_l^T \mathbf{y}_l$, we therefore have 
\begin{equation}\label{eq:posterior_x3} 
	\pi_{ \mathbf{X}_{l} | \boldsymbol{\Theta} = \boldsymbol{\theta}, \mathbf{Y}_{l} = \mathbf{y}_{l} }( \mathbf{x}_{l} ) 
		\propto \mathcal{N}( \mathbf{x}_l | \boldsymbol{\mu}_l, \Gamma_l ).
\end{equation}
The common kernel condition \cref{eq:common_kernel} ensures that $\Gamma_l$ is an SPD covariance matrix. 
We can now quantify uncertainty in $\mathbf{X}_l$ by sampling from the normal distribution and subsequently determining, for instance, the sample mean and credible intervals. 
We refer to \cite{vono2022high} for more details on sampling from high-dimensional Gaussian distributions. 

\begin{remark}[Full posterior sampling]
	Although our approach quantifies uncertainty in the parameter vectors, it does not account for the uncertainty in the hyper-parameters. 
	To fully address uncertainty, Bayesian MAP estimation should be replaced with, for instance, full posterior sampling using a Markov chain Monte Carlo (MCMC) method. 
	The goal of MCMC is to compute realizations of a Markov chain that is stationary w.r.t.\ the posterior distribution \cite{owen2013monte}. 
	However, sampling from sparsity-promoting hierarchical models is challenging since 
	(1) they are high-dimensional, which leads to high `per sample’ costs; 
	(2) they can have multiple modes separated by regions of low density, which are challenging to traverse; and 
	(3) there is a strong correlation between the parameters of primary interest and the hyper-parameters, resulting in poor mixing and slow convergence. 
	Recent research, specifically in \cite{calvetti2023computationally}, has made strides in addressing the issue of high 'per sample' costs. 
	This was achieved through a re-parameterization that converts the posterior into a form dominated by white Gaussian noise. 
	Following this transformation, the preconditioned Crank--Nicholson (pCN) scheme was employed to efficiently sample from the transformed posterior. 
	Nonetheless, challenges (2) and (3), concerning mode traversal and parameter correlation, respectively, remain unresolved. 
	Extensive research is needed to address such challenges and is beyond the scope of this investigation.
\end{remark}
 
\section{Analysis: Complexity, convexity, and convergence} 
\label{sec:analysis} 

We briefly analyze the computational complexity, convexity, and convergence of the MMV-IAS algorithm (\cref{algo:MMV_IAS}) and the underlying objective function.

\subsection{Computational complexity}
\label{sub:complexity}

Consider $L$ parameter vectors $\mathbf{x}_1,\dots,\mathbf{x}_L \in \R^N$. 
Different methods can be used to solve the $\mathbf{x}_l$-updates in \cref{algo:MMV_IAS}. 
Assuming that we use the PCG method, each $\mathbf{x}_l$-update has computational complexity $\mathcal{O}(\tilde{N}_l)$, where $\tilde{N}_l$ is the number of non-zero elements of the matrix $F_l^T F_l + R^T D_{\boldsymbol{\theta}}^{-1/2} R$. 
In the worst case ($\tilde{N}_l = N^2$ for all $l=1,\dots,L$), the computational cost for updating all parameter vectors is $\mathcal{O}(L N^2)$. 
As already noted, however, these updates can be performed in parallel. 
Furthermore, we perform the $\boldsymbol{\theta}$-update in \cref{algo:MMV_IAS} using one of the explicit formulas \cref{eq:update_beta4}. 
If $R \in \R^{K \times N}$ for $l=1,\dots,L$, then the $\boldsymbol{\theta}$-update has computational complexity $\mathcal{O}(K N)$. 
See \cite[Section 4.1]{glaubitz2022generalized} for more details.
In sum, if we run \cref{algo:MMV_IAS} for $I$ iterations, then its overall order of operations is (at most) $\mathcal{O}(I (L N^2 + K N) )$.

\subsection{Convexity of the objective function}
\label{sub:convexity} 

We next investigate the convexity of the objective function $\mathcal{G}$ in \cref{eq:G}.  
\cref{thm:convexity} provides the choices of hyper-parameters $(r, \beta, \vartheta_{1:K})$ for which the objective function $\mathcal{G}$ is globally or locally convex, that is, when the convexity is restricted to specific values for $\boldsymbol{\theta}$. It also describes how the number of MMVs influences convexity.

\begin{theorem}[Convexity of the objective function]\label{thm:convexity}
	Let $\mathcal{G}$ be the objective function in \cref{eq:G} and $\eta = r \beta - (L/2 + 1)$. 
	\begin{enumerate}
		\item[(a)] 
		If $r \geq 1$ and $\eta > 0$, then $\mathcal{G}$ is globally convex. 
		
		\item[(b)]
		If $0<r<1$ and $\eta > 0$, or if $r < 0$, then $\mathcal{G}$ is convex provided that 
		\begin{equation}\label{eq:convexity_cond}
			\theta_k < \vartheta_k \left( \frac{\eta}{r |r-1|} \right)^{1/r}, 
				\quad k=1,\dots,K. 
		\end{equation}
		
	\end{enumerate}
\end{theorem}

\cref{thm:convexity} highlights the impact of MMV data on the convexity of the objective function $\mathcal{G}$ in our joint-sparsity-promoting hierarchical Bayesian model. 
In particular, as $L$ increases, the linear decrease of $\eta$ results in the condition $\eta > 0$ becoming more restrictive. 
Furthermore, since  $\eta$ decreases as $L$ increases, we see in part (b) that the right-hand side of \cref{eq:convexity_cond} is also smaller.
That is, the convex set in which $G$ is convex shrinks as $L$ increases, revealing a trade-off between promoting joint sparsity and decreasing convexity as the number of coupled MMV data and parameter vectors increases.
\cref{thm:convexity} is a natural extension of \cite[Theorem 4.1]{calvetti2020sparse} (also see \cite[Theorem 3.1]{calvetti2020sparsity}), which we recover as the special case of $L=1$ (and $R = I$). 
The proof is provided in \cref{app:convexity_proof}.

\begin{remark} [Convergence of MMV-IAS]\label{rem:convergence} 
	The findings in \cref{thm:convexity} impact the IAS algorithm's performance. 
	When the objective function $\mathcal{G}$ exhibits global convexity, the MMV-IAS algorithm is guaranteed to converge to the unique minimum of $\mathcal{G}$. 
	This follows from the standard theory for coordinate descent approaches \cite{wright2015coordinate,beck2017first}. 
	Although global convexity streamlines the computation of the MAP estimate, there is a strong justification for investigating alternative choices of $r$ that yield hierarchical priors with enhanced sparsity-promoting properties. 
	Empirical studies suggest that when the objective function is not globally convexity, the resulting sparsity of the minimizer is often increased. 
	Nevertheless, a non-convex $\mathcal{G}$ can lead to the emergence of misleading local minima, potentially causing the MMV-IAS algorithm to become entrapped in one. 
	To overcome the risk of the algorithm prematurely converging to an incorrect local minimum, \cite{calvetti2020sparsity} recommends the adoption of hybrid versions of the IAS algorithm in the single-measurement-vector case. 
	In such constructions, the global convergence traits of gamma hyper-priors ($r=1$) are leveraged initially to close in on the vicinity of the unique global minimum, after which there is a shift to a generalized gamma hyper-prior with $r < 1$ to provide a stronger sparsity stimulus.
\end{remark} 
\section{Extension to generalized sparse Bayesian learning} 
\label{sec:GSBL}

We now briefly demonstrate how our method for fostering joint sparsity can be incorporated into the GSBL framework \cite{tipping2001sparse,wipf2004sparse,glaubitz2022generalized}. 
Sparse Bayesian learning (SBL), first introduced in \cite{tipping2001sparse}, is a statistical approach that employs Bayesian inference to recover sparse solutions from indirect, incomplete, and noisy data. 
This technique is characterized by combining a conditional zero-mean Gaussian prior with a gamma hyper-prior for the precision of the Gaussian prior. 
Traditional SBL methods have predominantly operated under the sparsity assumption in the parameter vector $\mathbf{x}$. 
However, the recent work \cite{glaubitz2022generalized} broadened this framework by proposing that sparsity can also apply to some linear transformation of the parameter vector, denoted as $R \mathbf{x}$. 
Here, $R$ is permitted to possess a non-trivial kernel, provided the common kernel condition $\ker(F) \cap \ker(R) = {\mathbf{0}}$ holds. 
This extension led to the development of the GSBL approach. 
Within this context, we now further evolve the GSBL method to encourage \emph{joint} sparsity in the case of MMVs corresponding to jointly sparse parameter vectors. 

As highlighted in the introduction, both IAS and GSBL are established cases of sparsity-promoting algorithms that can benefit from our joint sparsity-promoting priors in the presence of MMV data. 
Importantly, these priors are not restricted to the discussed algorithms, as they can also be employed to enhance the performance of other sparsity-promoting MAP estimators, demonstrating their versatile applicability.

\subsection{The hierarchical Bayesian model}
\label{sub:GSBL_model}

The main difference between the hierarchical model discussed in \Cref{sec:model} and the one underlying GSBL is that the latter treats the diagonal entries of the precision (inverse covariance) matrix $D_{\boldsymbol{\theta}}$ as gamma distributed random variables. 
In this case, the joint prior is 
\begin{equation}\label{eq:joint_prior_GSBL} 
	\pi_{\mathbf{X}_{1:L} | \boldsymbol{\Theta}}( \mathbf{x}_{1:L} | \boldsymbol{\theta} ) 
		\propto \det( D_{\boldsymbol{\theta}} )^{L/2} \exp\left( -\frac{1}{2} \sum_{l=1}^L \| D_{\boldsymbol{\theta}}^{1/2} R \mathbf{x}_l \|_2^2 \right) 
\end{equation}
rather than \cref{eq:joint_prior} and the gamma hyper-prior is 
\begin{equation}\label{eq:hyper_priors_GSBL} 
	\pi_{\boldsymbol{\Theta}}(\boldsymbol{\theta}) 
		= \prod_{k=1}^K \mathcal{GG}( \theta_k | 1, \beta, \vartheta_k ) 
		\propto \det(D_{\boldsymbol{\theta}})^{\beta - 1} \exp\left( - \sum_{k=1}^K \theta_k/\vartheta_k \right) 
\end{equation}
rather than \cref{eq:hyper_priors}.
We still assume the joint likelihood function \cref{eq:joint_likelihood}.

\subsection{Bayesian inference}
\label{sub:GSBL_inference}

We perform Bayesian inference for the GSBL model by again solving for the MAP estimate of its posterior. 
To this end, a block-coordinate descent approach similar to the IAS algorithm was recently investigated in \cite{glaubitz2022generalized} (also see \cite{xiao2023sequential}).
For the GSBL model above, the objective function $\mathcal{G}$ that is minimized by the MAP estimate is 
\begin{equation}\label{eq:G_GSBL}
\begin{aligned} 
	\mathcal{G}( \mathbf{x}_{1:L}, \boldsymbol{\theta} ) 
		= \frac{1}{2} \left( \sum_{l=1}^L \| F_l \mathbf{x}_l - \mathbf{y}_l \|_2^2 
			+ \| D_{\boldsymbol{\theta}}^{1/2} R \mathbf{x}_l  \|_2^2 \right) 
			+ \sum_{k=1}^K \frac{\theta_k}{\vartheta_k} 
			+ ( -L/2 + 1 - \beta ) \sum_{k=1}^K \log( \theta_k ) 
\end{aligned}	
\end{equation}
up to constants that neither depend on $\mathbf{x}_{1:L}$ nor $\boldsymbol{\theta}$. 
We again minimize $\mathcal{G}$ by alternatingly (i) minimizing $\mathcal{G}$ w.r.t.\ $\mathbf{x}_{1:L}$ for fixed $\boldsymbol{\theta}$ and (ii) minimizing $\mathcal{G}$ w.r.t.\ $\boldsymbol{\theta}$ for fixed $\mathbf{x}_{1:L}$.  
In the case of GSBL, updating $\mathbf{x}_{1:L}$ given $\boldsymbol{\theta}$ reduces to solving the quadratic optimization problems 
\begin{equation}\label{eq:x_update_GSBL} 
	\mathbf{x}_l 
		= \argmin_{\mathbf{x}} \left\{ \| F_l \mathbf{x} - \mathbf{y}_l \|_2^2 + \| D_{\boldsymbol{\theta}}^{1/2} R \mathbf{x} \|_2^2 \right\}, \quad 
		l=1,\dots,L.
\end{equation}
Moreover, using the same arguments as in \cref{sub:beta_update}, the minimizer of $\mathcal{G}$ w.r.t.\ $\boldsymbol{\theta}$ for fixed $\mathbf{x}_{1:L}$ is 
\begin{equation}\label{eq:update_theta_GSBL} 
	\theta_k = \frac{ L/2 - 1 + \beta }{ \sum_{l=1}^L [R \mathbf{x}_l]_k^2/2 + \vartheta_k^{-1} }, \quad k=1,\dots,K.
\end{equation}  
We refer to \cite{glaubitz2022generalized} for more details. 
\cref{algo:MMV_GSBL} summarizes the above procedure to approximate the MAP estimate of the joint-sparsity-promoting GSBL model above. 
Henceforth we refer to this method as the \emph{MMV-GSBL algorithm}.

\begin{algorithm}[h!]
\caption{The MMV-GSBL algorithm}\label{algo:MMV_GSBL} 
\begin{algorithmic}[1]
    \STATE{Choose model parameters $(\beta,\boldsymbol{\vartheta})$ and initialize $\boldsymbol{\theta}$} 
    \REPEAT
		\STATE{Update the parameter vectors $\mathbf{x}_{1:L}$ (in parallel) according to \cref{eq:x_update_GSBL}}
		\STATE{Update the hyper-parameters $\boldsymbol{\theta}$ according to \cref{eq:update_theta_GSBL}} 
    \UNTIL{convergence or the maximum number of iterations is reached}
\end{algorithmic}
\end{algorithm}

\begin{remark}[Uncertainty quantification]\label{rem:GSBL_UQ}
	We can partially quantify uncertainty in the parameter vectors recovered by the MMV-GSBL method described in \cref{algo:MMV_GSBL} by following the discussion in \cref{sub:UQ_IAS}. 
	The only difference to the MMV-IAS algorithm is that the covariance matrices $\Gamma_l$ in \cref{eq:posterior_x3} become $\Gamma_l = ( F_l^T F_l + R^T D_{\boldsymbol{\theta}} R )$ in the MMV-GSBL framework. 
\end{remark}

\subsection{Analysis}
\label{sub:GSBL_analysis} 

The analysis carried out for the MMV-IAS model and algorithm in \Cref{sec:analysis} can be extended to the MMV-GSBL model and algorithm. 
Both algorithms share a similar computational complexity of $\mathcal{O}(I (L N^2 + K N))$. 
However, as mentioned in previous studies  \cite{wipf2004sparse,glaubitz2022generalized}, the GSBL cost function can exhibit non-convexity with multiple local minima. 
This non-convexity is expected to persist in the MMV-GSBL cost function as well.

\subsection{Connection to existing methods} 
\label{sub:GSBL_connection}

Recovering jointly sparse signals from MMV data using SBL was considered in \cite{wipf2007empirical}. 
The MMV-SBL method proposed \cite{wipf2007empirical} has certain limitations, however, including restrictions on the forward operators, the noise distribution, and the requirement of sparse parameter vectors. 
Furthermore, the evidence approach used in \cite{wipf2007empirical} can slow performance for large problems. 
In contrast, the MMV-GSBL algorithm (\cref{algo:MMV_GSBL}) is more efficient and flexible. 
It allows for varying forward operators, different noise distributions, and more general regularization operators promoting sparsity in an arbitrary linear transformation of the parameter vectors, 
This makes the proposed MMV-GSBL algorithm suitable for various MMV problems and large-scale parameter vectors.

 \section{Numerical results} 
\label{sec:numerics} 

We conduct numerical experiments to showcase the effectiveness of our joint-sparsity-promoting MMV-IAS and MMV-GSBL algorithms, detailed in \cref{algo:MMV_IAS,algo:MMV_GSBL}.
For a fair comparison, we also evaluate the individual signal recovery performance using the traditional IAS and GSBL algorithms with the same model parameters. 
The MATLAB code used to generate the numerical tests can be found in the code repository \url{https://github.com/jglaubitz/LeveragingJointSparsity}.

\subsection{Hyper-prior parameter selection}
\label{sub:param_selection}

For all signal recovery problems, we either chose $(\beta,\vartheta) = ( 1, 1.501, 10^{-2} )$ for the IAS algorithm and $(\beta,\vartheta) =  ( 1, L/2+1.501, 10^{-2} )$ for the MMV-IAS algorithm, resulting in globally convex objective functions, or $(r,\beta,\vartheta) = (-1,1,10^{-4})$ for the IAS and MMV-IAS algorithm, resulting in non-convex objective functions. 
Moreover, we use $(\beta,\vartheta) = (1,10^{3})$ for the GSBL and MMV-GSBL algorithm, resulting in a non-convex objective function. 
Similar parameters were used in \cite{glaubitz2022generalized,xiao2022sequential} and \cite{calvetti2020sparse,calvetti2020sparsity} for the GSBL and IAS algorithm, respectively. 
We did not attempt to optimize any of these parameters.

\subsection{Signal deblurring}
\label{sub:deblurring}

We first consider (jointly) deblurring four piecewise-constant signals with a shared edge profile. 
The signals are generated by fixing five transition points in the interval $[0,1]$, dividing $[0,1]$ into six constant subintervals on which the signals are constant, and then randomly assigning signal values drawn from a uniform distribution.
The values are then normalized such that the maximum value of each signal is set to $1$. 
\cref{fig:deb_signal1_intro,fig:deb_signal2_intro} in \Cref{sec:introduction} show the first two signals and the given noisy blurred measurements. 
We aim to recover the nodal values $\mathbf{x}_{1:4}$ of all four signals at $N = 40$ equidistant grid points. 
The corresponding data model is 
\begin{equation}\label{eq:deblurring_model}
    \mathbf{y}_l = F \mathbf{x}_l + \mathbf{e}_l, \quad 
    l=1,\dots,4.
\end{equation} 
The discrete forward operator, $F$, represents the application of the midpoint quadrature to the convolution equation 
\begin{equation} 
	y(s) = \int_0^1 k(s-s') x(s) \intd s',
\end{equation}
where we assume a Gaussian convolution kernel of the form 
\begin{equation} 
	k(s) = \frac{1}{2 \pi \gamma^2} \exp\left( - \frac{s^2}{2 \gamma^2} \right)
\end{equation} 
with blurring parameter $\gamma = 3 \cdot 10^{-2}$. 
The forward operator is then given by 
\begin{equation}\label{eq:disc_convolution}
	[F]_{m,n} = h k( h[i-j] ), \quad i,j=1,\dots,n,
\end{equation}
where $h=1/n$ is the distance between consecutive grid points. 
Note that $F$ has full rank but quickly becomes ill-conditioned. 
The noise vectors $\mathbf{e}_{1:4}$ in \cref{eq:deblurring_model} are i.i.d., with zero mean and a common variance $\sigma^2 = 10^{-2}$. 
To reflect our prior knowledge that the signals are piecewise constant, we use  
\begin{equation}\label{eq:deblurring_R}
	R = 
    \begin{bmatrix}
        -1 & 1 & & \\ 
         & \ddots & \ddots & \\ 
         & & -1 & 1 
    \end{bmatrix} 
    \in \R^{(n-1) \times n} 
\end{equation} 
for the sparsifying operator. 
\cref{fig:deb_signal1_IAS_rp1_L4,fig:deb_signal2_IAS_rp1_L4} show the recovered first two signals using the IAS algorithm (to promote sparsity separately) and the proposed MMV-IAS algorithm (to promote sparsity jointly) for $r=1$, resulting in a globally convex objective function. 
\cref{fig:deb_signal1_IAS_rm1_L4,fig:deb_signal2_IAS_rm1_L4} show the same results for $r=-1$, resulting in a non-convex objective function.
\cref{fig:deb_signal1_GSBL_L4,fig:deb_signal2_GSBL_L4} report on the same test using the GSBL and MMV-GSBL algorithms.  
The results demonstrate that incorporating joint sparsity into the IAS and GSBL algorithms improves signal recovery accuracy. 

\begin{figure}[tb]
	\centering
  	\begin{subfigure}[b]{0.33\textwidth}
		\includegraphics[width=\textwidth]{%
      		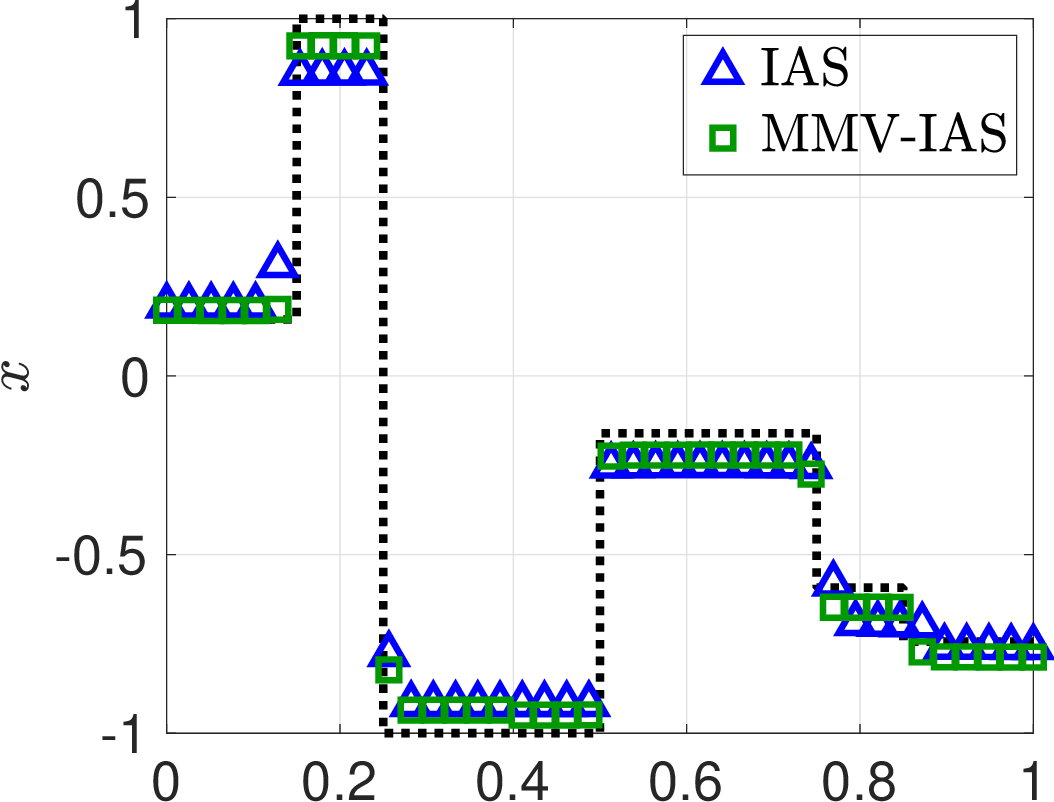} 
    	\caption{(MMV-)IAS for $r=1$}
    	\label{fig:deb_signal1_IAS_rp1_L4}
  	\end{subfigure}%
  	\begin{subfigure}[b]{0.33\textwidth}
		\includegraphics[width=\textwidth]{%
      		figures/deb_signal1_IAS_rm1_L4} 
    	\caption{(MMV-)IAS for $r=-1$}
    	\label{fig:deb_signal1_IAS_rm1_L4}
  	\end{subfigure}%
	\begin{subfigure}[b]{0.33\textwidth}
		\includegraphics[width=\textwidth]{%
      		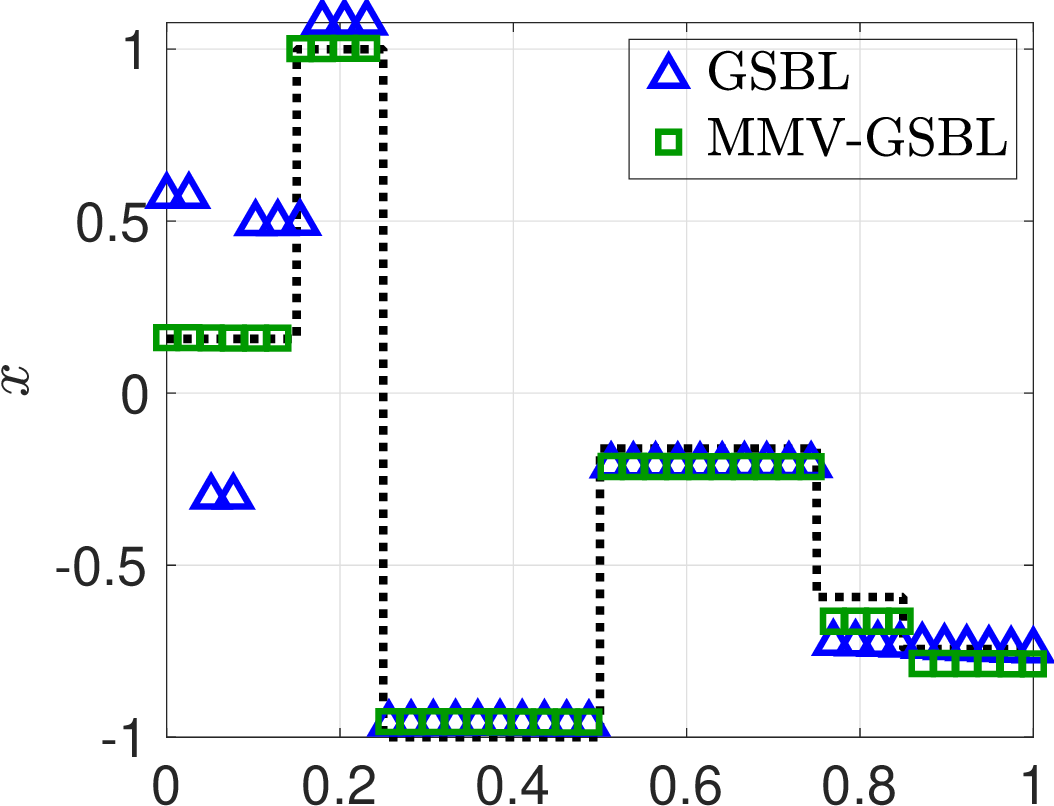} 
    	\caption{(MMV-)GSBL}
    	\label{fig:deb_signal1_GSBL_L4}
  	\end{subfigure}%
	\\
	\begin{subfigure}[b]{0.33\textwidth}
		\includegraphics[width=\textwidth]{%
      		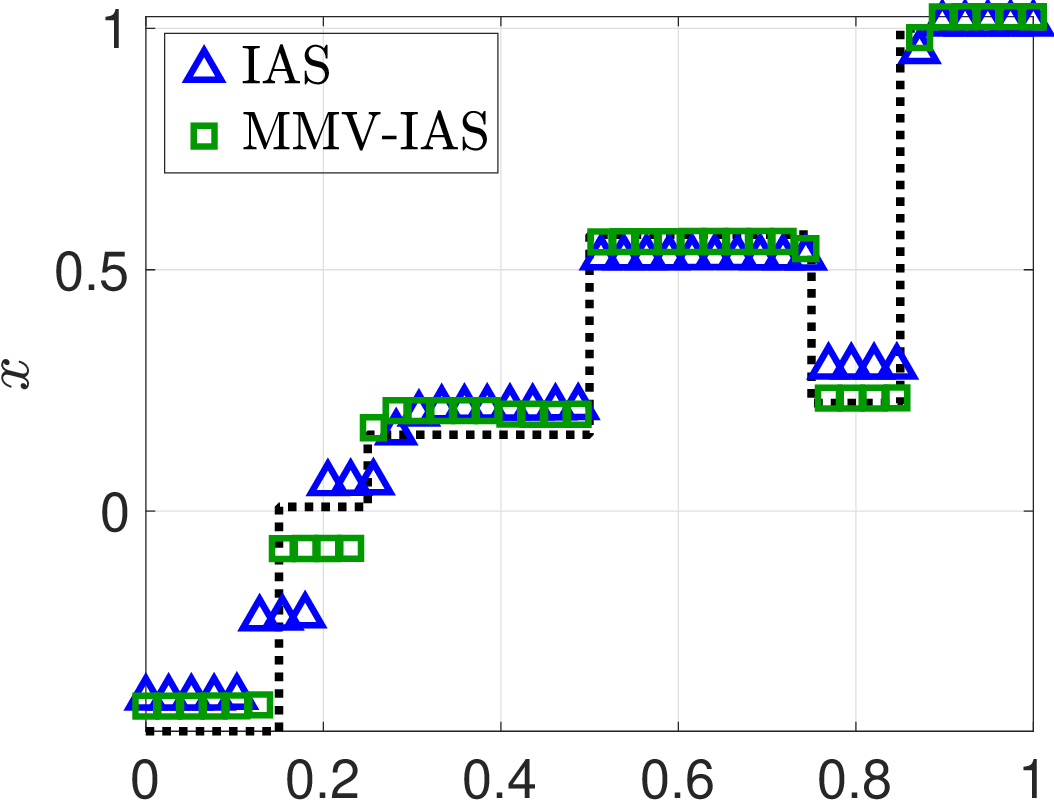} 
    	\caption{(MMV-)IAS for $r=1$}
    	\label{fig:deb_signal2_IAS_rp1_L4}
  	\end{subfigure}%
  	\begin{subfigure}[b]{0.33\textwidth}
		\includegraphics[width=\textwidth]{%
      		figures/deb_signal2_IAS_rm1_L4} 
    	\caption{(MMV-)IAS for $r=-1$}
    	\label{fig:deb_signal2_IAS_rm1_L4}
  	\end{subfigure}%
	\begin{subfigure}[b]{0.33\textwidth}
		\includegraphics[width=\textwidth]{%
      		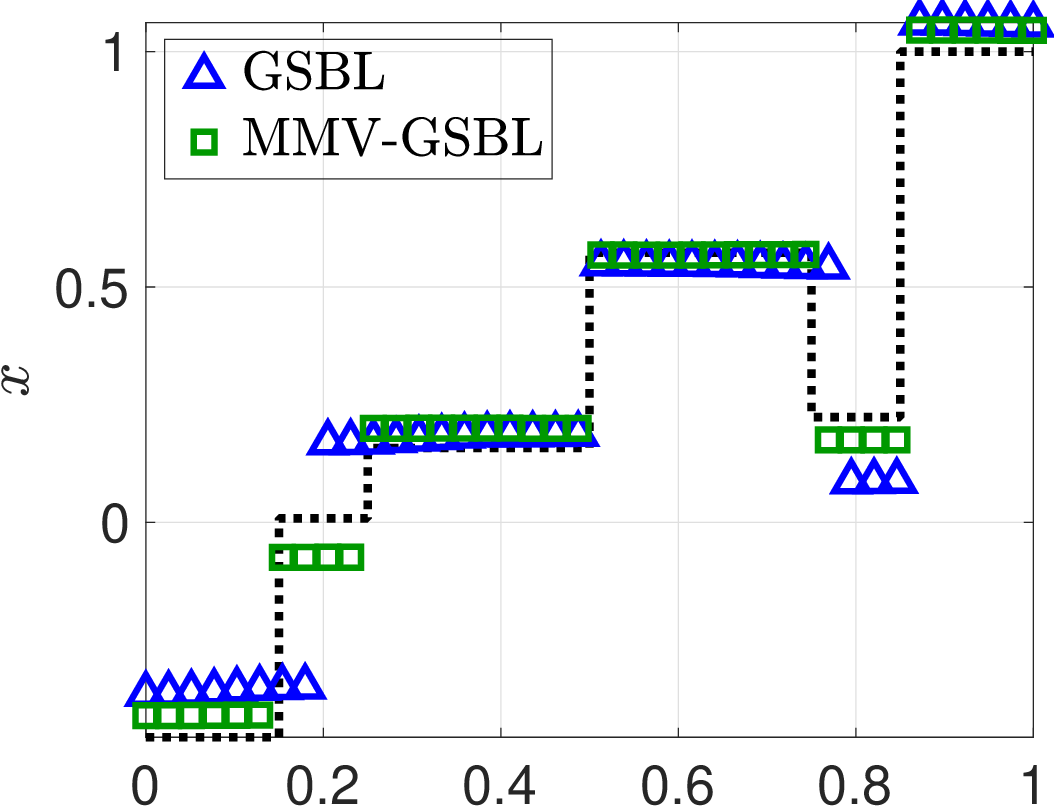} 
    	\caption{(MMV-)GSBL}
    	\label{fig:deb_signal2_GSBL_L4}
  	\end{subfigure}%
  	\caption{ 
	Different reconstructions of the first (top row) and second (bottom row) of four piecewise constant signals with a common edge profile and noisy blurred measurements using the existing IAS/GSBL algorithm to separately recover them (blue triangles) and the proposed MMV-IAS/-GSBL algorithm to jointly recover them (green squares). 
  	}
  	\label{fig:deb_signal}
\end{figure}

\begin{figure}[tb]
	\centering 
	\begin{subfigure}[b]{0.33\textwidth}
		\includegraphics[width=\textwidth]{%
      		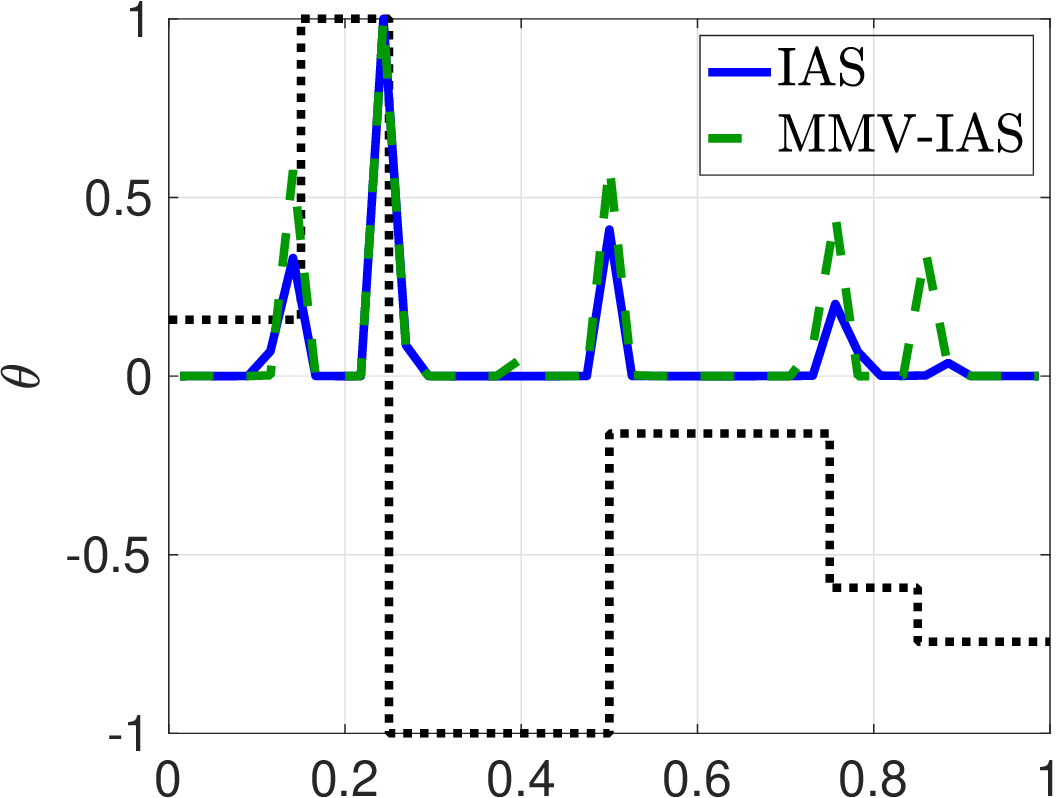} 
    	\caption{(MMV-)IAS for $r=1$}
    	\label{fig:deb_signal1_IAS_rp1_L4_theta}
  	\end{subfigure}%
  	\begin{subfigure}[b]{0.33\textwidth}
		\includegraphics[width=\textwidth]{%
      		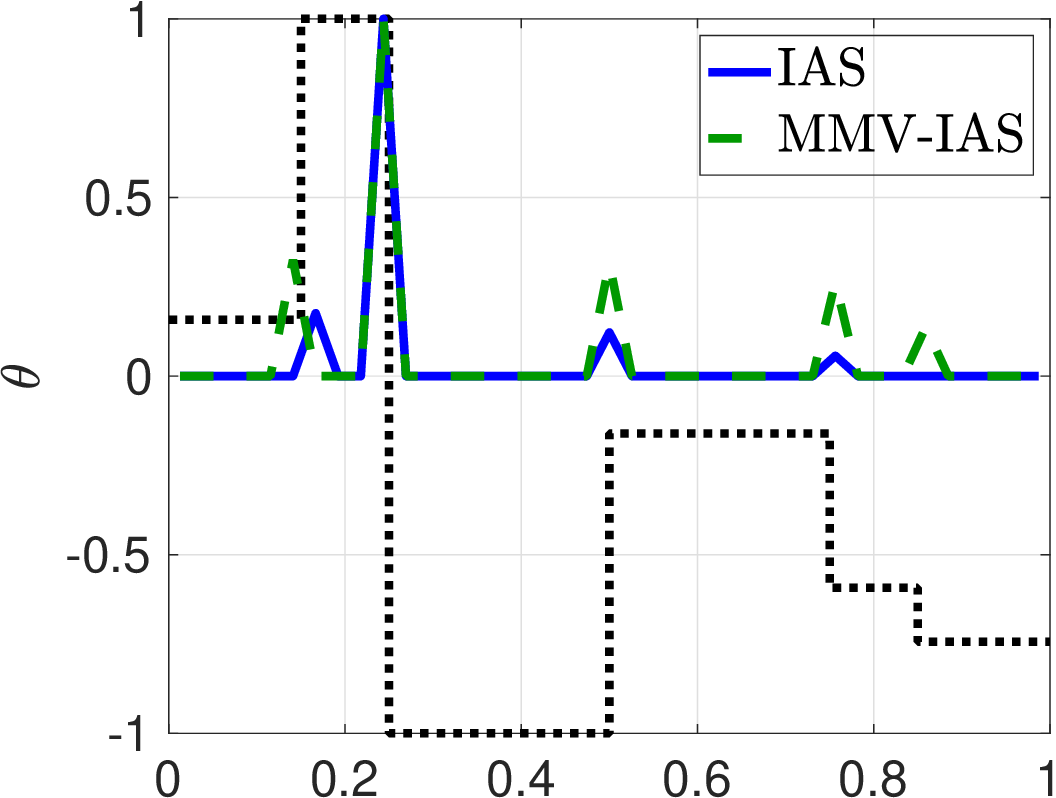} 
    	\caption{(MMV-)IAS for $r=-1$}
    	\label{fig:deb_signal1_IAS_rm1_L4_theta}
  	\end{subfigure}%
	\begin{subfigure}[b]{0.33\textwidth}
		\includegraphics[width=\textwidth]{%
      		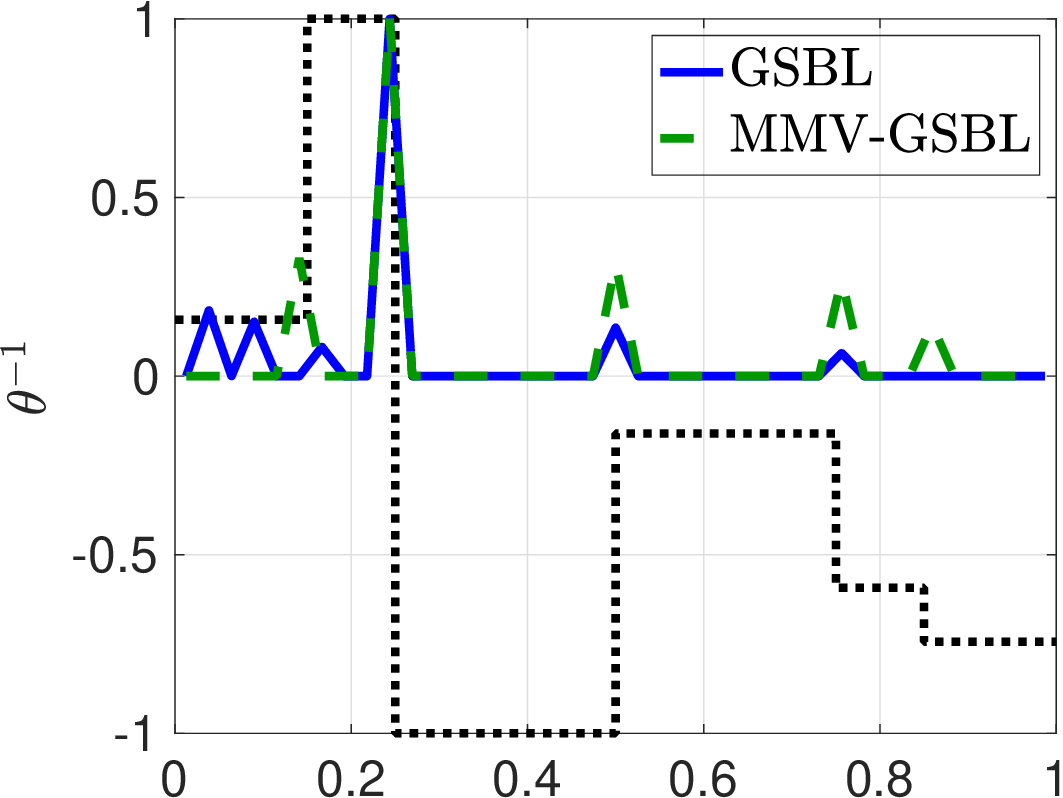} 
    	\caption{(MMV-)GSBL}
    	\label{fig:deb_signal1_GSBL_L4_theta}
  	\end{subfigure}%
	\\
	\begin{subfigure}[b]{0.33\textwidth}
		\includegraphics[width=\textwidth]{%
      		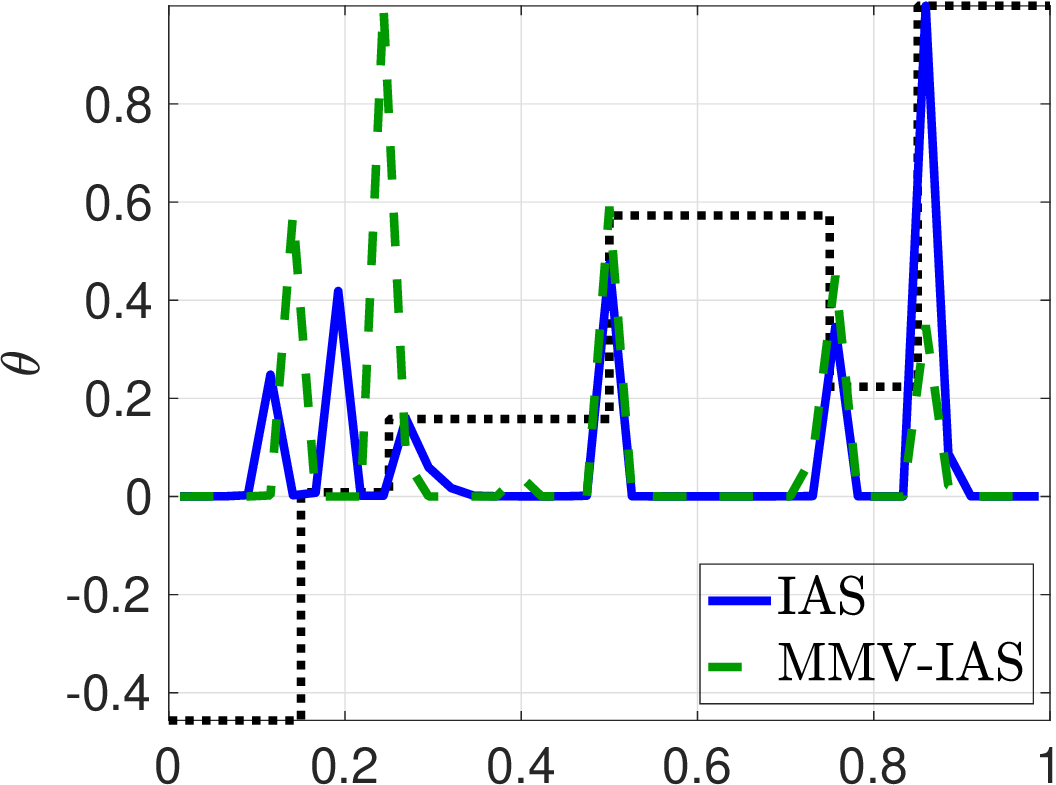} 
    	\caption{(MMV-)IAS for $r=1$}
    	\label{fig:deb_signal2_IAS_rp1_L4_theta}
  	\end{subfigure}%
  	\begin{subfigure}[b]{0.33\textwidth}
		\includegraphics[width=\textwidth]{%
      		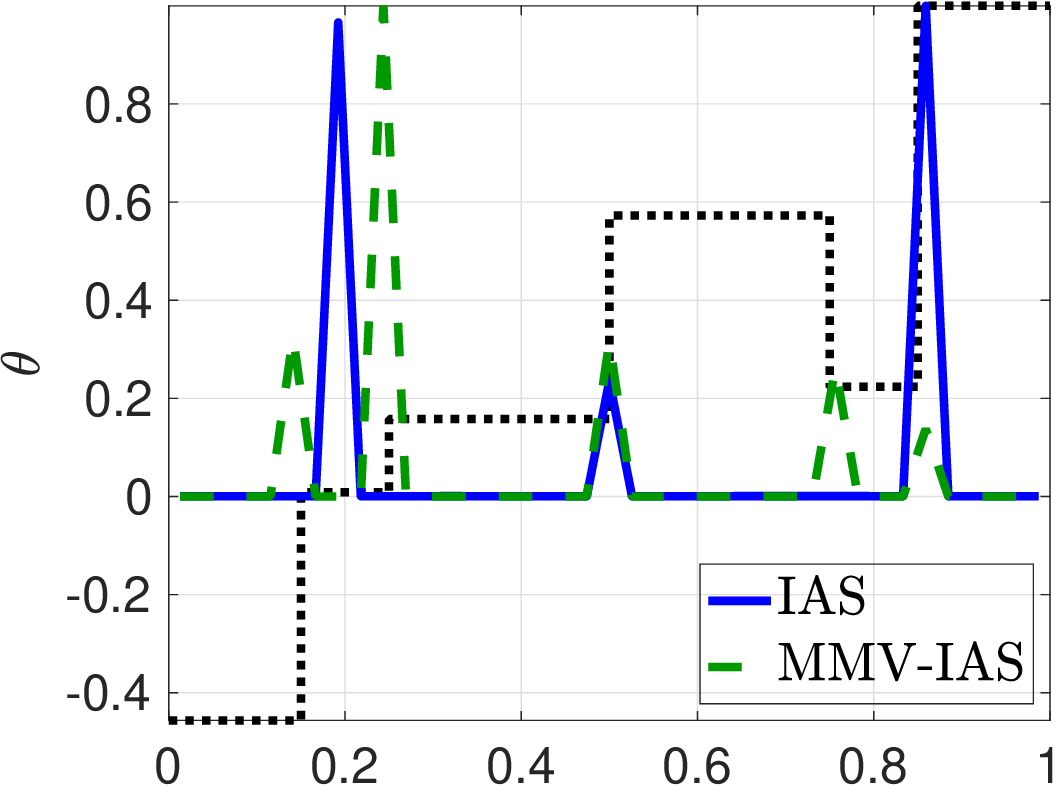} 
    	\caption{(MMV-)IAS for $r=-1$}
    	\label{fig:deb_signal2_IAS_rm1_L4_theta}
  	\end{subfigure}%
  	\begin{subfigure}[b]{0.33\textwidth}
		\includegraphics[width=\textwidth]{%
      		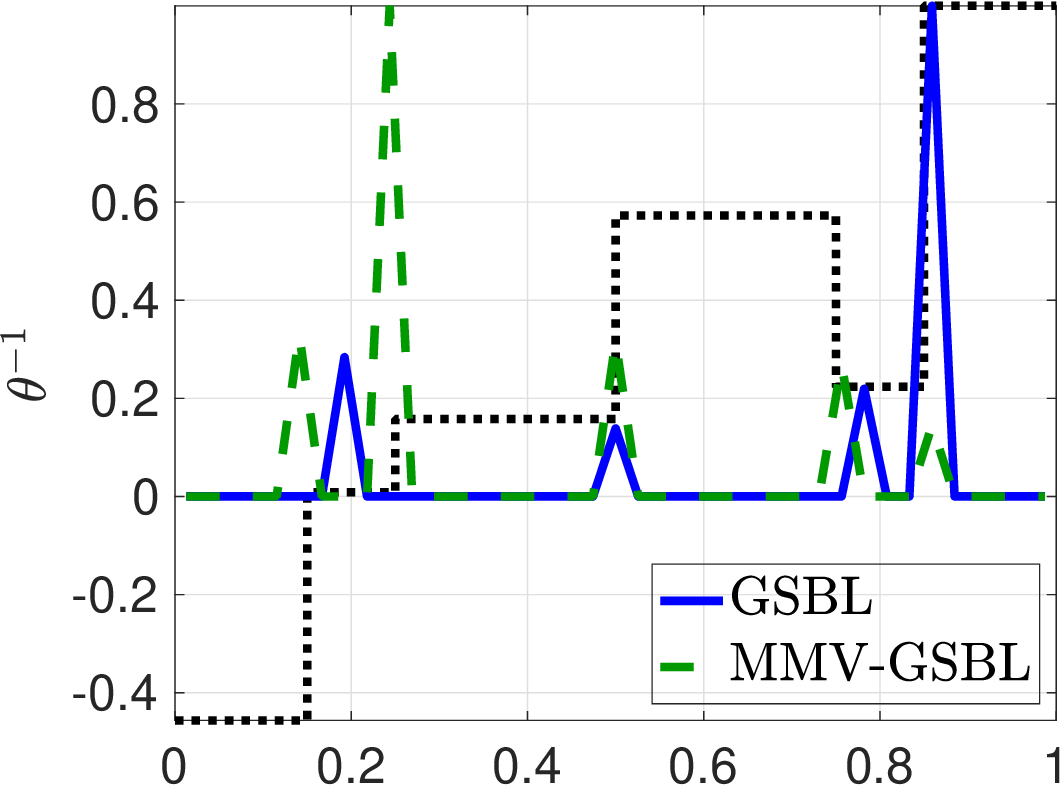} 
    	\caption{(MMV-)GSBL}
    	\label{fig:deb_signal2_GSBL_L4_theta}
  	\end{subfigure}%
  	\caption{ 
	Normalized MAP estimate of the hyper-parameter $\theta$ for the first (top row) and second (bottom row) signal using the IAS and MMV-IAS algorithms with $r=\pm1$ and the GSBL and MMV-GSBL algorithms 
	}
  	\label{fig:deb_signal_theta}
\end{figure}

The improved accuracy of the MMV-IAS and MMV-GSBL algorithms can be attributed to the use of a common hyper-parameter vector $\boldsymbol{\theta}$ that more accurately detects edge locations compared to separate hyper-parameter vectors $\boldsymbol{\theta}_{1:L}$ used in the IAS and GSBL algorithms. 
This is evident in \cref{fig:deb_signal_theta}, which shows the normalized estimated hyper-parameters produced by the IAS/GSBL and MMV-IAS/GSBL algorithms for the first two signals. 
While the MMV-IAS and MMV-GSBL algorithms accurately capture all edge locations, the IAS and GSBL algorithms produce visibly erroneous hyper-parameter profiles. 
The impact of missed true edge locations and false detection of others are clearly visible in \cref{fig:deb_signal1_GSBL_L4}, which shows that the MMV-GSBL approach eliminates the artificial edges around $0.1$ on the horizontal axis. 
Similarly, \cref{fig:deb_signal2_GSBL_L4} demonstrates that the MMV-GSBL approach retains the existing edges around $0.2$ on the horizontal axis, which are missed by the GSBL approach. 

\begin{figure}[tb]
	\centering
	\begin{subfigure}[b]{0.33\textwidth}
		\includegraphics[width=\textwidth]{%
      		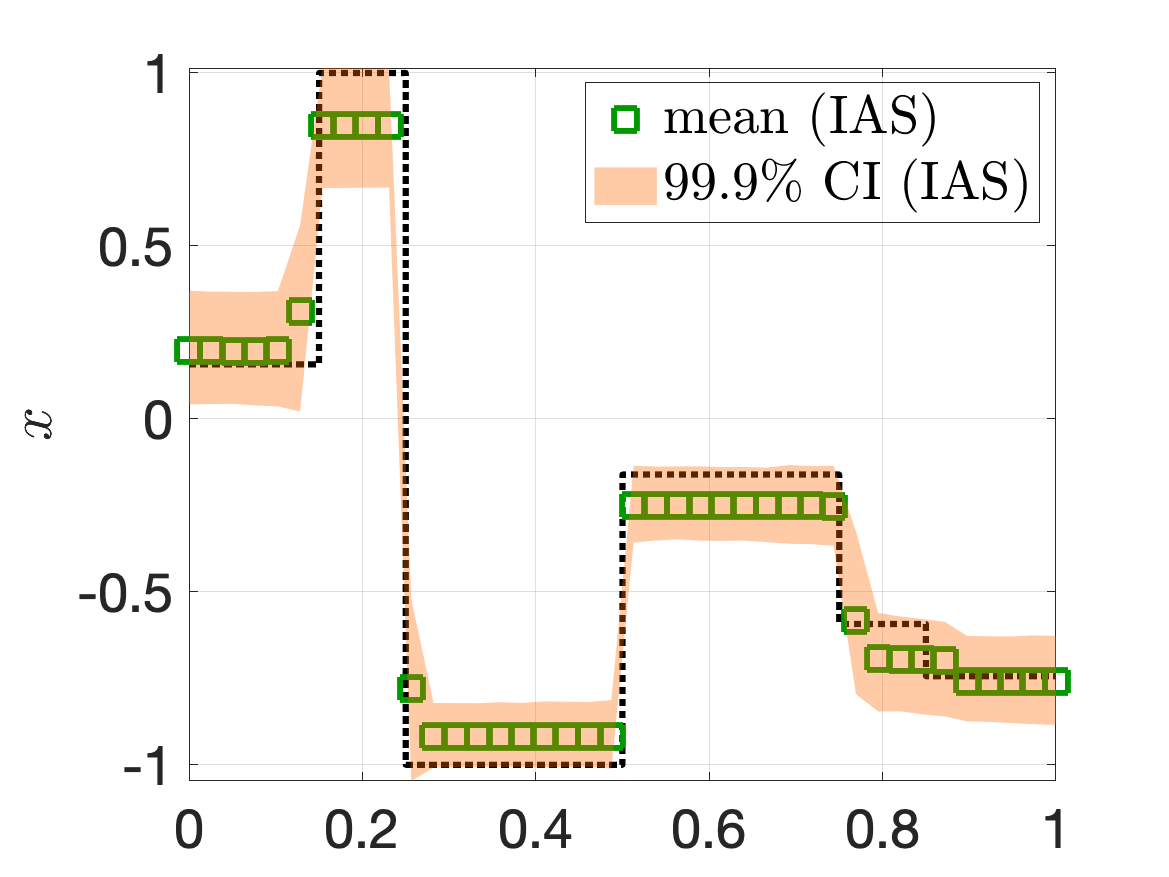} 
    	\caption{IAS algorithm for $r=1$}
    	\label{fig:deb_signal1_IAS_rp1_L4_CI}
  	\end{subfigure}%
	\begin{subfigure}[b]{0.33\textwidth}
		\includegraphics[width=\textwidth]{%
      		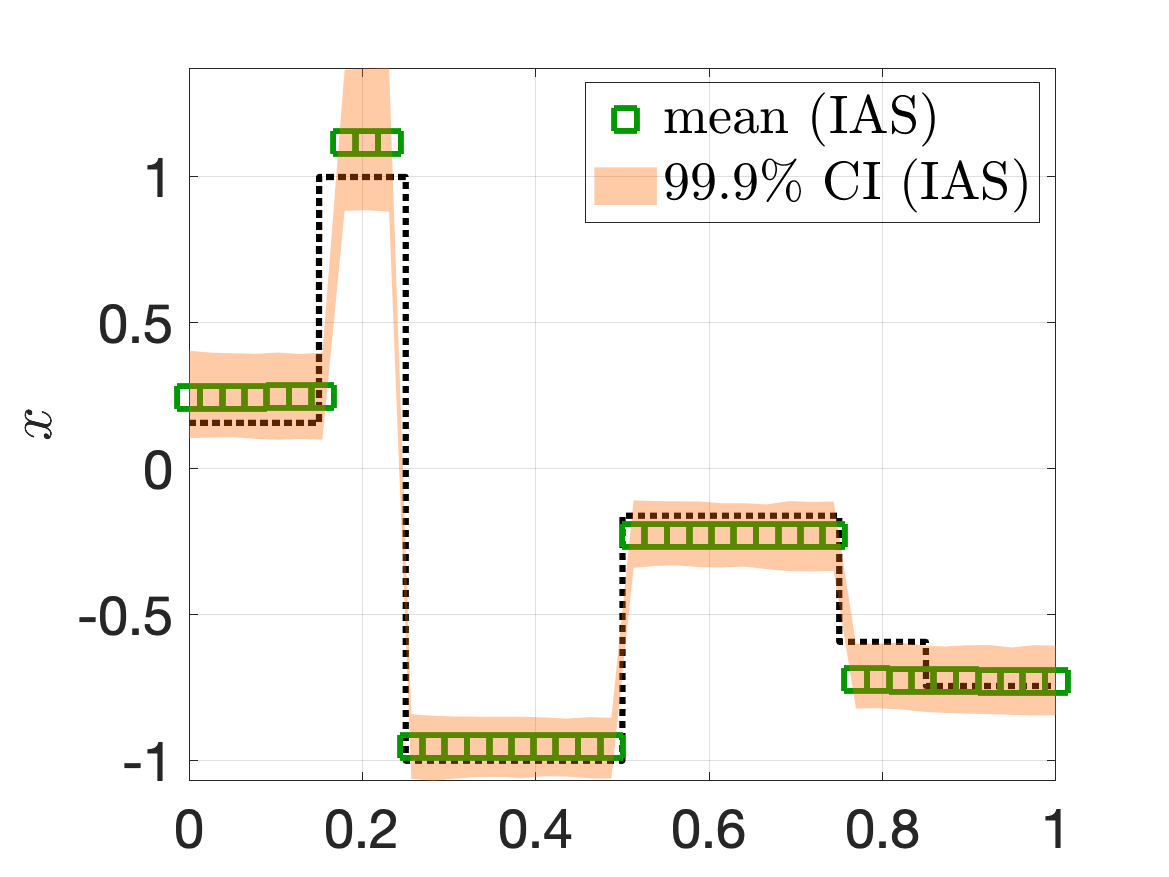} 
    	\caption{IAS algorithm for $r=-1$}
    	\label{fig:deb_signal1_IAS_rm1_L4_CI}
  	\end{subfigure}%
  	\begin{subfigure}[b]{0.33\textwidth}
		\includegraphics[width=\textwidth]{%
      		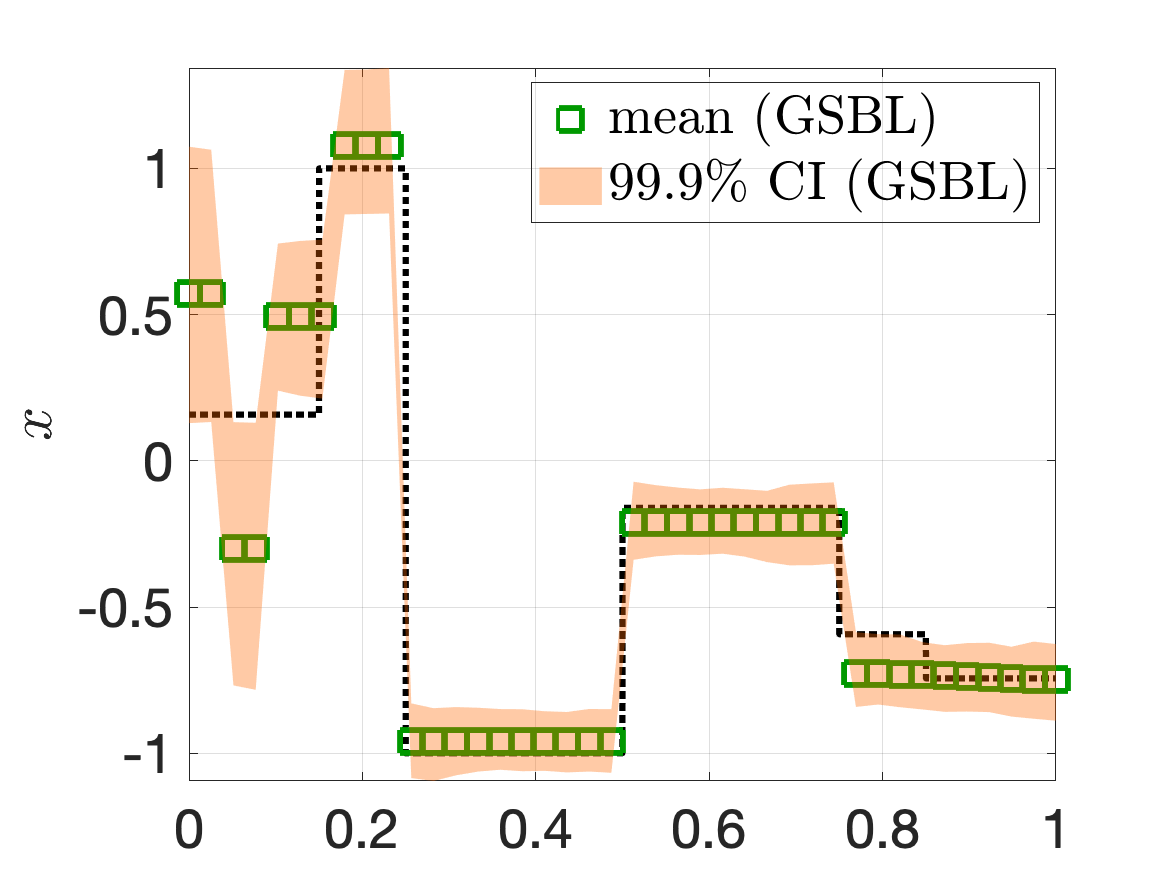} 
    	\caption{GSBL algorithm}
    	\label{fig:deb_signal1_GSBL_L4_CI}
  	\end{subfigure}%
	\\
	\begin{subfigure}[b]{0.33\textwidth}
		\includegraphics[width=\textwidth]{%
      		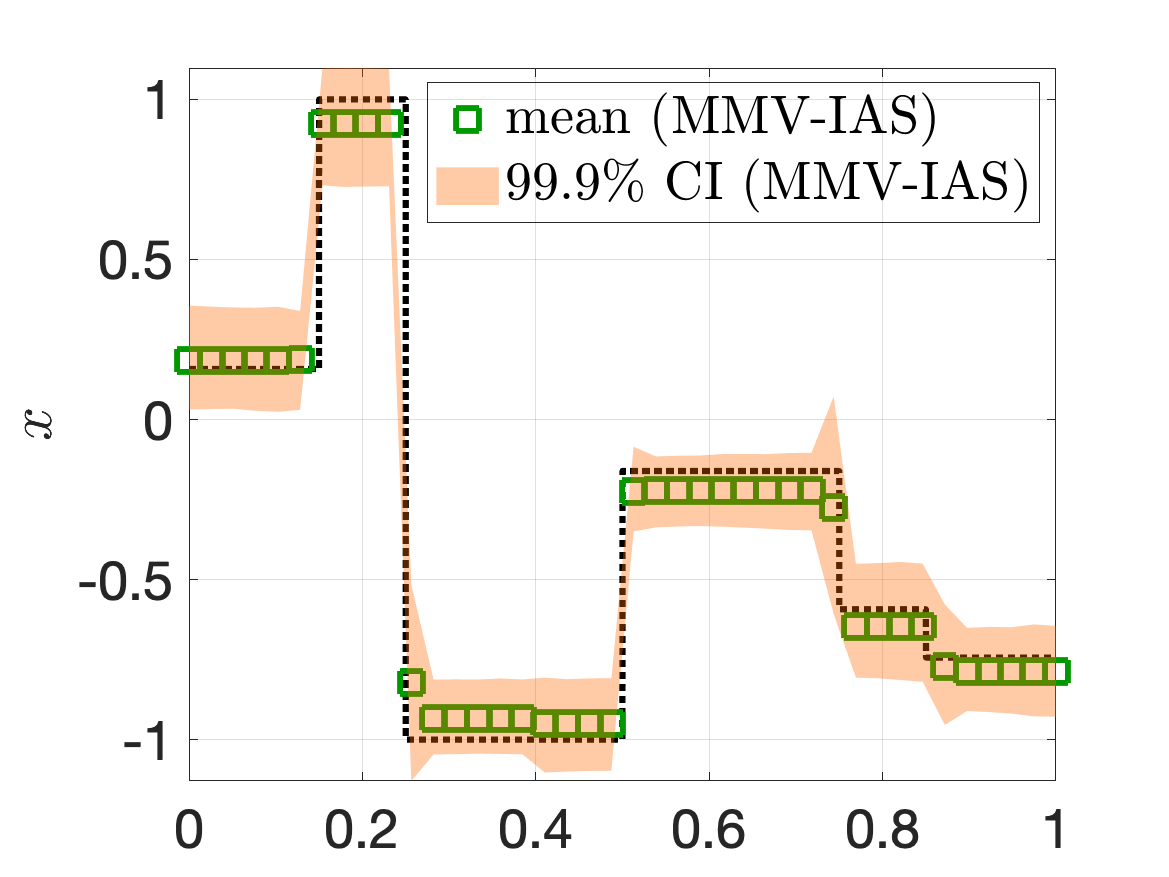} 
    	\caption{MMV-IAS algorithm for $r=1$}
    	\label{fig:deb_signal1_MMVIAS_rp1_L4_CI}
  	\end{subfigure}%
	\begin{subfigure}[b]{0.33\textwidth}
		\includegraphics[width=\textwidth]{%
      		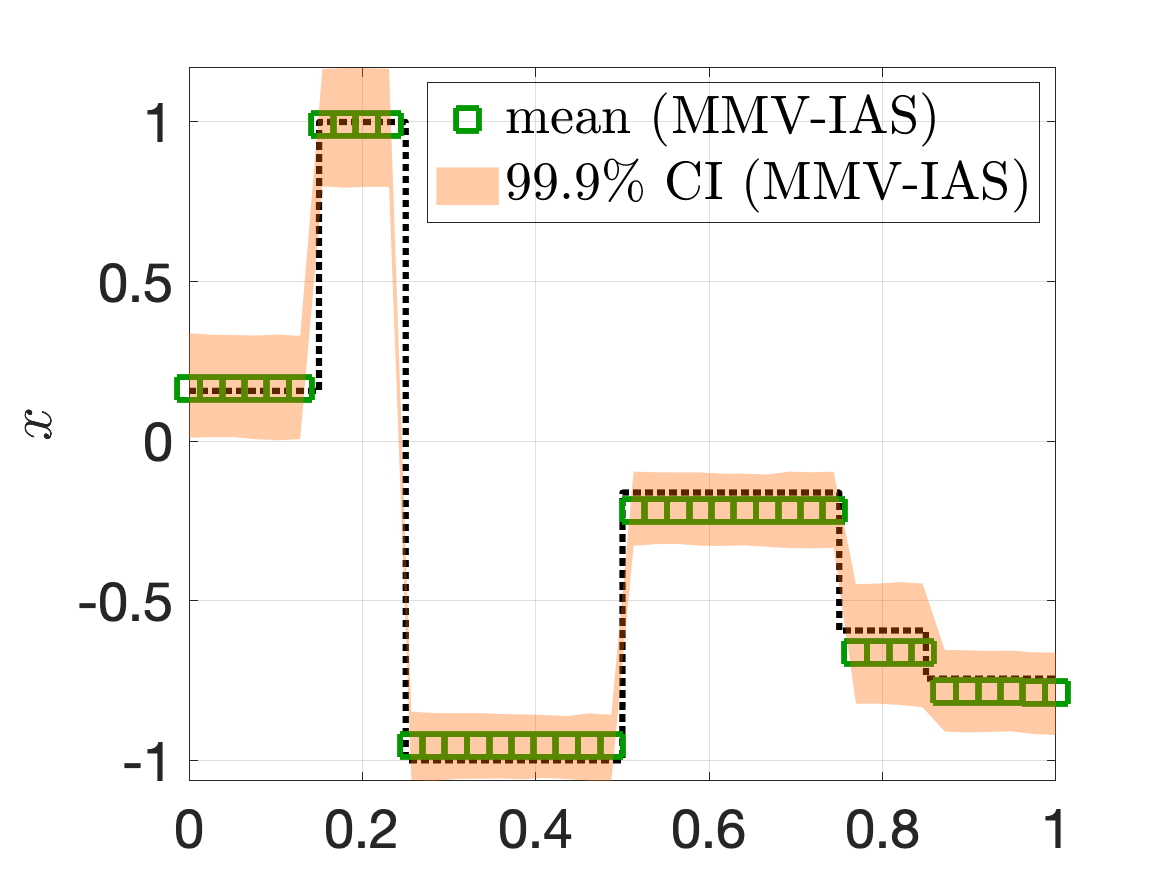} 
    	\caption{MMV-IAS algorithm for $r=-1$}
    	\label{fig:deb_signal1_MMVIAS_rm1_L4_CI}
  	\end{subfigure}%
  	\begin{subfigure}[b]{0.33\textwidth}
		\includegraphics[width=\textwidth]{%
      		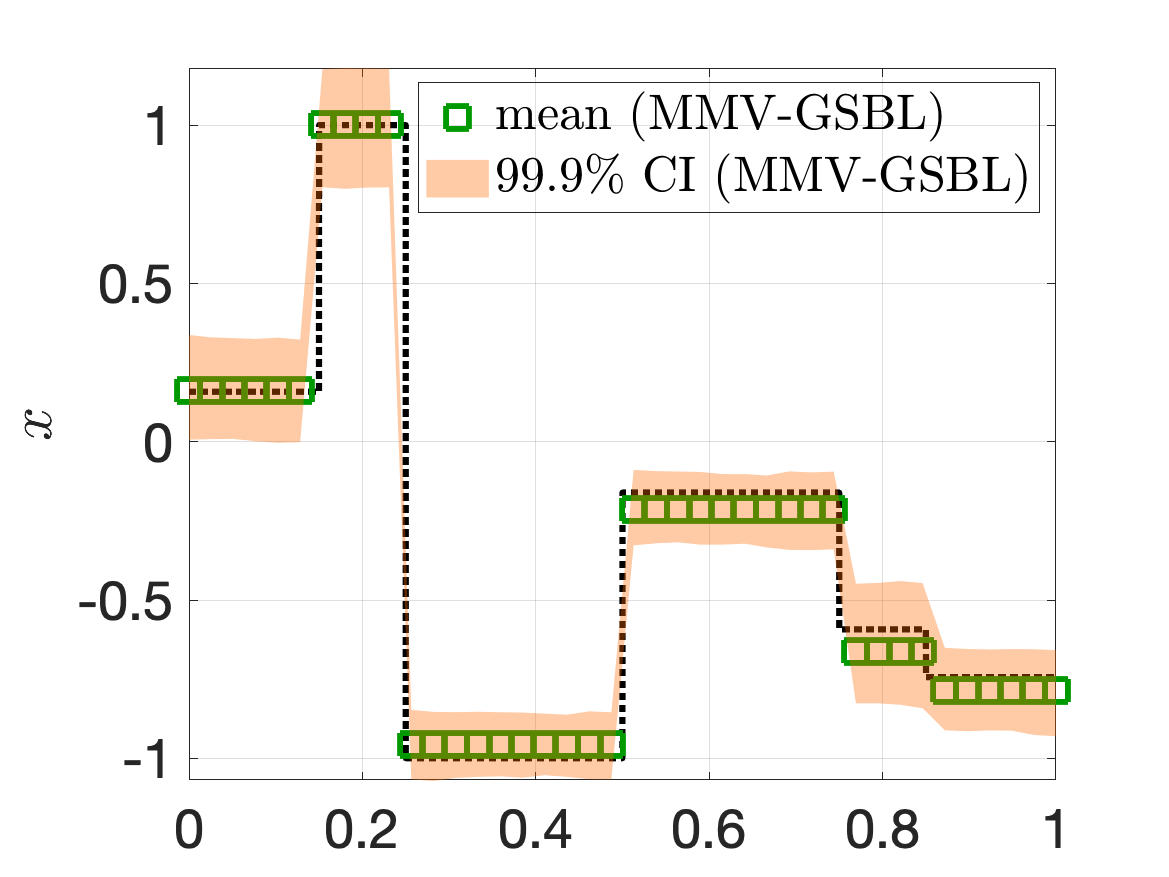} 
    	\caption{MMV-GSBL algorithm}
    	\label{fig:deb_signal1_MMVGSBL_L4_CI}
  	\end{subfigure}%
  	\caption{ 
	The $99.9\%$ credible intervals (CIs) for the recovered first (top row) and second (bottom row) signal conditioned on the MAP estimate of the hyper-parameter vector $\boldsymbol{\theta}^{\rm MAP}$ using the IAS and MMV-IAS algorithm with $r=\pm1$ as well as the GSB and MMV-GSBL algorithm
  	}
  	\label{fig:deb_signal_CI}
\end{figure}

The proposed MMV-IAS and MMV-GSBL algorithms have the additional advantage of quantifying uncertainty in the recovered signals, as described in \cref{sub:UQ_IAS} and \cref{rem:GSBL_UQ}. 
This is demonstrated in \cref{fig:deb_signal_CI}, which shows the $99.9\%$ credible intervals of the fully conditional posterior densities $\pi_{\mathbf{X}_1|\boldsymbol{\Theta}=\boldsymbol{\theta},\mathbf{Y}_{1:L}=\mathbf{y}_{1:L}}$ of the first recovered signal for the IAS, MMV-IAS, GSBL, and MMV-GSBL model.
Here, $\mathbf{y}_{1:L}$ are the given noisy blurred MMVs and $\boldsymbol{\theta}$ is the estimated hyper-parameter vector.

\subsection{Error and success analysis}
\label{sub:analysis}

We now conduct a synthetic sparse signal recovery experiment to further assess the performance of the proposed joint-sparsity-promoting MMV-IAS and MMV-GSBL algorithms. 
We consider $L$ randomly generated signals, $\mathbf{x}_{1:L}$, each of size $N$. 
We fix the number of measurements, $M$, non-zero components, $s$, and trials, $T$. 
For each trial, $t=1,\dots,T$, we proceed as follows:
\begin{enumerate}
	\item[(i)] 
	Generate a support set $S \subset \{ 1,\dots,N \}$ uniformly at random with size $|S| = s$;
	
	\item[(ii)] 
	Define signal vectors $\mathbf{x}_1,\dots,\mathbf{x}_L$ such that $\mathrm{supp}(\mathbf{x}_1) = \dots = \mathrm{supp}(\mathbf{x}_L) = S$, where the non-zero entries are drawn from the standard normal distribution; 
	
	\item[(iii)] 
	Generate a forward operator $F$ as described below, fix a noise variance $\sigma^2$, and compute the measurement vectors $\mathbf{y}_l = F \mathbf{x}_l + \mathbf{e}_l$, where $\mathbf{e}_l$ is drawn from $\mathcal{N}(\mathbf{0},\sigma^2 I)$;
	
	\item[(iv)]
	Compute the reconstructions $\hat{\mathbf{x}}_1,\dots,\hat{\mathbf{x}}_L$ using the desired algorithm (e.g., IAS or MMV-IAS); 
	 
	\item[(v)] 
	Compute the normalized error $E_t = \sqrt{ \sum_{l=1}^L \| \mathbf{x}_l - \hat{\mathbf{x}}_l \|_2^2 / \sum_{l=1}^L \| \mathbf{x}_l \|_2^2 }$ for each algorithm.
	
\end{enumerate}
Finally, we evaluate the algorithm performance using the average error and empirical success probability (ESP). 
The \emph{average error} is $E = (E_1+\dots+E_T)/T$, i.e., the average of the individual trial errors. 
The \emph{ESP} is the fraction of trials that successfully recovered the vectors $\mathbf{x}_1,\dots,\mathbf{x}_L$ up to a given tolerance $\varepsilon_{\rm tol}$, i.e., $E_t < \varepsilon_{\rm tol}$. 
We use a subsampled discrete cosine transform (DFT) as the forward operator $F$. 
This mimics the situation in which Fourier data are collected (e.g., synthetic aperture radar and magnetic resonance imaging), but some of the data are determined to be unusable due to a system malfunction or obstruction.
Specifically, we generate a set $\Omega \subset \{1,\dots,N\}$ of size $M$ uniformly at random and let 
\begin{equation} 
	F = P_{\Omega} A, 
\end{equation} 
where $A \in \R^{N \times N}$ is the DCT matrix and $P_{\Omega} \in \R^{M \times N}$ is the operator selecting rows of $A$ corresponding to the indices in $\Omega$. 
The identity matrix is used as the sparsifying operator as the signals are assumed to be sparse. 
In our experiments we set $N=100$, {$s=20$}, $T=10$, $\sigma^2=10^{-6}$, and $\varepsilon_{\rm tol}=10^{-2}$. 

\begin{figure}[tb]
	\centering
  	\begin{subfigure}[b]{0.33\textwidth}
		\includegraphics[width=\textwidth]{%
      		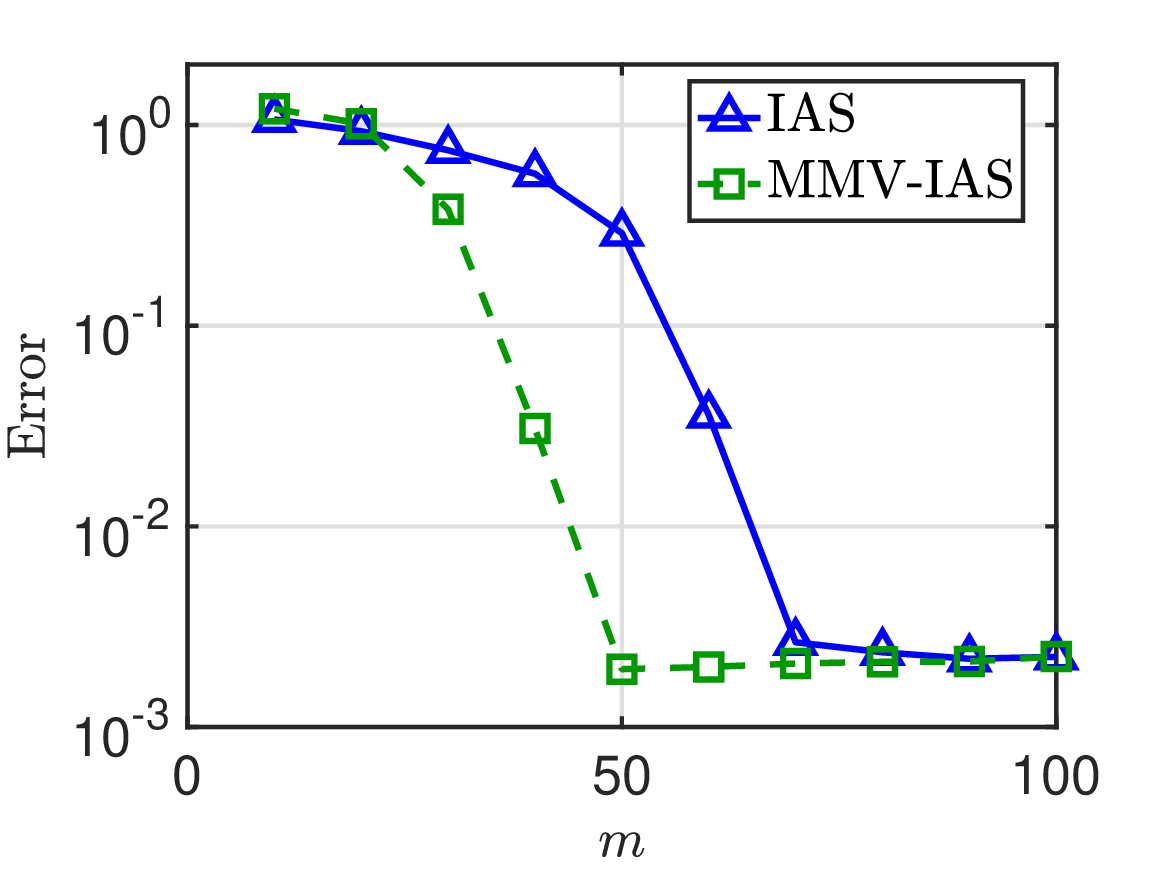} 
    	\caption{Average errors, $L=4$}
    	\label{fig:comparision_error_IAS_C4}
  	\end{subfigure}%
  	\begin{subfigure}[b]{0.33\textwidth}
		\includegraphics[width=\textwidth]{%
      		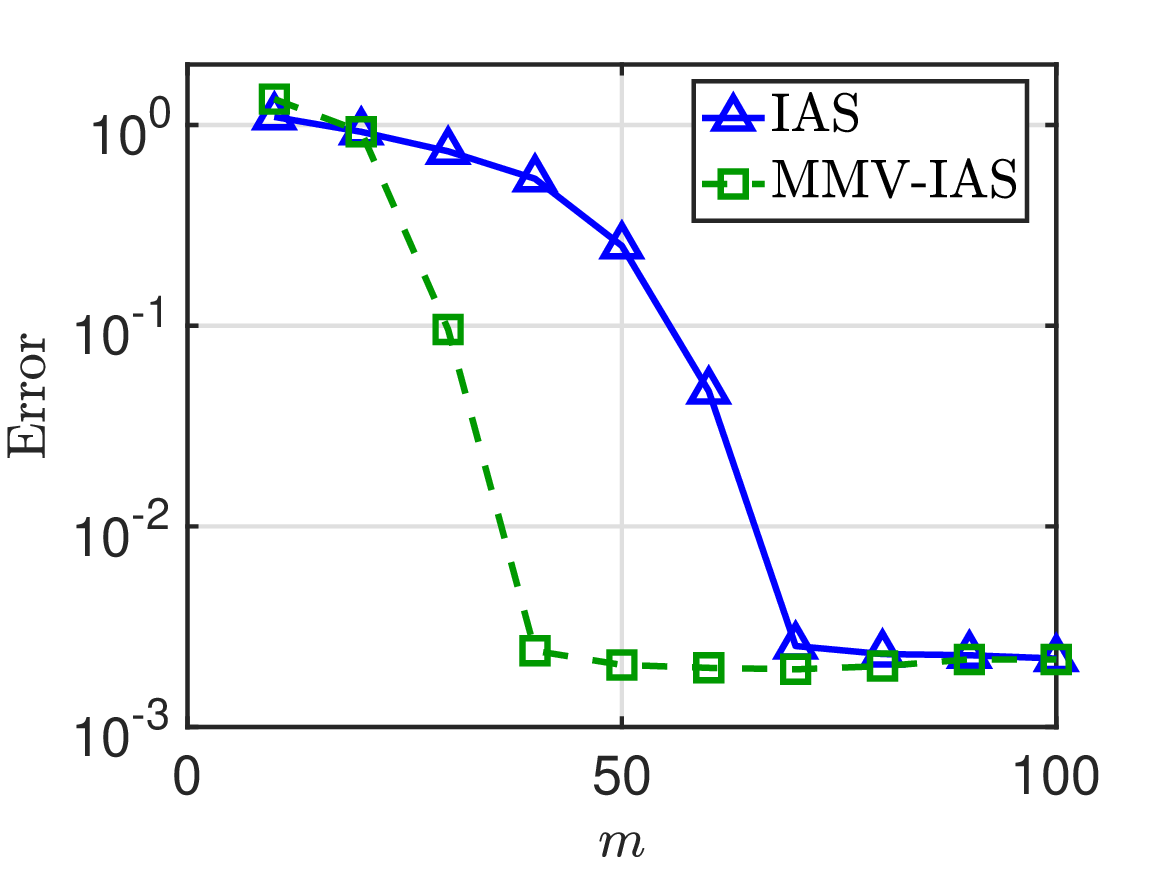} 
    	\caption{Average errors, $L=8$}
    	\label{fig:comparision_error_IAS_C8}
  	\end{subfigure}%
	\begin{subfigure}[b]{0.33\textwidth}
		\includegraphics[width=\textwidth]{%
      		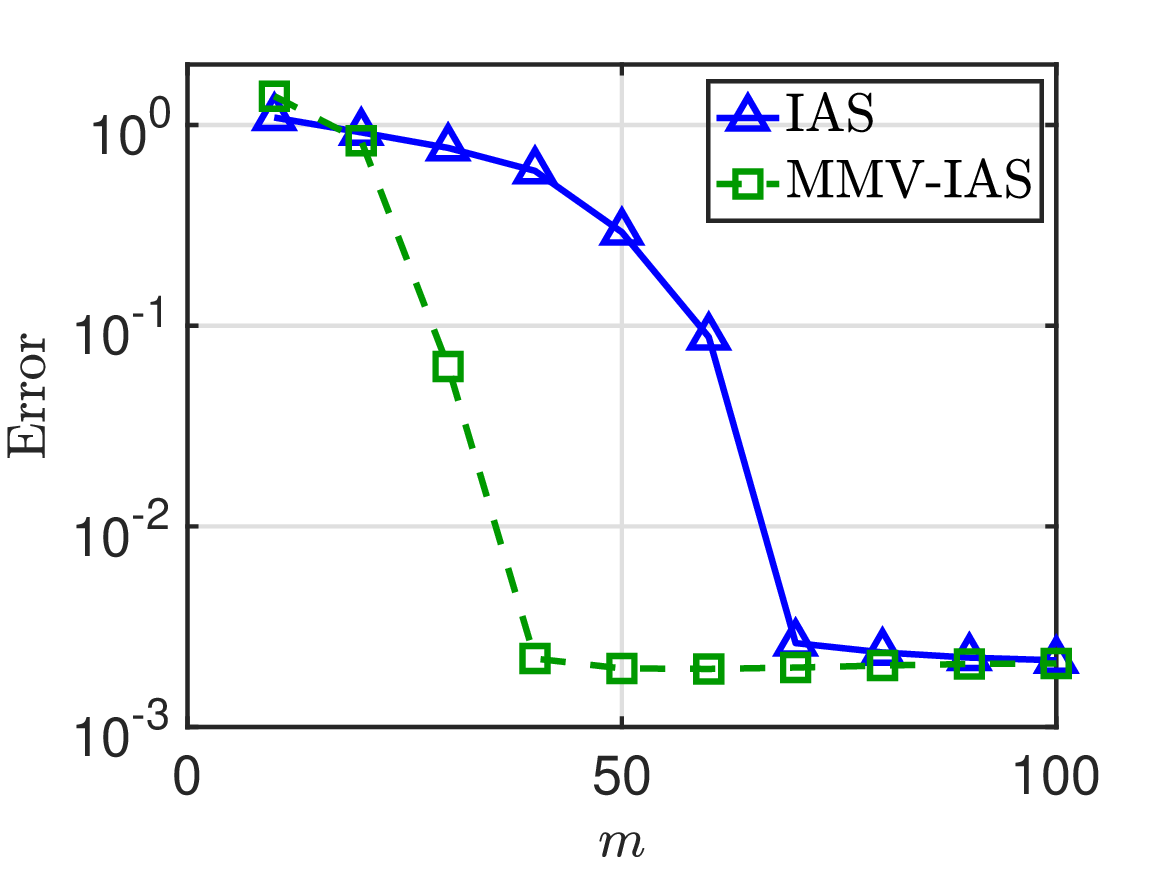} 
    	\caption{Average errors, $L=16$}
    	\label{fig:comparision_error_IAS_C16}
  	\end{subfigure}%
	\\
	\begin{subfigure}[b]{0.33\textwidth}
		\includegraphics[width=\textwidth]{%
      		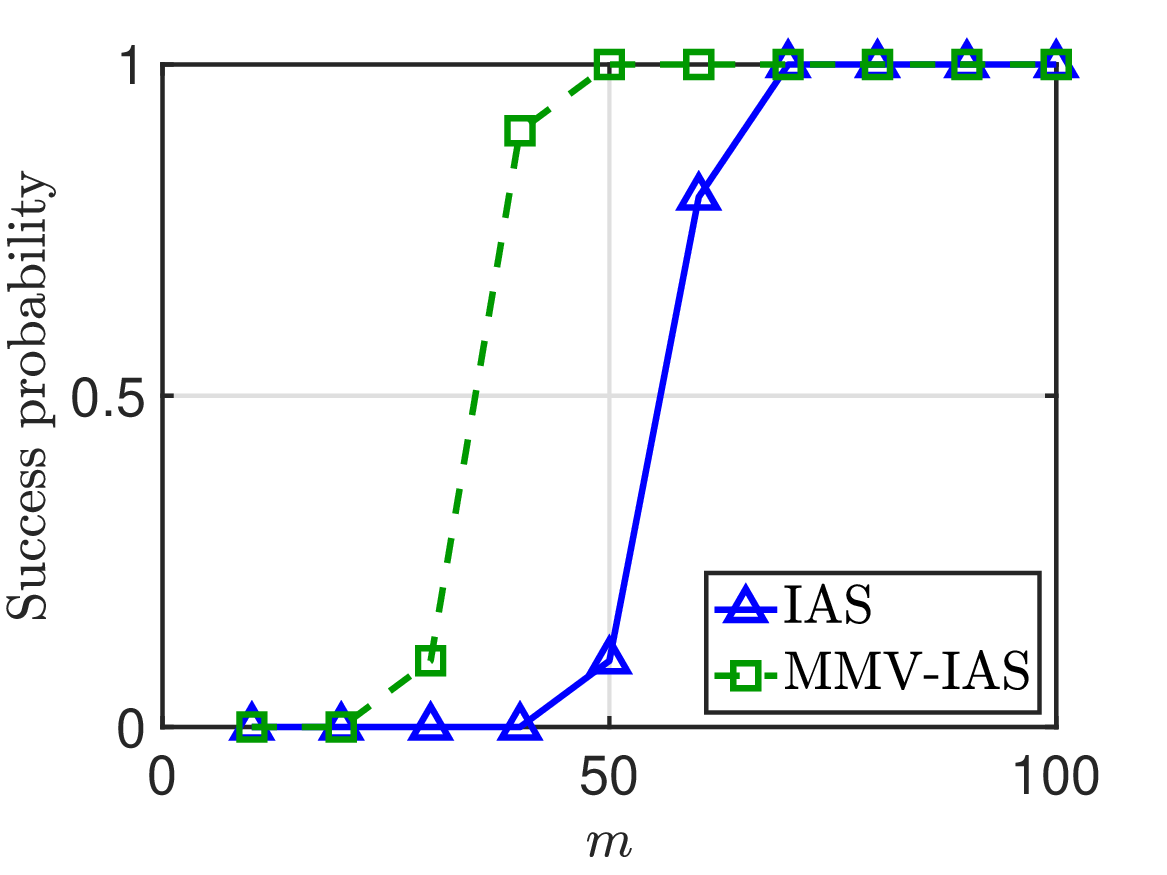} 
    	\caption{Success probability, $L=4$}
    	\label{fig:comparision_succ_IAS_C4}
  	\end{subfigure}%
  	\begin{subfigure}[b]{0.33\textwidth}
		\includegraphics[width=\textwidth]{%
      		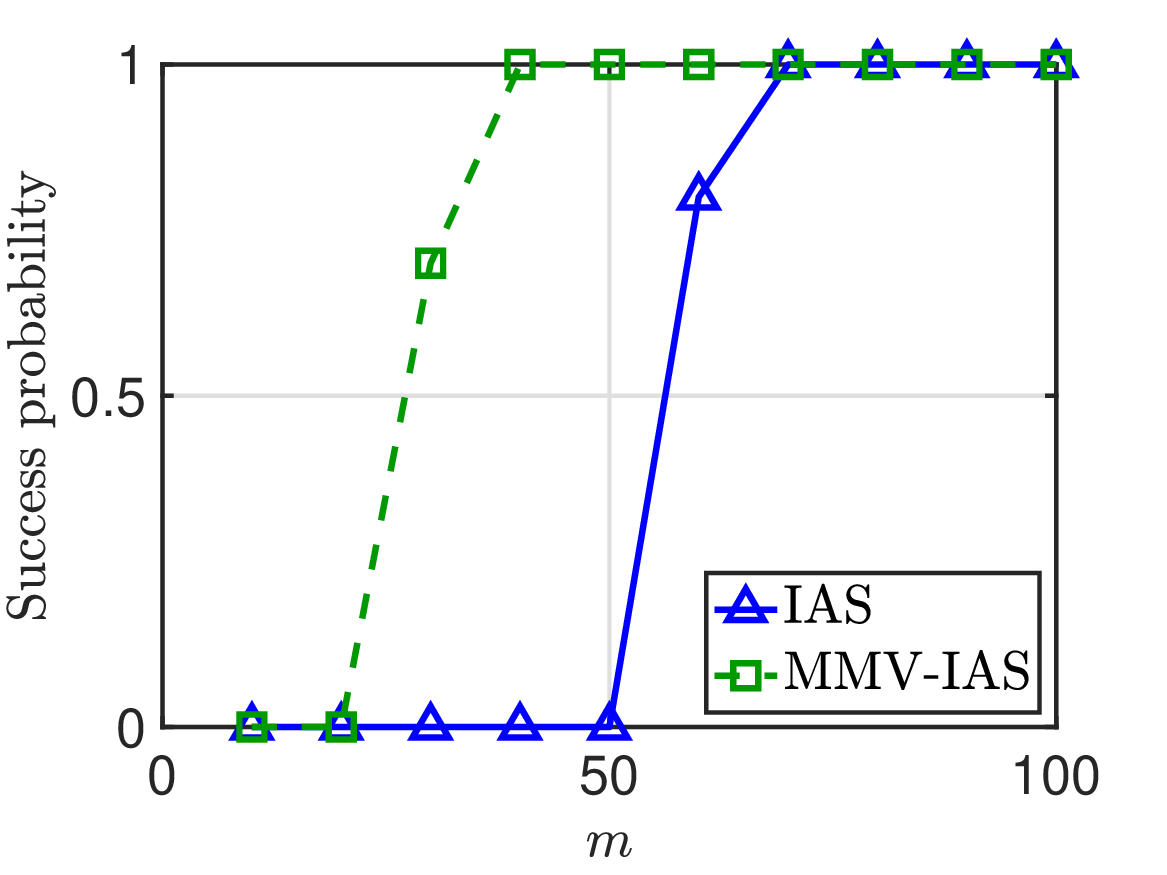} 
    	\caption{Success probability, $L=8$}
    	\label{fig:comparision_succ_IAS_C8}
  	\end{subfigure}%
  	\begin{subfigure}[b]{0.33\textwidth}
		\includegraphics[width=\textwidth]{%
      		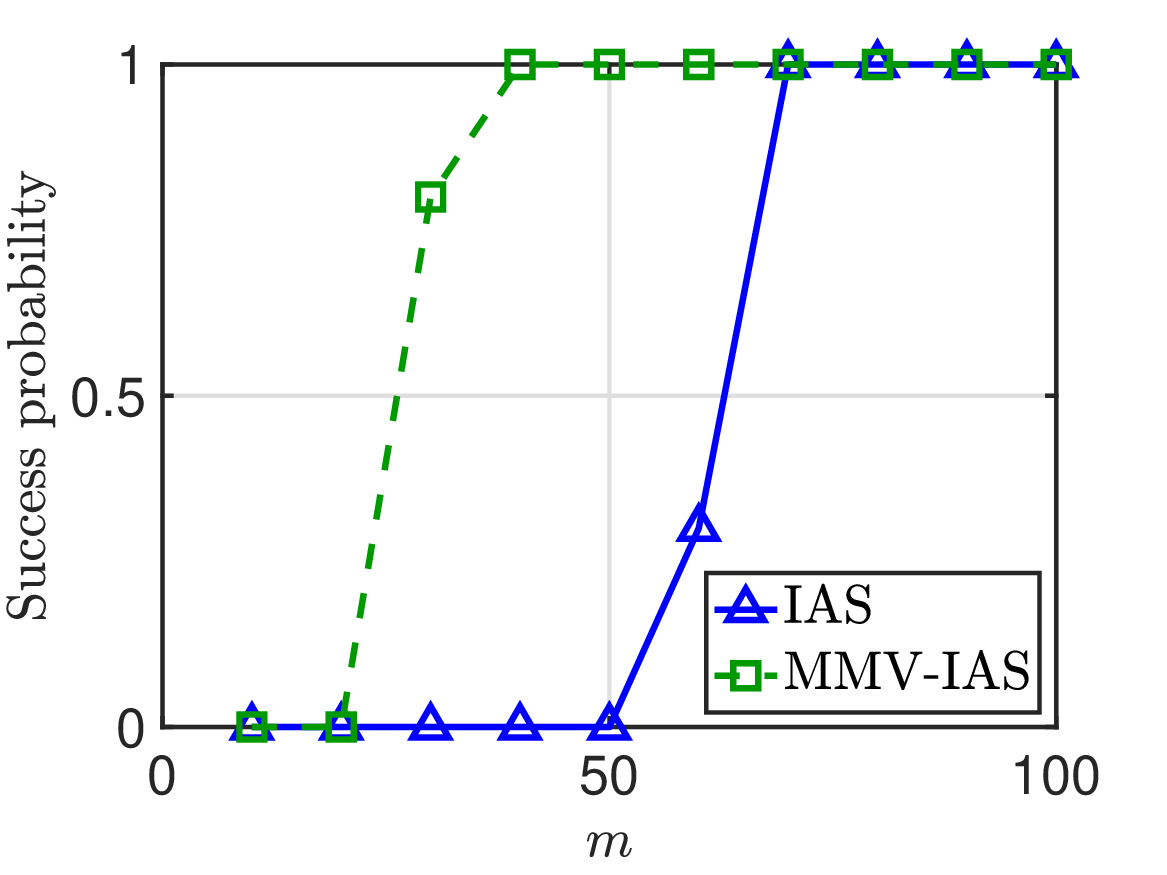} 
    	\caption{Success probability, $L=16$}
    	\label{fig:comparision_succ_IAS_C16}
  	\end{subfigure}%
  	\caption{ 
	Comparison of the average error and success probability for the sparsity-promoting IAS algorithm (blue triangles) and joint-sparsity-promoting MMV-IAS algorithm (green squares), both for $r=-1$. 
	We recover a signal of size $N=100$ with $s=20$ non-zero entries from an increasing number of measurements $m$. 
  	}
  	\label{fig:comparision_IAS}
\end{figure}

\begin{figure}[tb]
	\centering
  	\begin{subfigure}[b]{0.33\textwidth}
		\includegraphics[width=\textwidth]{%
      		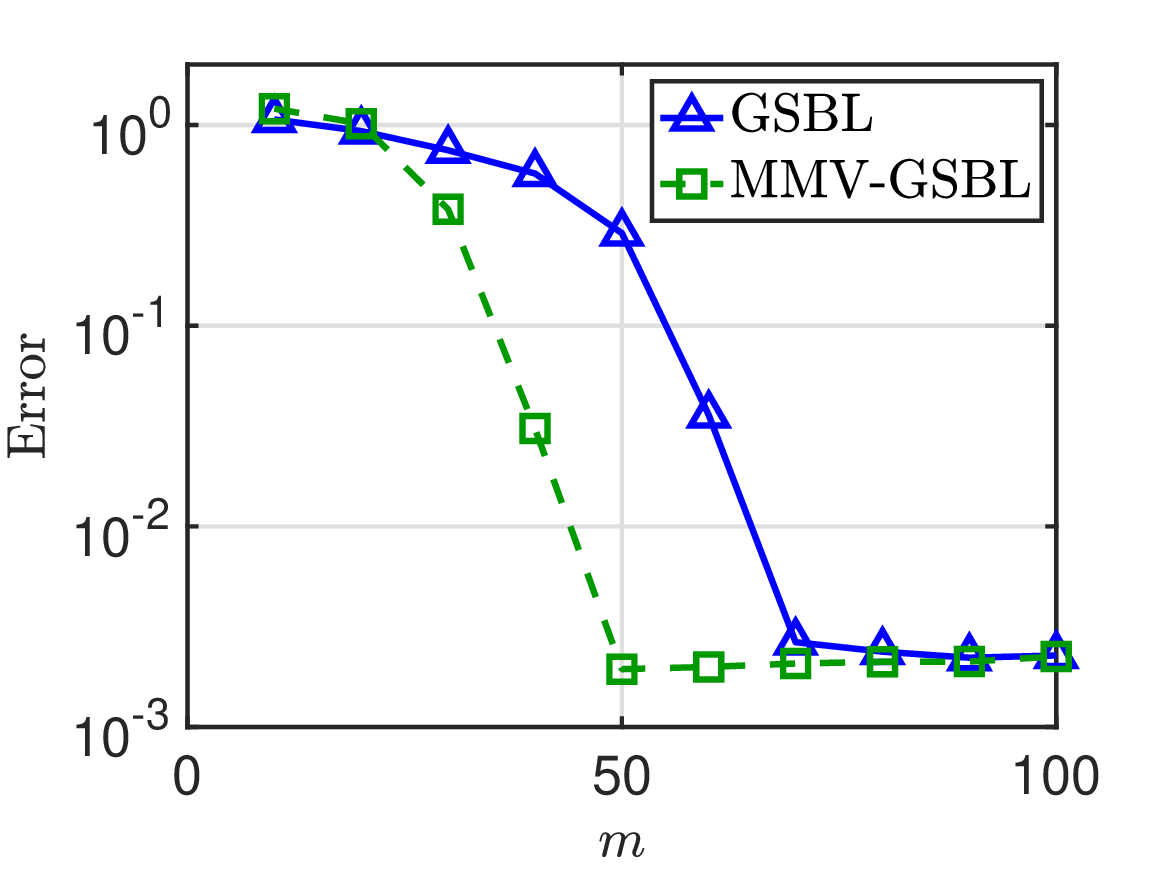} 
    	\caption{Average errors, $L=4$}
    	\label{fig:comparision_error_GSBL_C4}
  	\end{subfigure}%
  	\begin{subfigure}[b]{0.33\textwidth}
		\includegraphics[width=\textwidth]{%
      		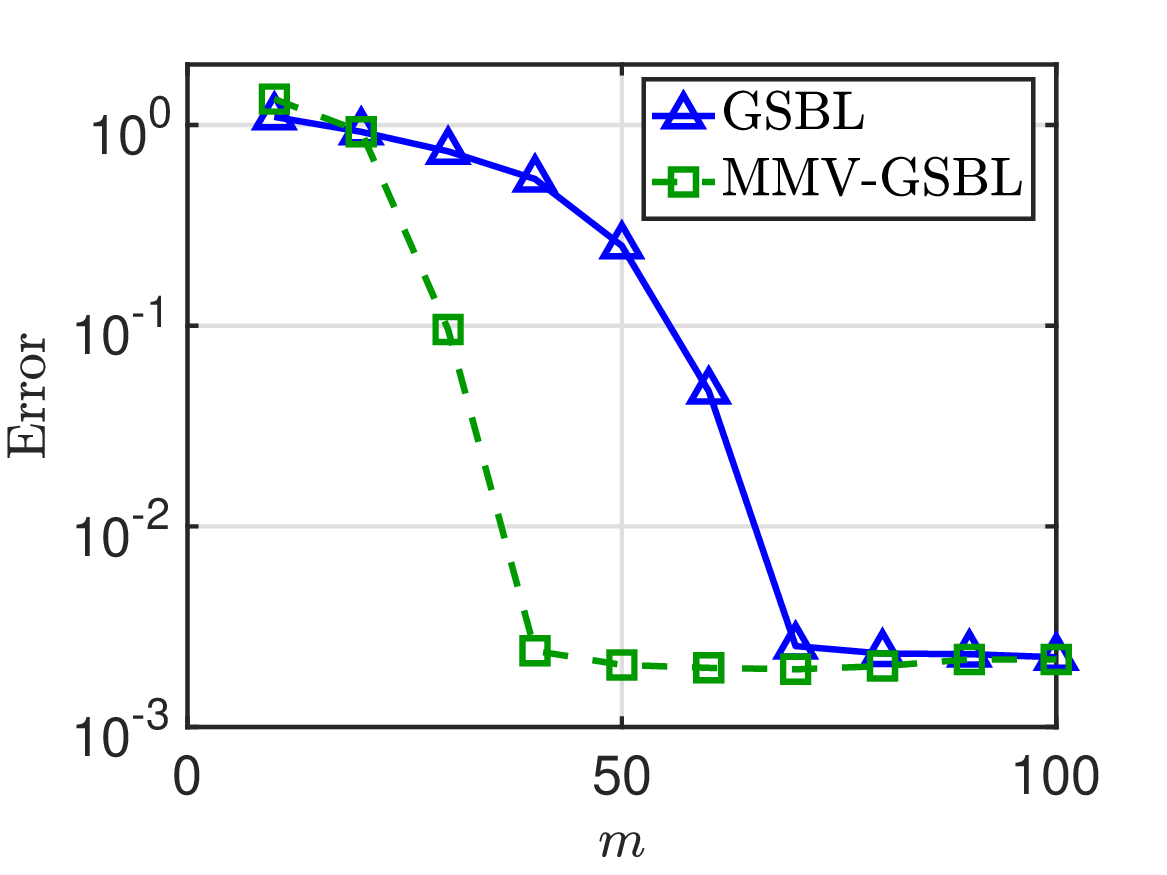} 
    	\caption{Average errors, $L=8$}
    	\label{fig:comparision_error_GSBL_C8}
  	\end{subfigure}%
	\begin{subfigure}[b]{0.33\textwidth}
		\includegraphics[width=\textwidth]{%
      		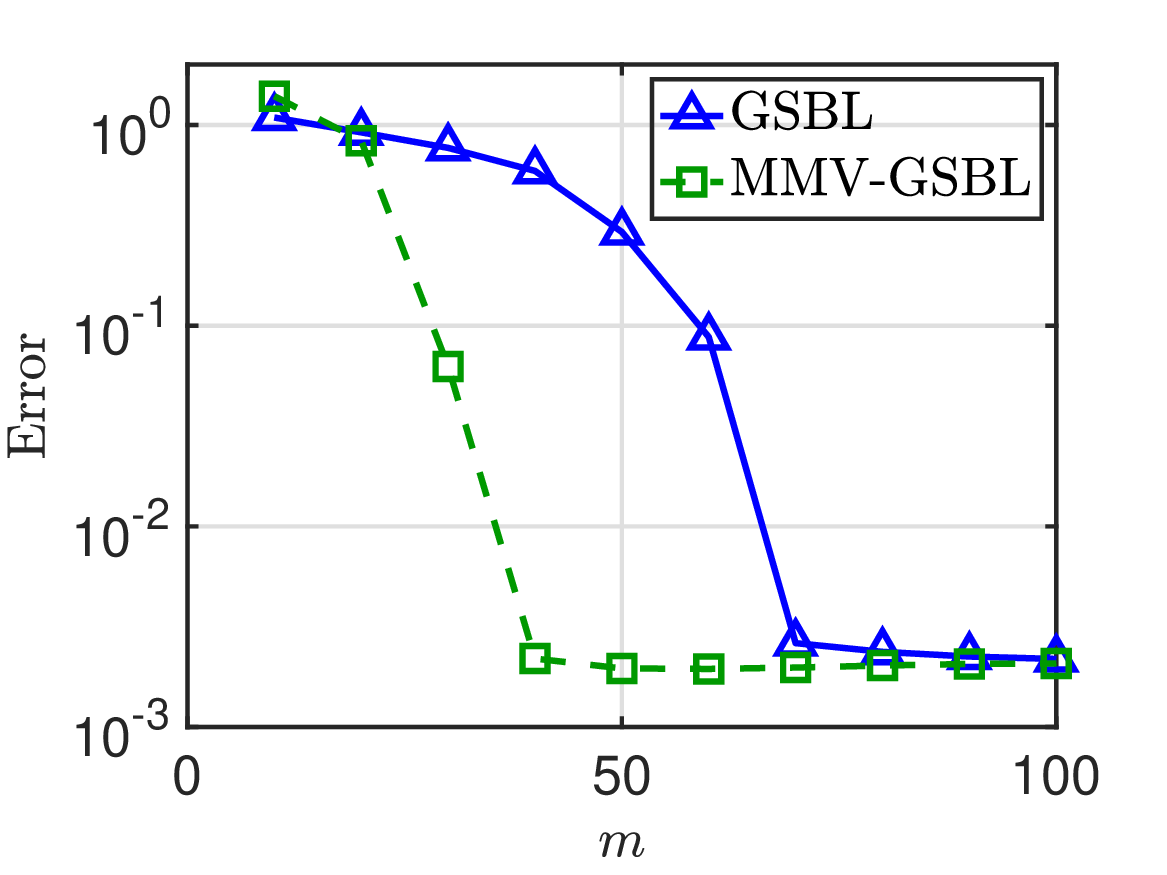} 
    	\caption{Average errors, $L=16$}
    	\label{fig:comparision_error_GSBL_C16}
  	\end{subfigure}%
	\\
	\begin{subfigure}[b]{0.33\textwidth}
		\includegraphics[width=\textwidth]{%
      		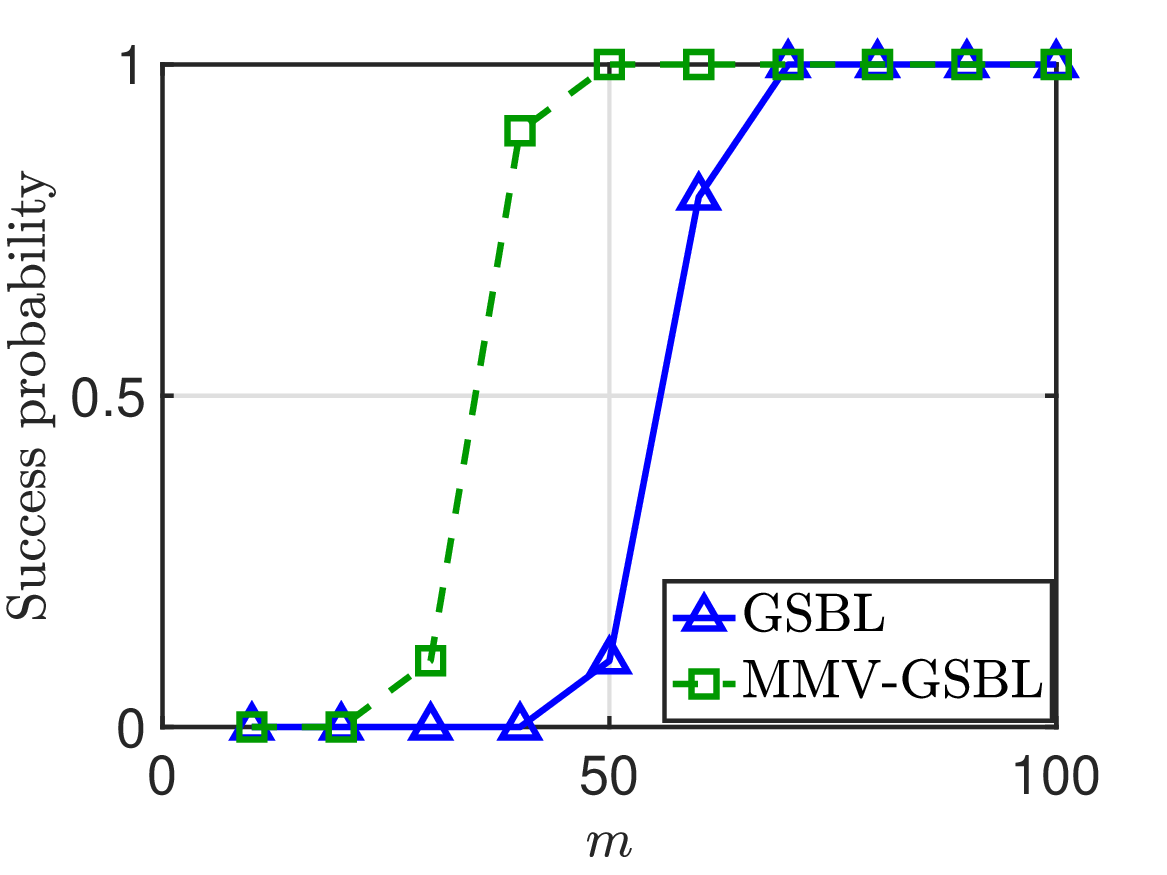} 
    	\caption{Success probability, $L=4$}
    	\label{fig:comparision_succ_GSBL_C4}
  	\end{subfigure}%
  	\begin{subfigure}[b]{0.33\textwidth}
		\includegraphics[width=\textwidth]{%
      		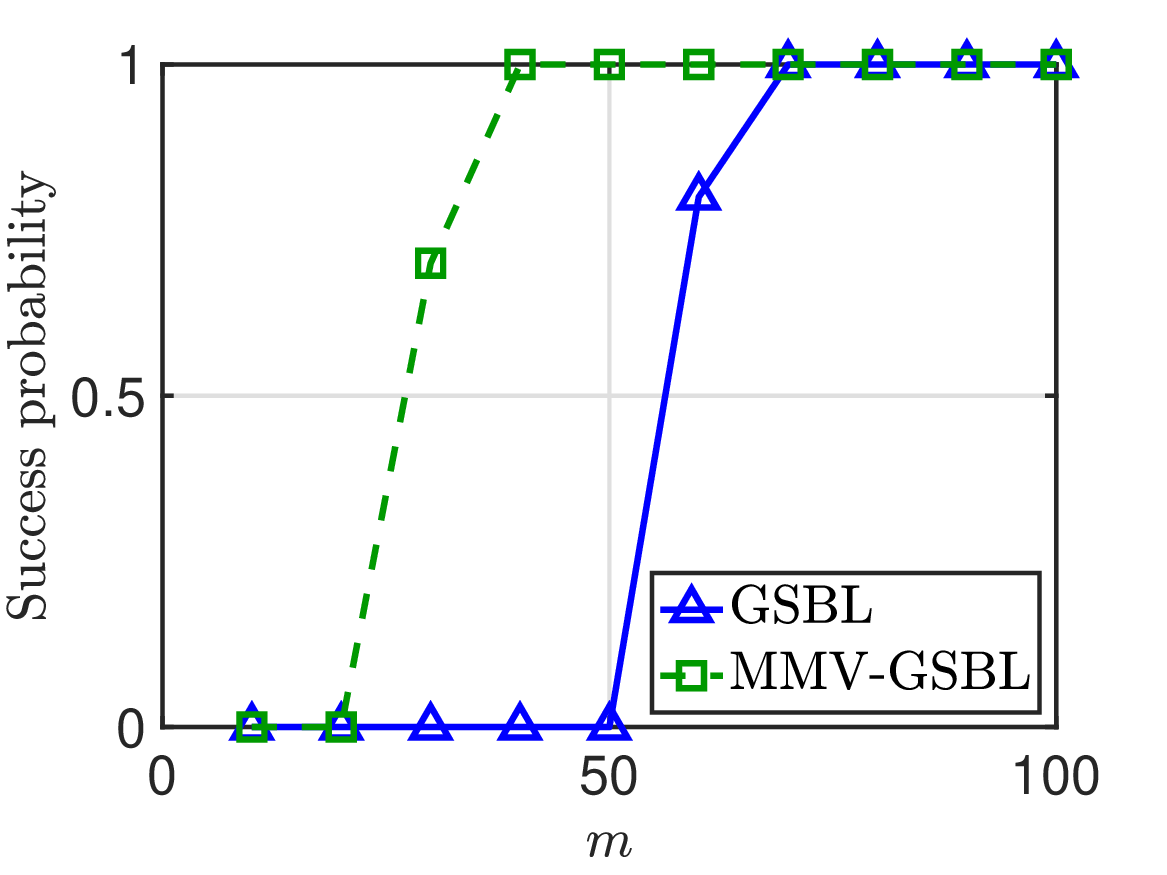} 
    	\caption{Success probability, $L=8$}
    	\label{fig:comparision_succ_GSBL_C8}
  	\end{subfigure}%
  	\begin{subfigure}[b]{0.33\textwidth}
		\includegraphics[width=\textwidth]{%
      		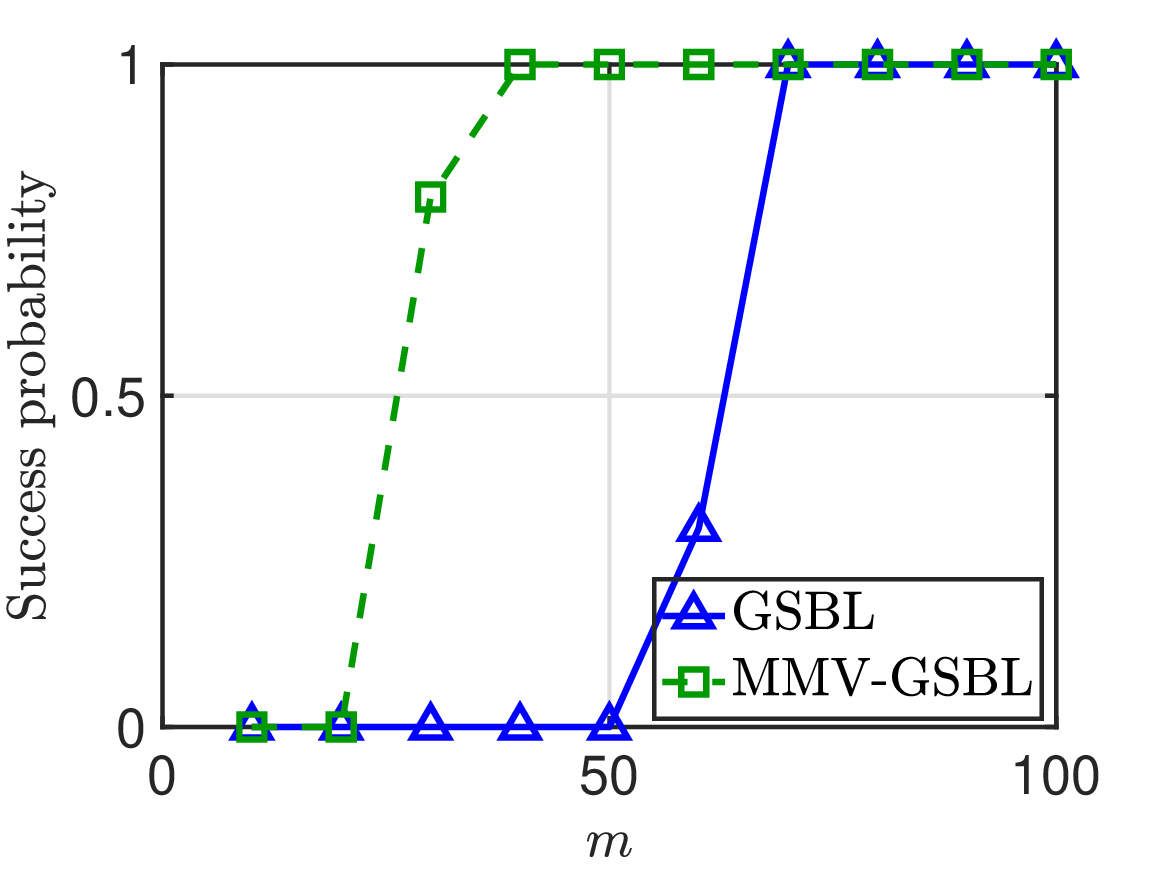} 
    	\caption{Success probability, $L=16$}
    	\label{fig:comparision_succ_GSBL_C16}
  	\end{subfigure}%
  	\caption{ 
	Comparison of the average error and success probability for the sparsity-promoting GSBL algorithm (blue triangles) and joint-sparsity-promoting MMV-GSBL algorithm (green squares). 
	We recover a signal of size $N=100$ with $s=20$ non-zero entries from an increasing number of measurements $m$. 
	}
  	\label{fig:comparision_GSBL}
\end{figure}  

\begin{figure}[tb]
	\centering
  	\begin{subfigure}[b]{0.49\textwidth}
		\includegraphics[width=\textwidth]{%
      		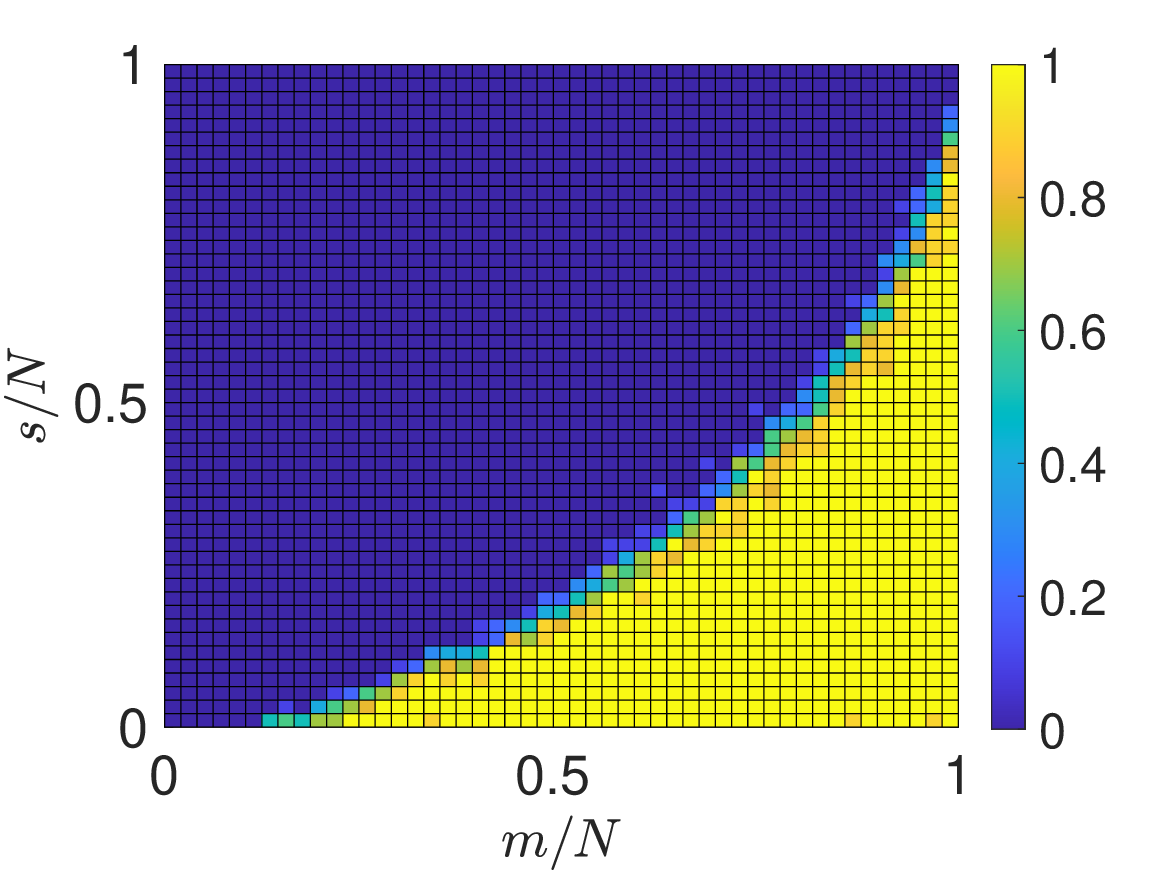} 
    	\caption{IAS algorithm (separately recovered)}
    	\label{fig:phase_IAS_rm1_L4}
  	\end{subfigure}
  	\begin{subfigure}[b]{0.49\textwidth}
		\includegraphics[width=\textwidth]{%
      		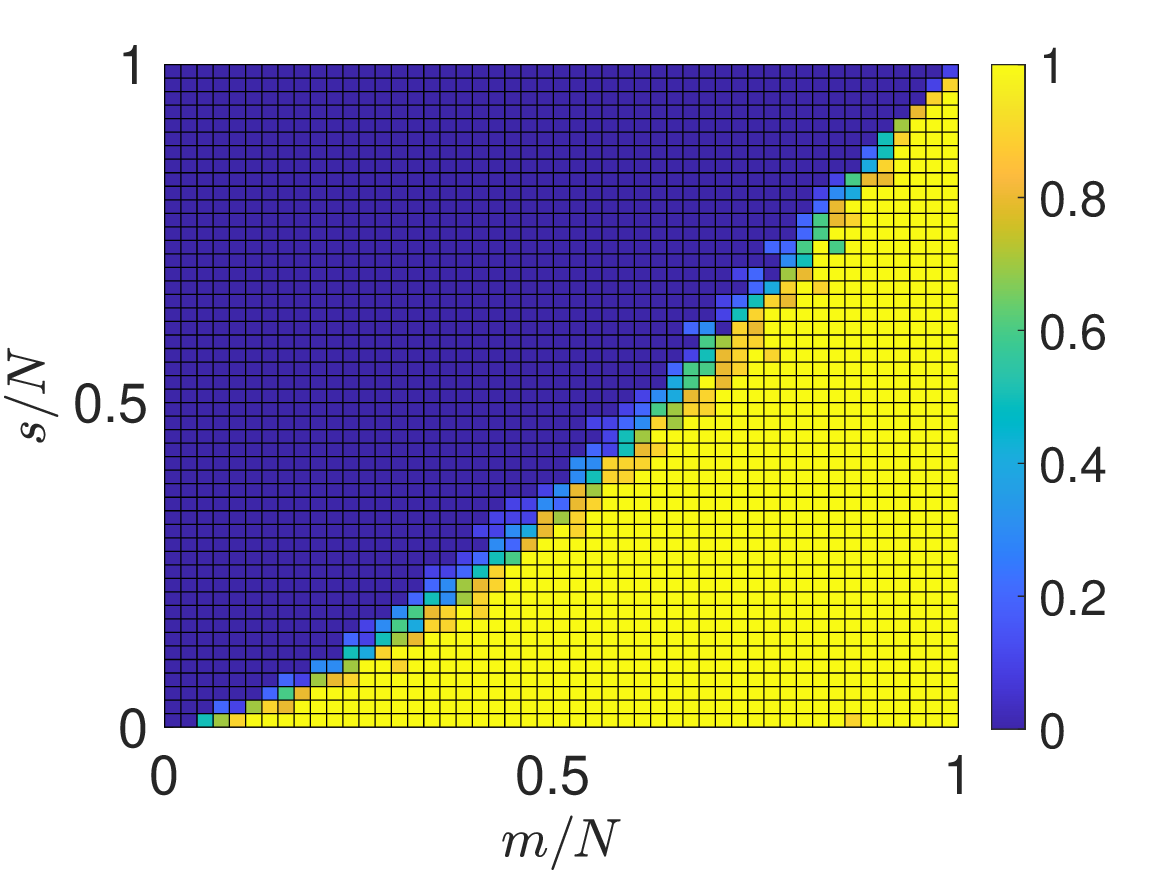} 
    	\caption{MMV-IAS algorithm, $L=4$}
    	\label{fig:phase_MMVIAS_rm1_L4}
  	\end{subfigure}
	\\
  	\begin{subfigure}[b]{0.49\textwidth}
		\includegraphics[width=\textwidth]{%
      		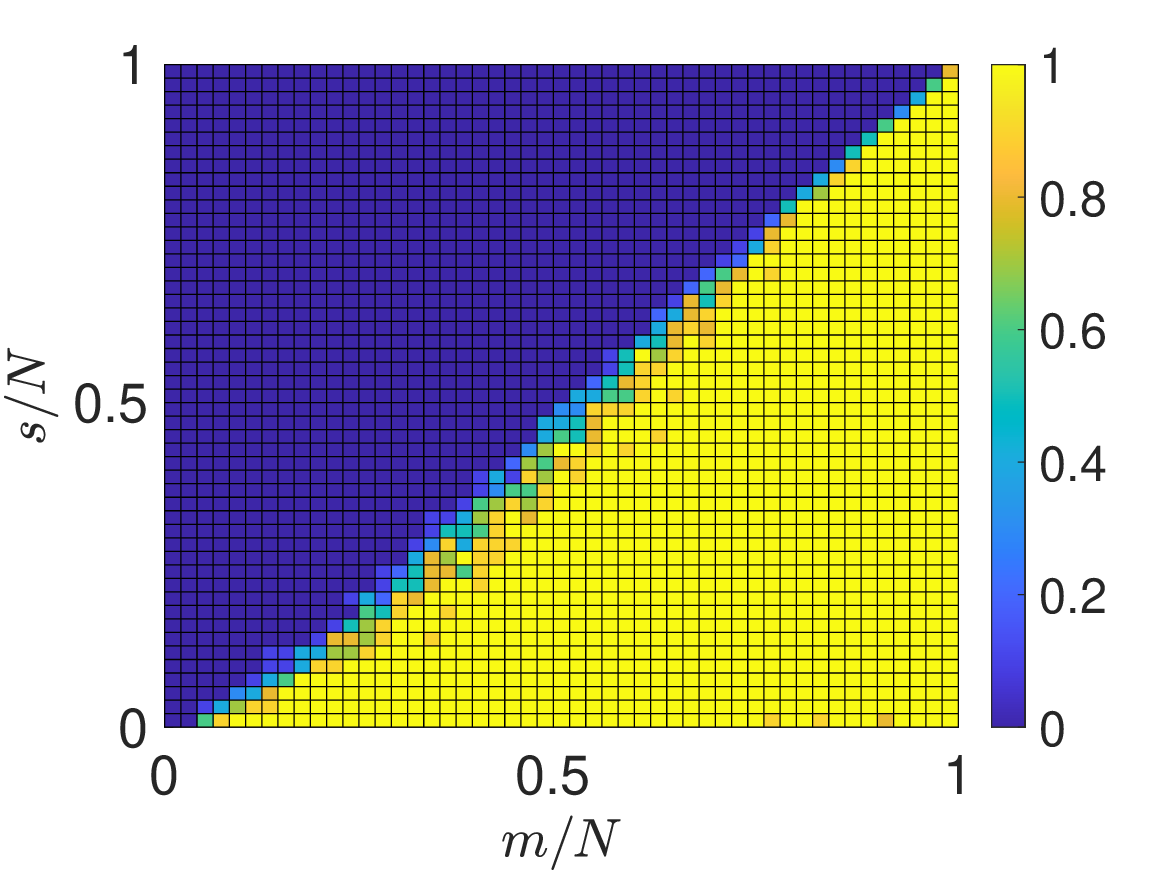} 
    	\caption{MMV-IAS algorithm, $L=8$}
    	\label{fig:phase_MMVIAS_L8}
  	\end{subfigure}
  	\begin{subfigure}[b]{0.49\textwidth}
		\includegraphics[width=\textwidth]{%
      		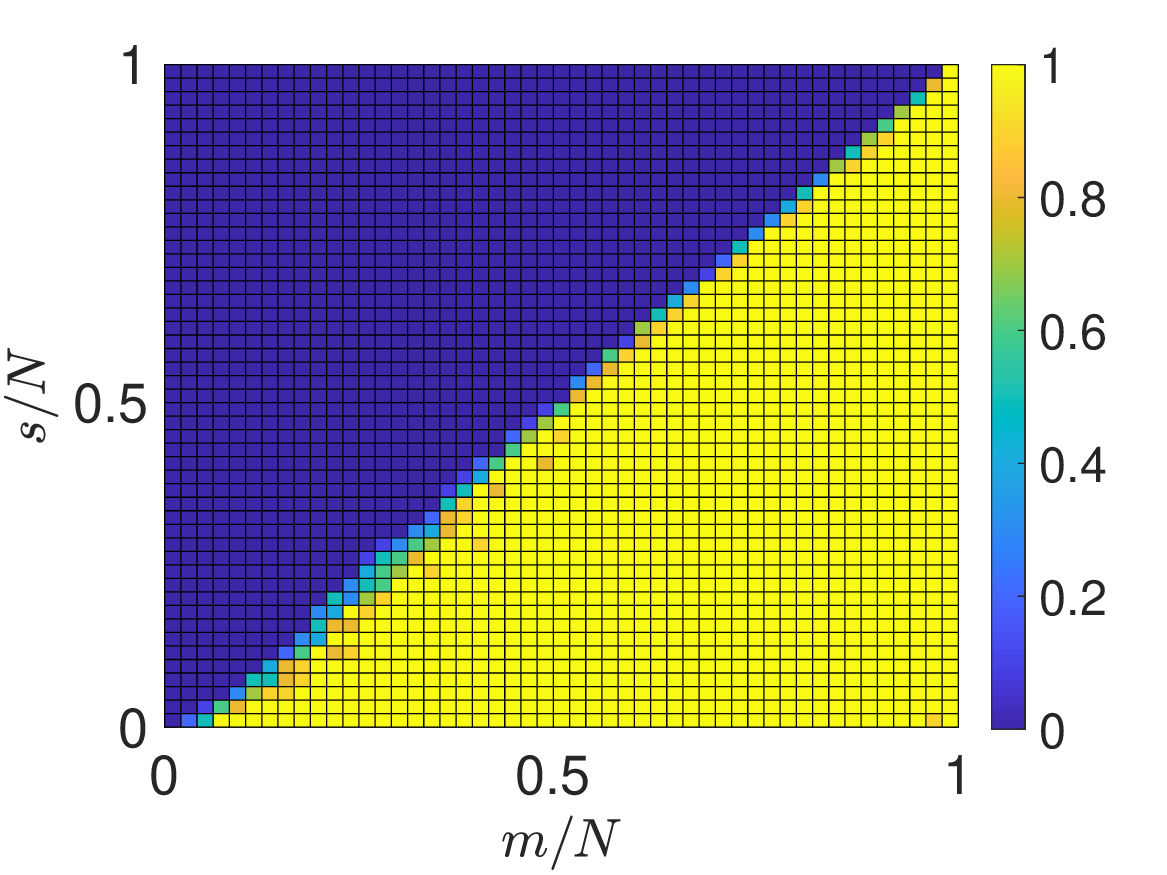} 
    	\caption{MMV-IAS algorithm, $L=16$}
    	\label{fig:phase_MMVIAS_L16}
  	\end{subfigure}
  	\caption{ 
	Phase transition diagrams for the IAS and MMV-IAS algorithm for $r=-1$ and $L = 4,8,16$. 
	The diagrams show the success probability for values $1 \leq s \leq N$ and $1 \leq M \leq N$. 
  	}
  	\label{fig:phase}
\end{figure} 

The performance of the proposed joint-sparsity-promoting MMV-IAS and MMV-GSBL algorithms is evaluated and compared to the existing IAS and GSBL algorithms in \cref{fig:comparision_IAS,fig:comparision_GSBL}. 
These figures report the average errors and ESP for different numbers of MMVs ($L = 4, 8, 16$). 
For brevity, we only report on the IAS and MMV-IAS results for $r=-1$. 
The results show that the proposed algorithms outperform the existing ones regarding average errors and ESP in most cases. 
As the number of MMVs increases, the superiority of the proposed algorithms becomes more pronounced. 
For instance, when $L=16$, the MMV-IAS algorithm requires only around $40$ measurements per signal for successful recovery, while the IAS algorithm requires around $70$. 
The phase transition plots in \cref{fig:phase} further demonstrate the improved performance of the MMV-IAS algorithm, which exhibits a phase transition close to the optimal $m = s$ line. 
The phase transition profiles for the MMV-GSBL algorithm are similar to the MMV-IAS algorithm but are not included here for brevity.

\subsection{Application to parallel magnetic resonance imaging} 
\label{sub:parallel_MRI} 

We next apply the proposed MMV-IAS and MMV-GSBL algorithms to a parallel MRI test problem. 
Parallel MRI is a multi-sensor acquisition system that uses multiple coils to simultaneously acquire image measurements for recovery. 
Details on parallel MRI can be found in \cite{guerquin2011realistic,chun2015efficient,chun2017compressed,adcock2019joint}. 
A standard discrete data model for parallel MRI is the following: 
Let $\mathbf{x} \in \C^N$ be the vectorized image to be recovered and $L$ be the number of coils. 
For the $l$th coil, the measurements acquired are 
\begin{equation}\label{eq:pMRI_model}
	\mathbf{y}_l = P_{\Omega_l} F \mathbf{x} + \mathbf{e}_l, 
\end{equation} 
where $F \in \C^{N \times N}$ is the discrete Fourier transform (DFT) matrix, $P_{\Omega_l} \in \C^{M \times N}$ is a sampling operator that selects the rows of $F$ corresponding to the frequencies in $\Omega_l$, and $\mathbf{e}_l \in \C^M$ is noise. 
The image measurement acquired by each of the $L$ coils is intrinsic to the particular coil. 
A typical sampling procedure for parallel MRI is to use data taken as radial line sampling in the Fourier space. 
\cref{fig:pMRI_sampling} illustrates the radial sampling maps corresponding to four different coils. 

\begin{figure}[tb]
	\centering
  	\includegraphics[width=\textwidth]{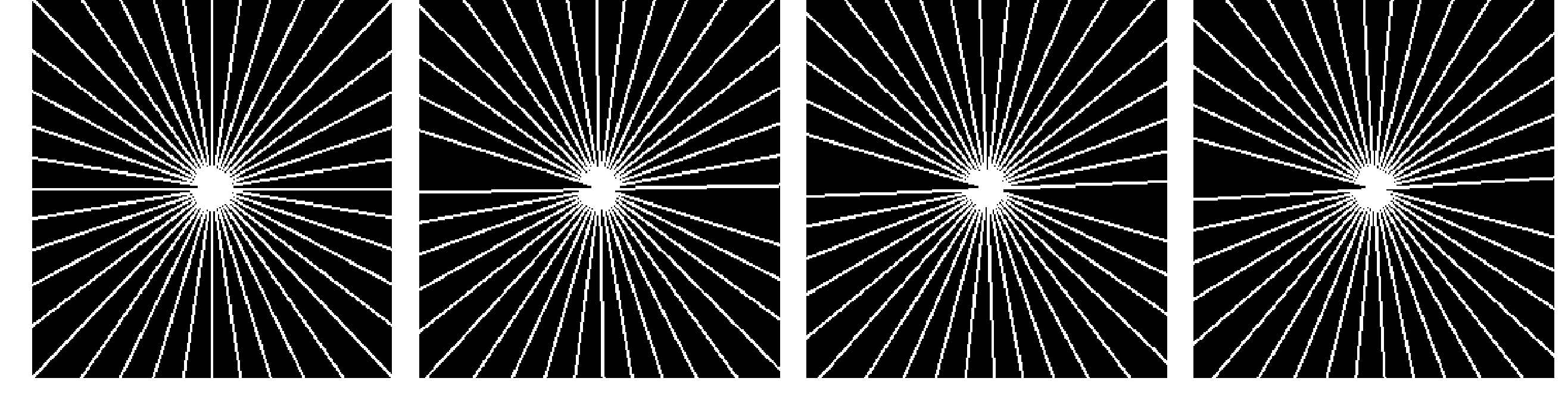} 
  	\caption{Four different radial sampling maps}
  	\label{fig:pMRI_sampling}
\end{figure} 

\begin{figure}[tb]
	\centering
  	\begin{subfigure}[b]{0.32\textwidth}
		\includegraphics[width=\textwidth]{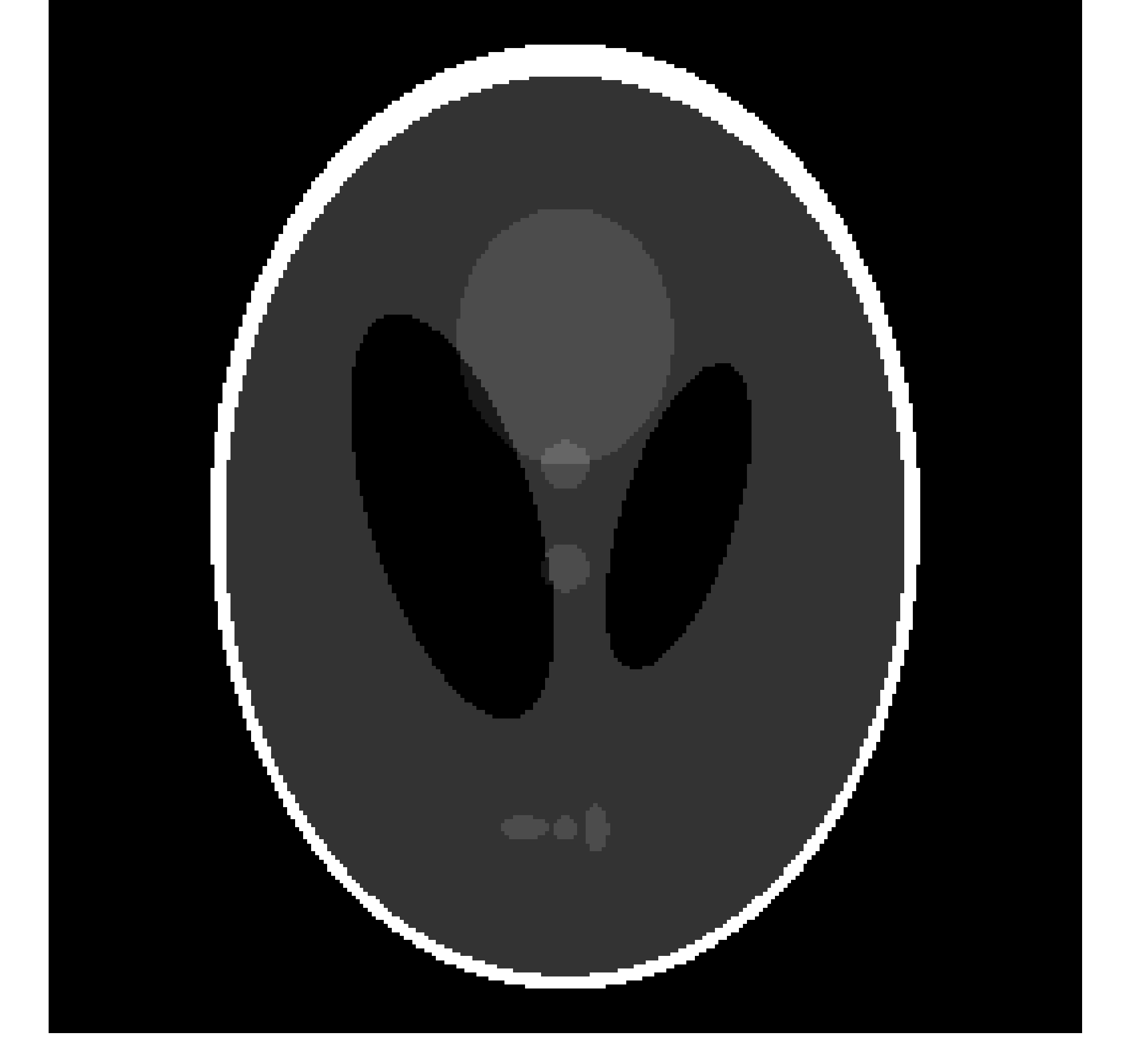} 
		\caption{Reference image}
    		\label{fig:pMRI_ref} 
  	\end{subfigure}
	\begin{subfigure}[b]{0.32\textwidth}
		\includegraphics[width=\textwidth]{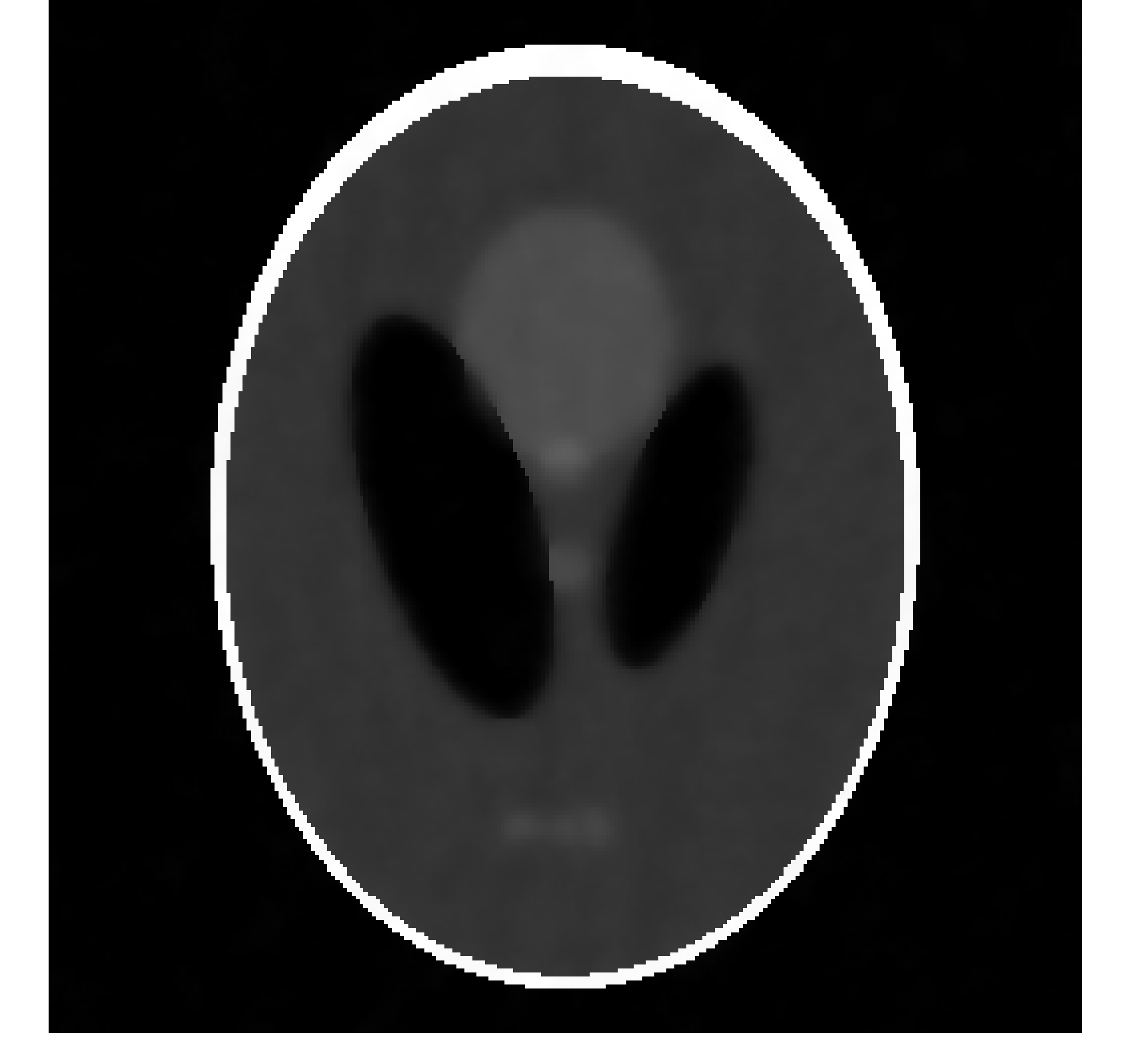} 
		\caption{IAS reconstruction}
		\label{fig:pMRI_coil1_IAS} 
  	\end{subfigure}
	\begin{subfigure}[b]{0.32\textwidth}
		\includegraphics[width=\textwidth]{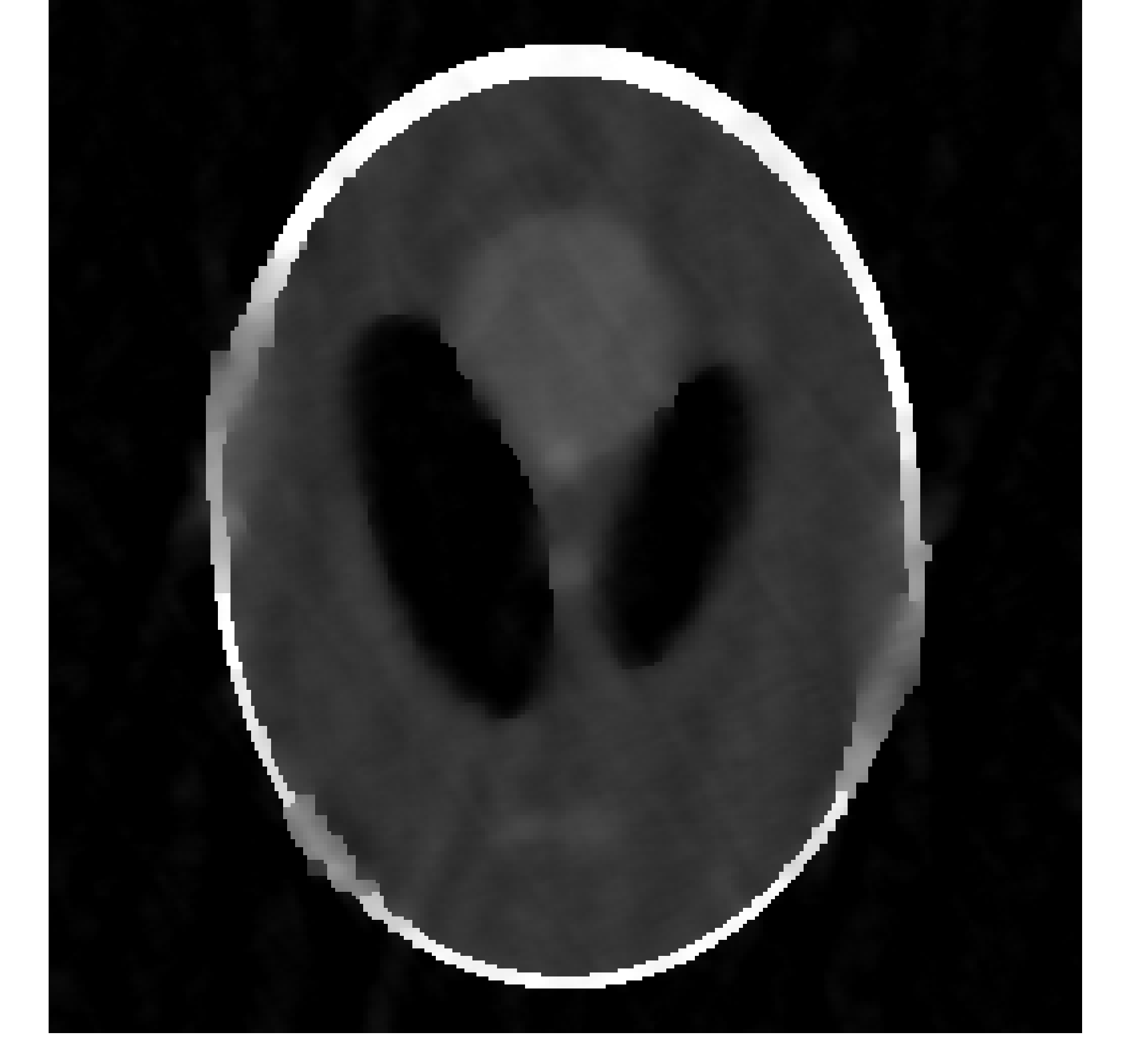} 
		\caption{GSBL reconstruction}
		\label{fig:pMRI_coil1_GSBL}
  	\end{subfigure}
	\\ 
	\begin{subfigure}[b]{0.32\textwidth}
		\includegraphics[width=\textwidth]{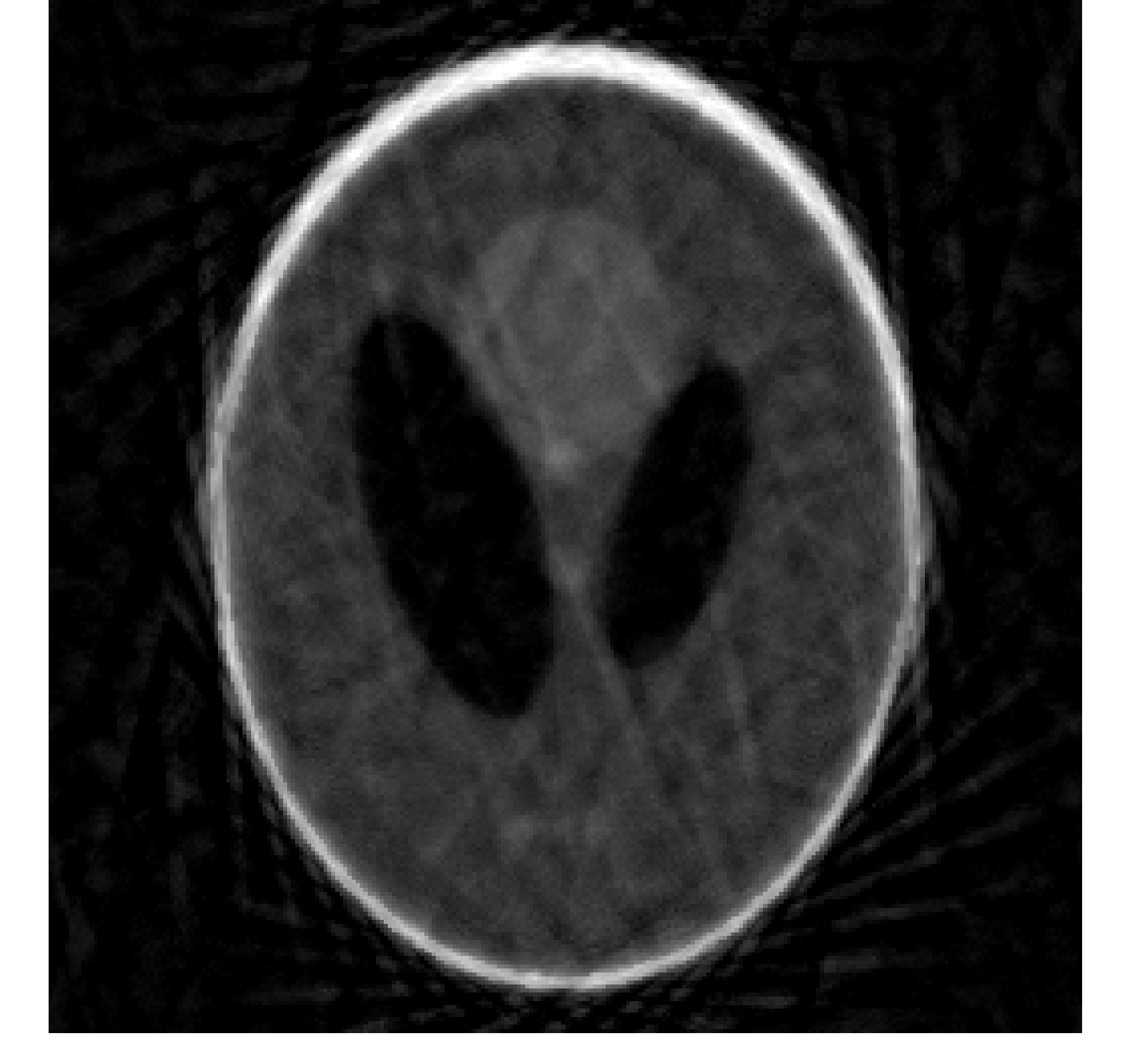} 
		\caption{Least squares reconstruction}
		\label{fig:pMRI_coil1_LS} 
  	\end{subfigure}
	\begin{subfigure}[b]{0.32\textwidth}
		\includegraphics[width=\textwidth]{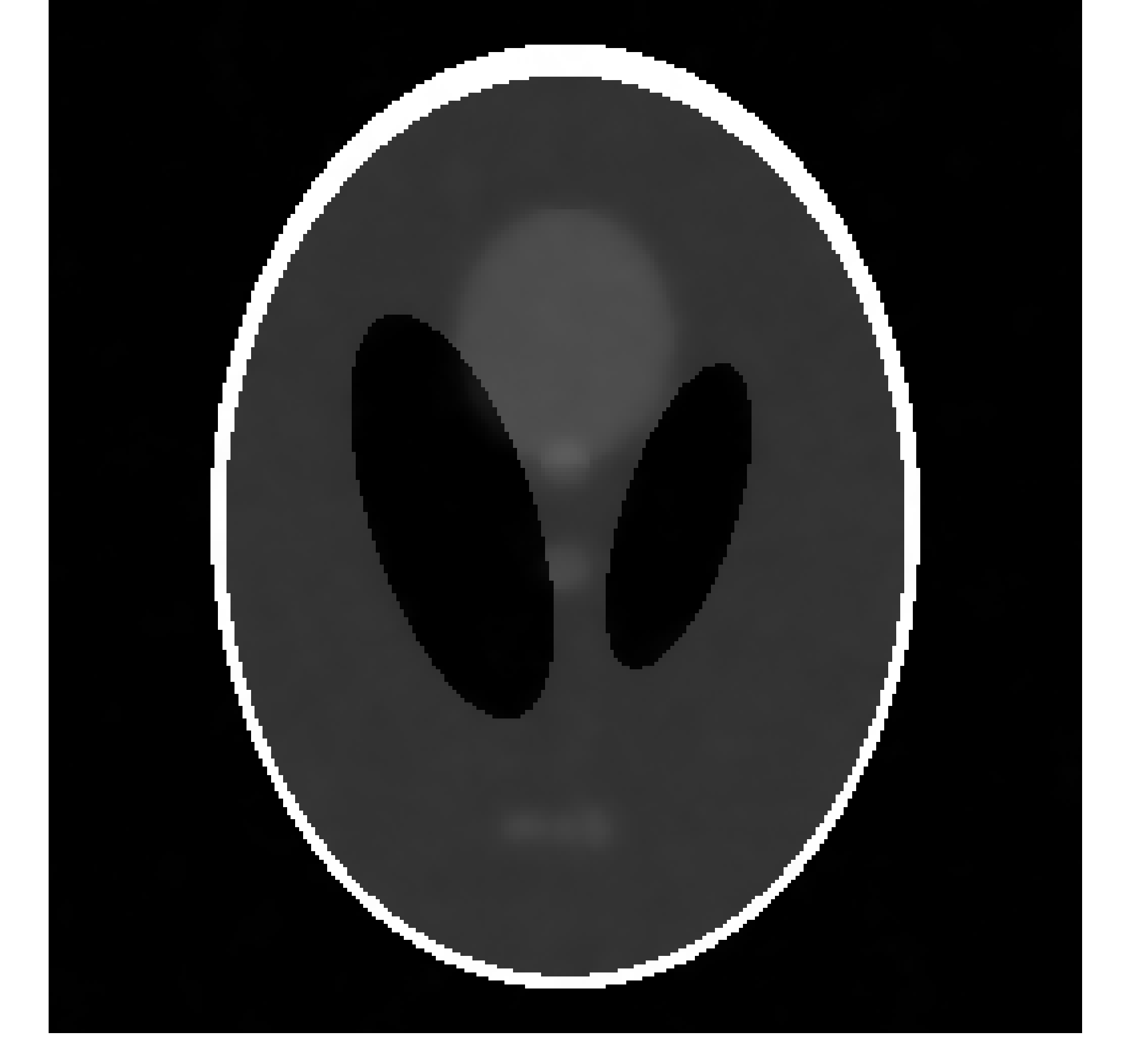} 
		\caption{MMV-IAS reconstruction}
		\label{fig:pMRI_coil1_MMVIAS} 
  	\end{subfigure}
	\begin{subfigure}[b]{0.32\textwidth}
		\includegraphics[width=\textwidth]{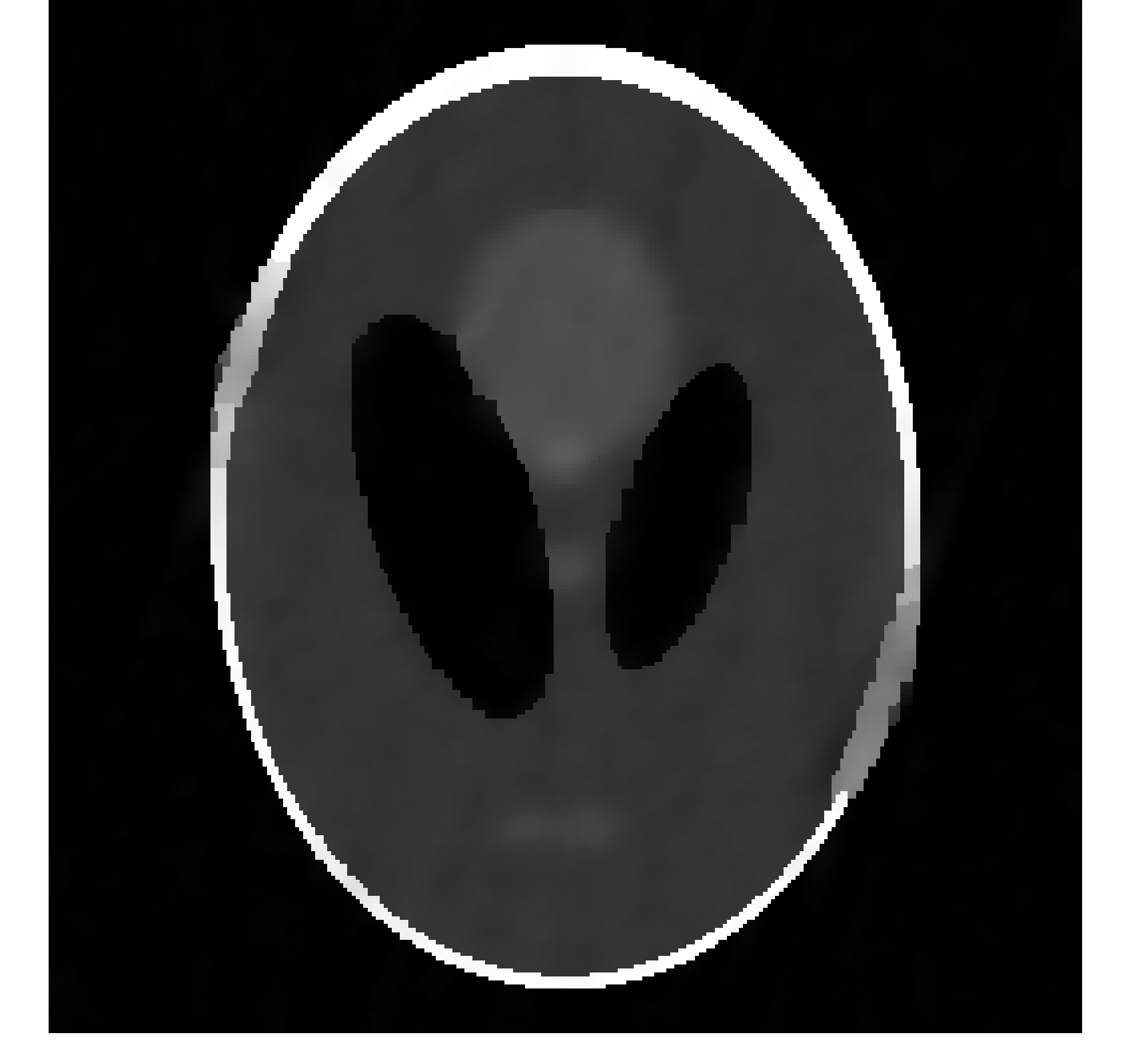} 
		\caption{MMV-GSBL reconstruction}
		\label{fig:pMRI_coil1_MMVGSBL} 
  	\end{subfigure}
	\caption{ 
	Reference image and the reconstructed first coil images
  	}
  	\label{fig:pMRI_coil_images}
\end{figure} 

\begin{figure}[tb]
	\centering
  	\begin{subfigure}[b]{0.32\textwidth}
		\includegraphics[width=\textwidth]{%
      		figures/pMRI_ref} 
		\caption{Reference image}
		\label{fig:pMRI_overall_ref} 
  	\end{subfigure}
	\begin{subfigure}[b]{0.32\textwidth}
		\includegraphics[width=\textwidth]{%
      		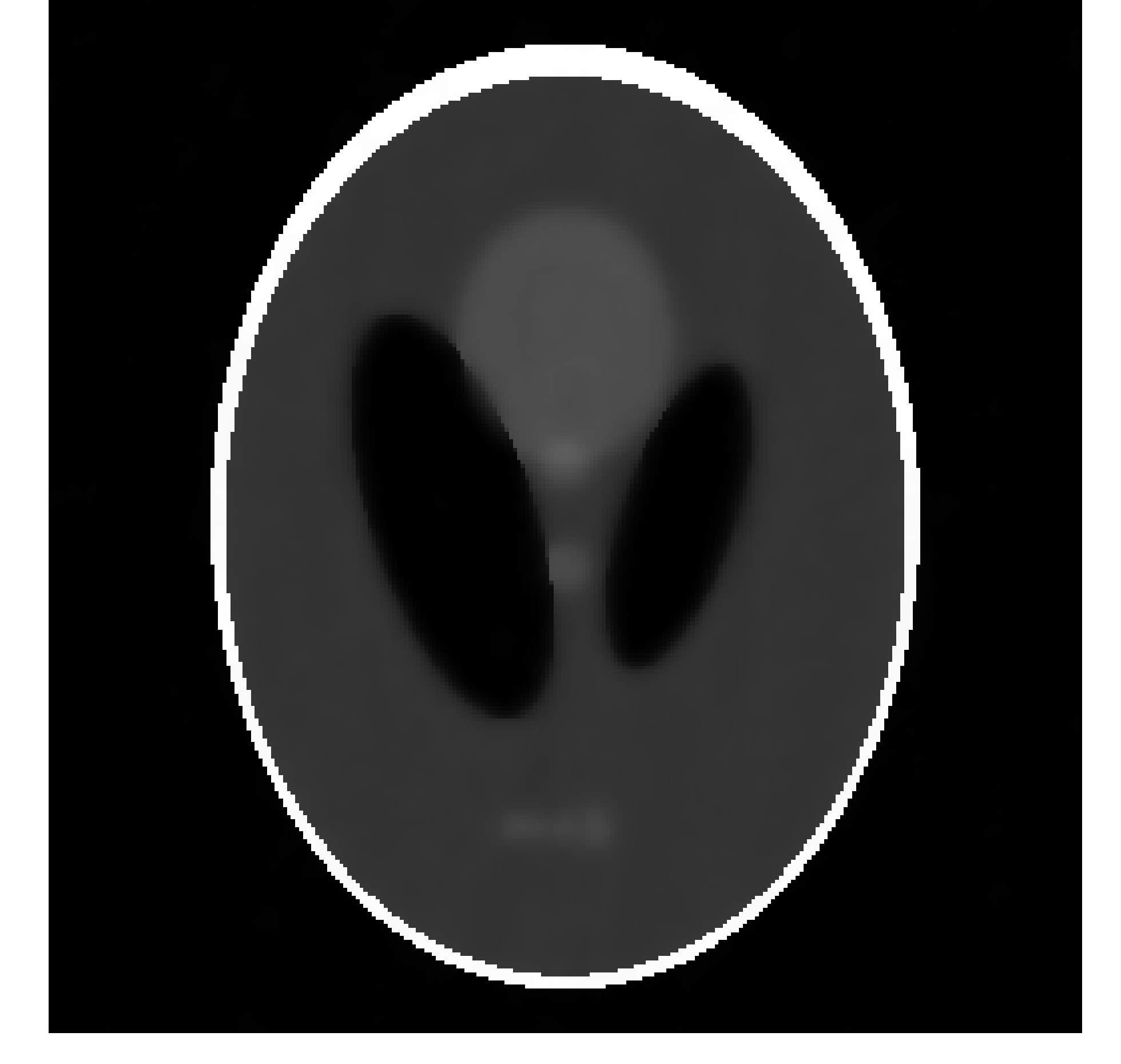} 
		\caption{IAS reconstruction}
		\label{fig:pMRI_overall_IAS}
  	\end{subfigure}
	\begin{subfigure}[b]{0.32\textwidth}
		\includegraphics[width=\textwidth]{%
      		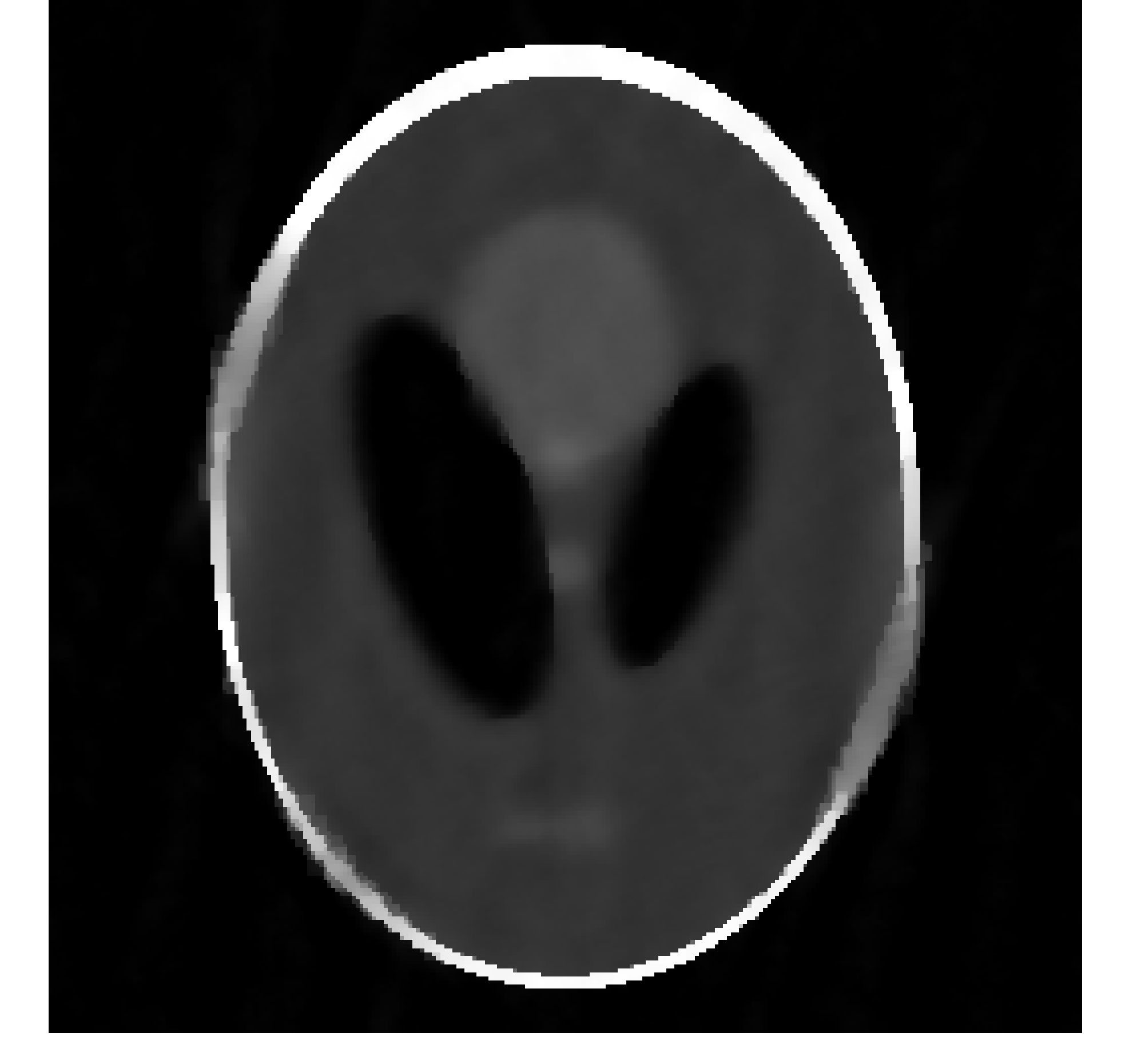} 
		\caption{GSBL reconstruction}
		\label{fig:pMRI_overall_GSBL}
  	\end{subfigure}
	\\ 
	\begin{subfigure}[b]{0.32\textwidth}
		\includegraphics[width=\textwidth]{%
      		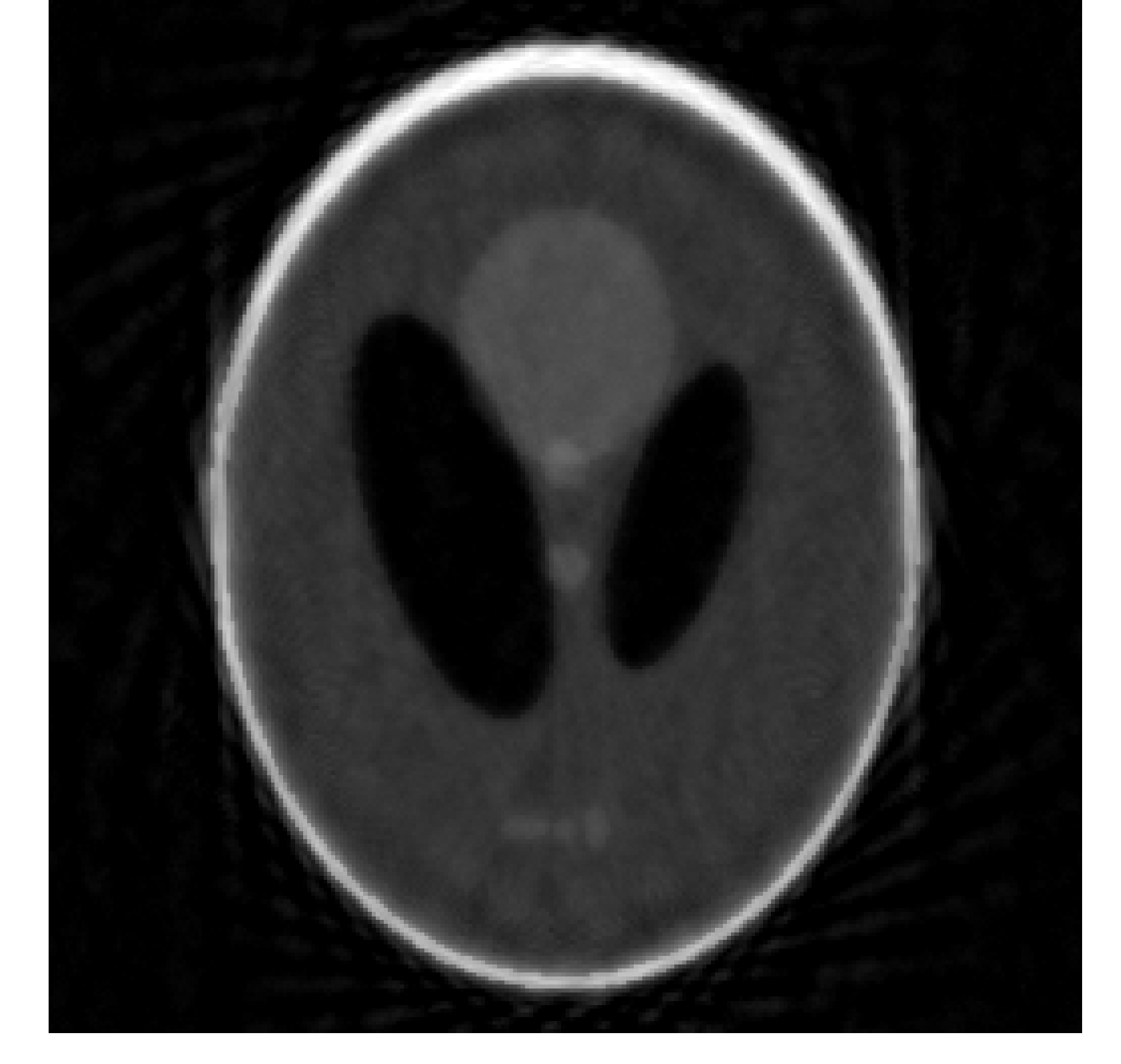} 
		\caption{Least squares reconstruction}
		\label{fig:pMRI_overall_LS}
  	\end{subfigure}
	\begin{subfigure}[b]{0.32\textwidth}
		\includegraphics[width=\textwidth]{%
      		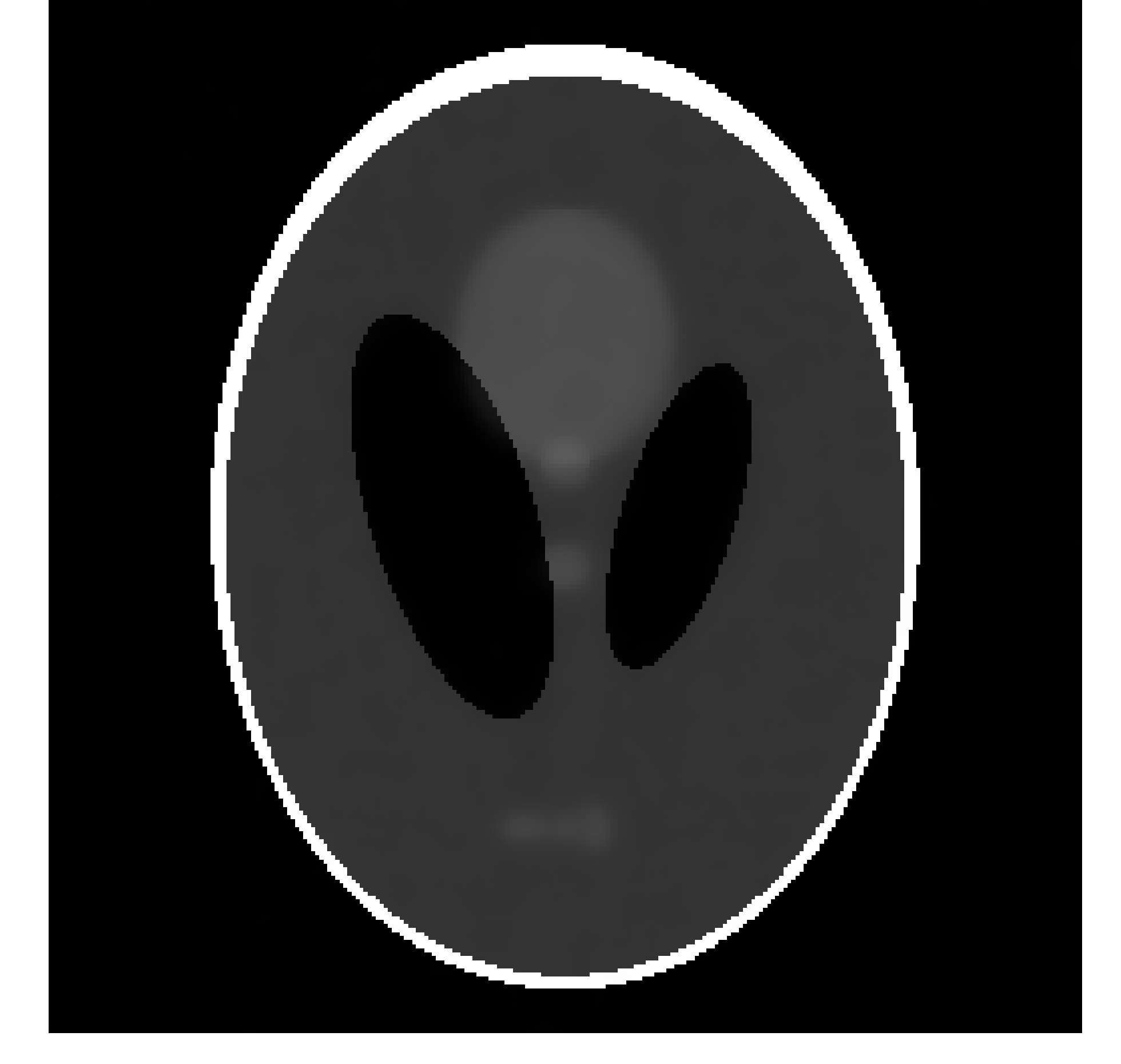} 
		\caption{MMV-IAS reconstruction} 
		\label{fig:pMRI_overall_MMVIAS} 
  	\end{subfigure}
	\begin{subfigure}[b]{0.32\textwidth}
		\includegraphics[width=\textwidth]{%
      		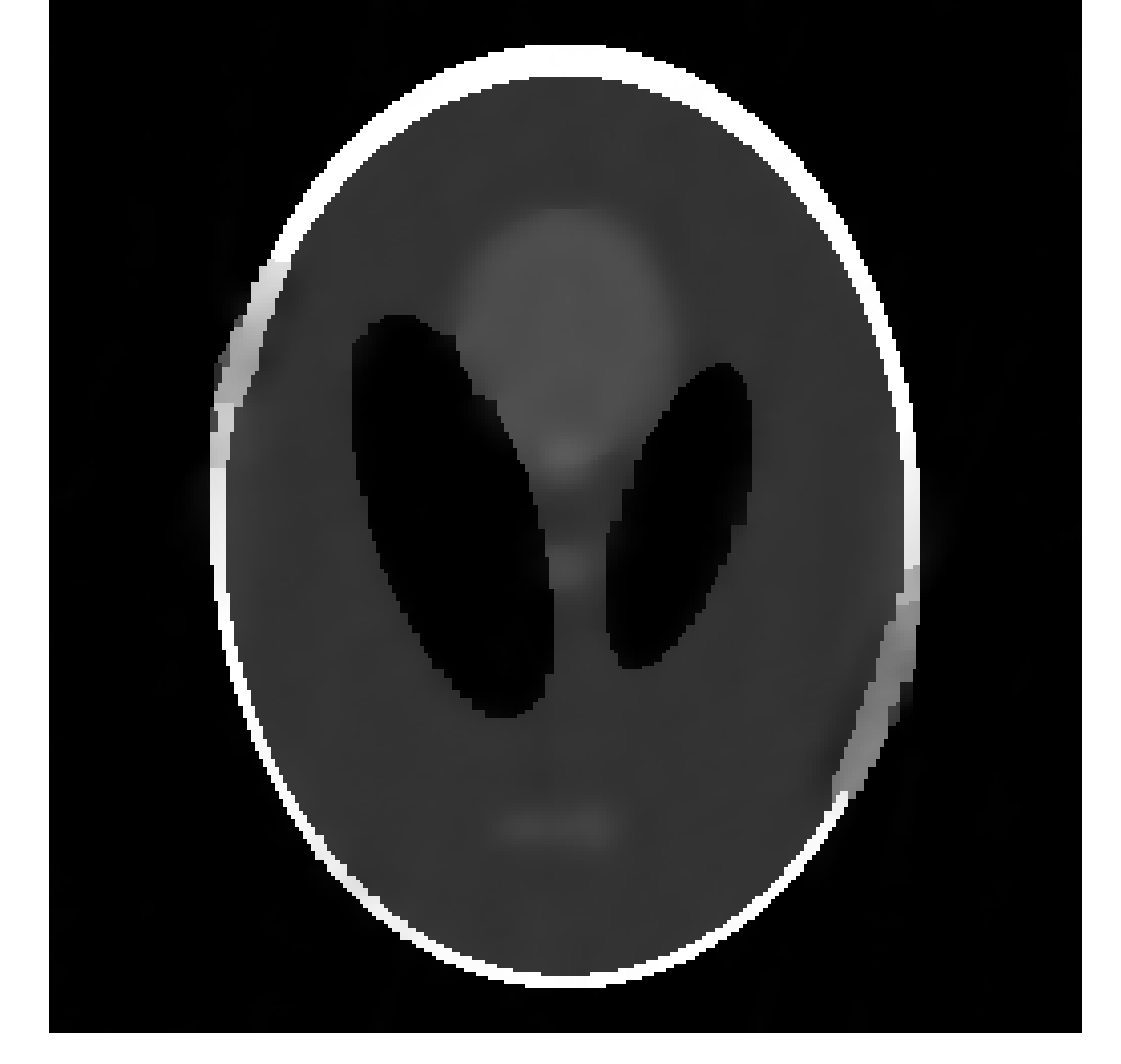} 
		\caption{MMV-GSBL reconstruction} 
		\label{fig:pMRI_overall_MMVGSBL} 
  	\end{subfigure}
  	\caption{ 
	Reference image and the reconstructed overall images. 
	For all methods, $20$ lines/angles (corresponding to around $16\%$ sampling of the k-space), $\sigma^2 = 10^{-3}$, and $L=4$ coils were used.
  	}
  	\label{fig:pMRI_overall_images}
\end{figure} 

\begin{figure}[tb]
	\centering
	\begin{subfigure}[b]{0.495\textwidth}
		\includegraphics[width=\textwidth]{%
      		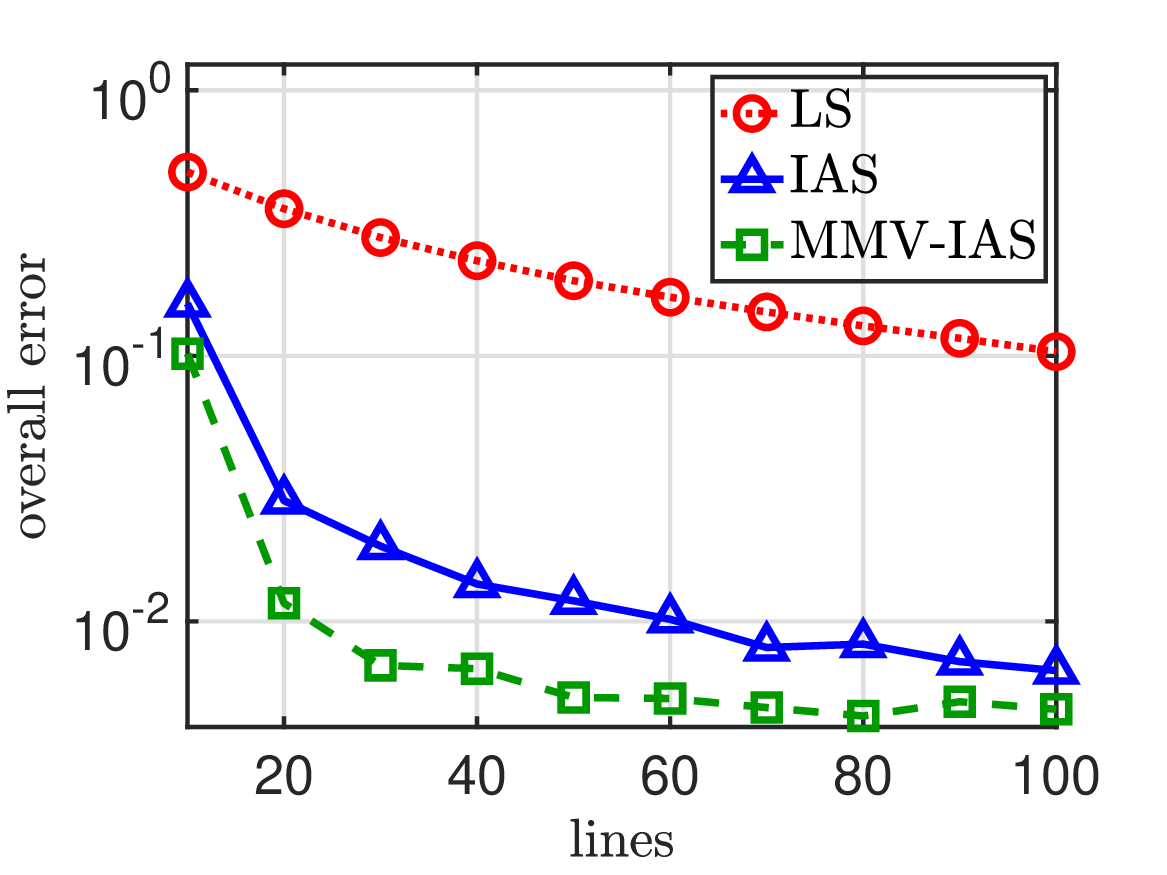} 
		\caption{Relative overall error, IAS \& MMV-IAS} 
		\label{fig:error_overall_IAS}
  	\end{subfigure}
	\begin{subfigure}[b]{0.495\textwidth}
		\includegraphics[width=\textwidth]{%
      		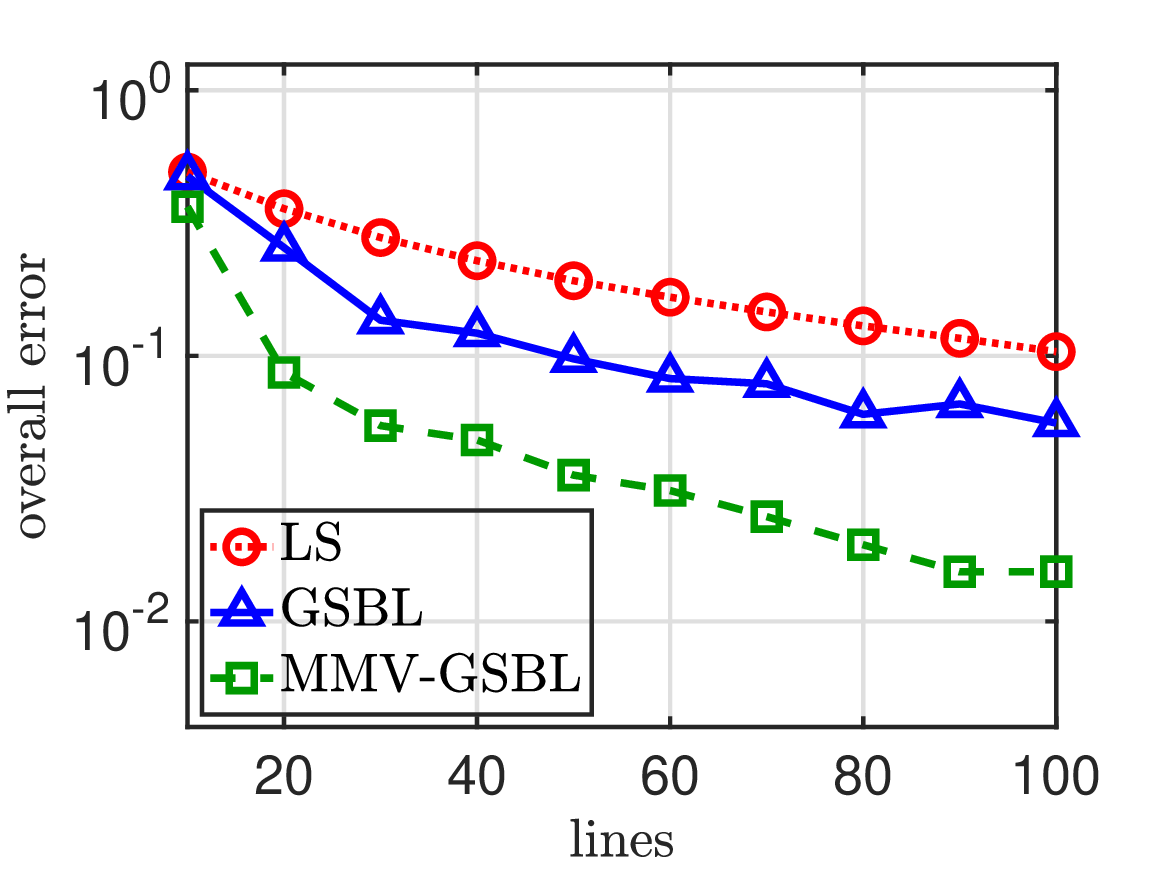} 
		\caption{Relative overall error, GSBL \& MMV-GSBL} 
		\label{fig:error_overall_GSBL}
  	\end{subfigure}
  	\caption{ 
	Relative error of the recovered overall image using the least squares (LS) approach, the existing IAS/GSBL algorithm, and the proposed MMV-IAS/GSBL method.
	In all cases, we used $L=4$ coils, noise variance $\sigma^2 = 10^{-3}$, and a varying number of lines. 
  	}
  	\label{fig:pMRI_errors}
\end{figure} 

Many techniques have been proposed for parallel MRI, see \cite{chun2015efficient} and references therein. 
Here we focus on the coil-by-coil approach, first computing the approximate coil images $\mathbf{\hat{x}}_{1:L}$ from \cref{eq:pMRI_model} and then computing an approximation $\mathbf{\hat{x}}$ to the overall image by considering the average of the coil images, i.e., $\mathbf{\hat{x}} = (\mathbf{\hat{x}}_{1} + \dots + \mathbf{\hat{x}}_{L})/L$. 
In our experiment, we compare the recovery of the $256 \times 256$ Shepp--Logan phantom image in \cref{fig:pMRI_ref} using the least-squares (LS) approach, the existing IAS/GSBL algorithm, and the proposed MMV-GSBL/IAS algorithm to recover the coil images. 
For brevity, we only report on the IAS and MMV-IAS results for $r=-1$.
We used the anisotropic first-order discrete gradient operator as the sparsifying operator. 
\cref{fig:pMRI_coil_images} shows the recovered first coil images for $20$ lines, noise variance $\sigma^2 = 10^{-3}$, and $L=4$ coils.
The recovered first coil images using the proposed MMV-IAS (\cref{fig:pMRI_coil1_MMVIAS}) and MMV-GSBL (\cref{fig:pMRI_coil1_MMVGSBL}) algorithms are visibly more accurate than using the corresponding IAS (\cref{fig:pMRI_coil1_IAS}) and GSBL (\cref{fig:pMRI_coil1_GSBL}) algorithms.  
Note the sharper transitions between internal structures. 
Consequently, the proposed MMV-IAS/GSBL algorithm also yields a more accurate approximation to the overall image, compared to the existing IAS/GSBL algorithm, which is demonstrated in \cref{fig:pMRI_overall_images}.
To further assess the performance of the proposed joint-sparsity-promoting MMV-IAS and MMV-GSBL algorithms, \cref{fig:pMRI_errors} reports the relative error of the recovered overall image for varying numbers of lines sampled in the Fourier space. 
The proposed MMV-IAS/GSBL algorithm to jointly recover the coil images consistently yields the smallest error.

\subsection{Comparison with a sequential approach} 
\label{sub:num_seq}

Comparing the proposed MMV-IAS/GSBL algorithm solely with the traditional IAS/GSBL algorithm for separate recovery of parameter vectors may not be entirely equitable.  
The MMV-IAS/GSBL algorithm leverages information from all parameter vectors to reconstruct each individually. 
In contrast, the traditional IAS/GSBL does not facilitate information sharing across different parameter vectors. 
Consequently, we also compare the MMV-IAS/GSBL algorithm with a sequential variant of the IAS/GSBL algorithm. 
In the sequential IAS/GSBL algorithm, we determine the $l$th parameter vector $\mathbf{x}_l$ and its corresponding hyper-parameter vector $\boldsymbol{\theta}_l$ by approximating the MAP estimate of the $l$th posterior $\pi_{ \mathbf{X}_l, \boldsymbol{\Theta}_l }( \mathbf{x}_l, \boldsymbol{\theta}_l ) \propto \pi_{ \mathbf{Y}_{l} | \mathbf{X}_{l} }( \mathbf{y}_{l} | \mathbf{x}_{l} ) \, \pi_{\mathbf{X}_l | \boldsymbol{\Theta}_l}( \mathbf{x}_l | \boldsymbol{\theta}_l ) \, \pi_{ \boldsymbol{\Theta}_l }( \boldsymbol{\theta}_l )$---as it is done in the IAS/GSBL algorithm. 
However, unlike the traditional IAS/GSBL algorithm, the initial value for $\boldsymbol{\theta}_l$ in the corresponding block-coordinate descent method is chosen as the MAP estimate $\boldsymbol{\theta}_{l-1}^{\rm MAP}$ derived from the previously learned parameter vector. 
This approach is reminiscent of strategies employed in time-dependent problems where data are received in sequential batches. 
For the first parameter vector, the IAS/GSBL and the sequential IAS/GSBL algorithms start with the same initialization for $\boldsymbol{\theta}_1$. 
For this reason, we do not report on the first parameter vector in the subsequent numerical tests. 
However, from the second parameter vector onward, their initializations diverge. 
The sequential approach ensures that the insights obtained from the previous measurement vector $\mathbf{y}_l$ and the learned $\mathbf{x}_l,\boldsymbol{\theta}_l$ are not disregarded, and as such facilitate a reasonable comparison for the proposed MMV-IAS/GSBL algorithm.

\begin{figure}[tb]
	\centering
  	\begin{subfigure}[b]{0.33\textwidth}
		\includegraphics[width=\textwidth]{%
      		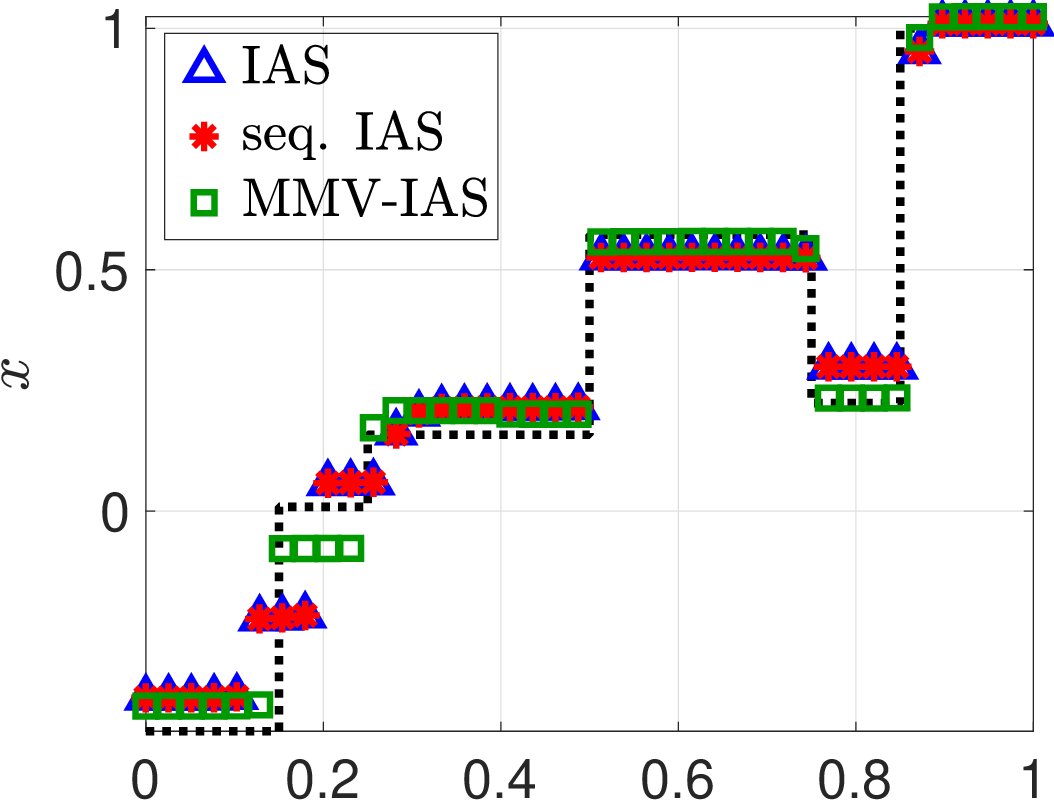} 
    	\caption{Reconstructions, $2$nd signal}
    	\label{fig:comp_signal2_IAS_rp1}
  	\end{subfigure}%
	\begin{subfigure}[b]{0.33\textwidth}
		\includegraphics[width=\textwidth]{%
      		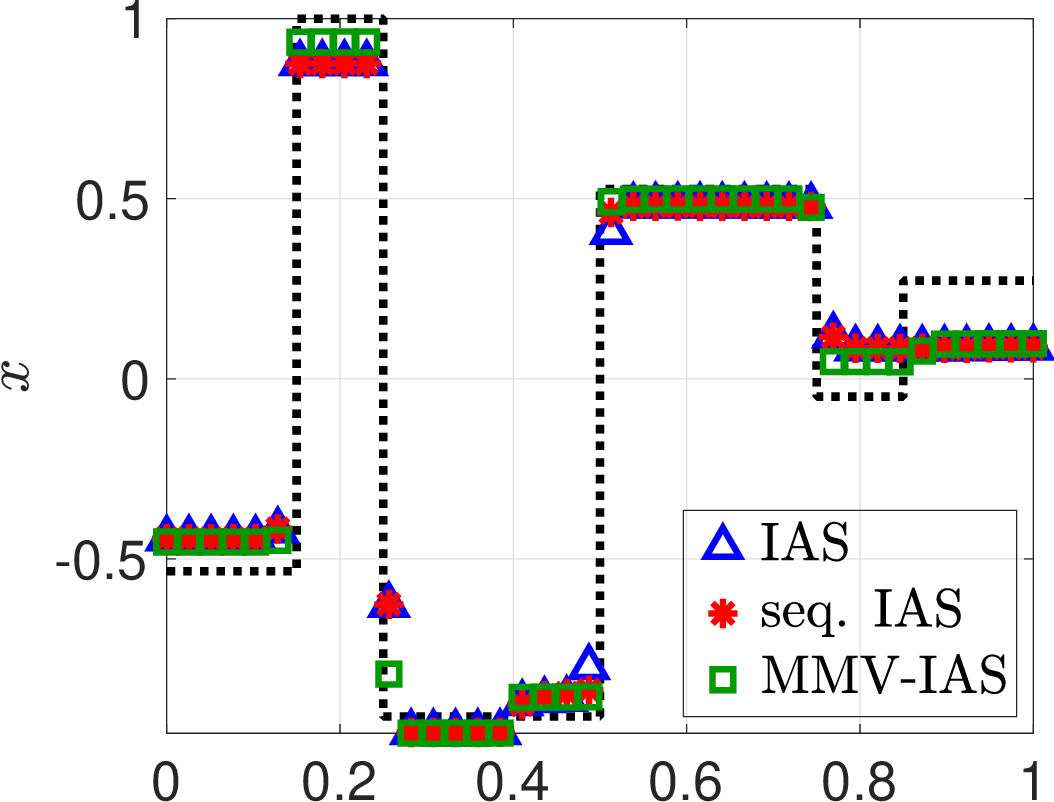} 
    	\caption{Reconstructions, $3$rd signal}
    	\label{fig:comp_signal3_IAS_rp1}
  	\end{subfigure}%
	\begin{subfigure}[b]{0.33\textwidth}
		\includegraphics[width=\textwidth]{%
      		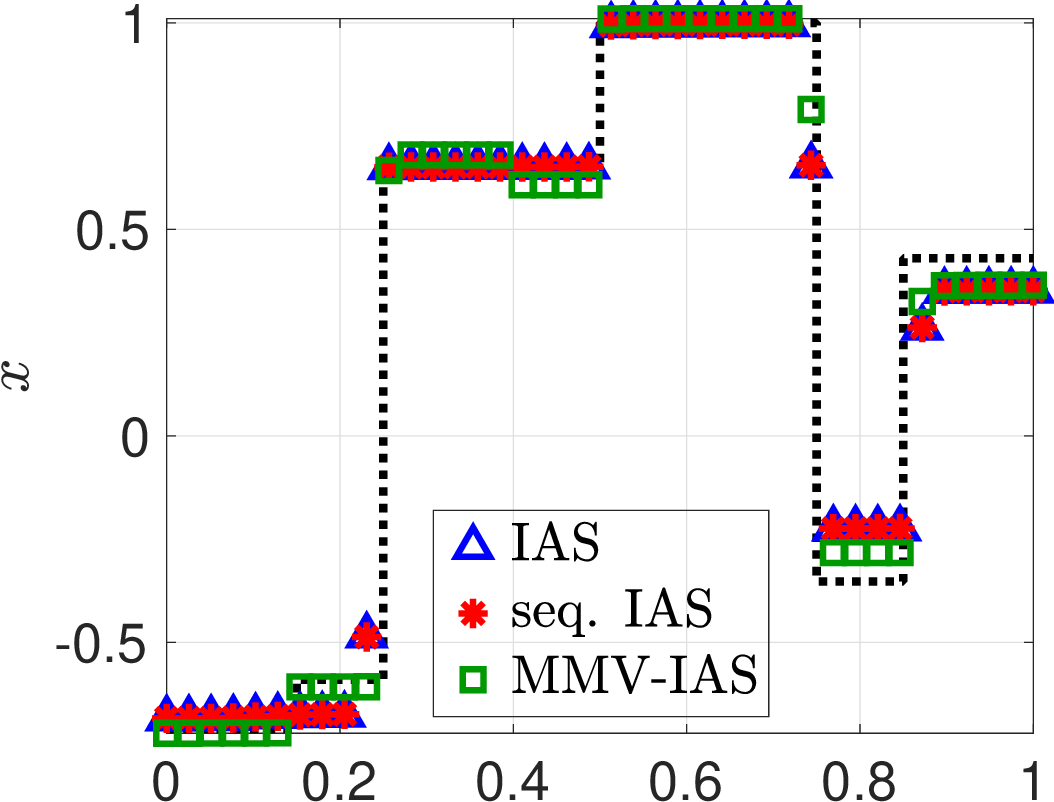} 
    	\caption{Reconstructions, $4$th signal}
    	\label{fig:comp_signal4_IAS_rp1}
  	\end{subfigure}%
	\\
  	\begin{subfigure}[b]{0.33\textwidth}
		\includegraphics[width=\textwidth]{%
      		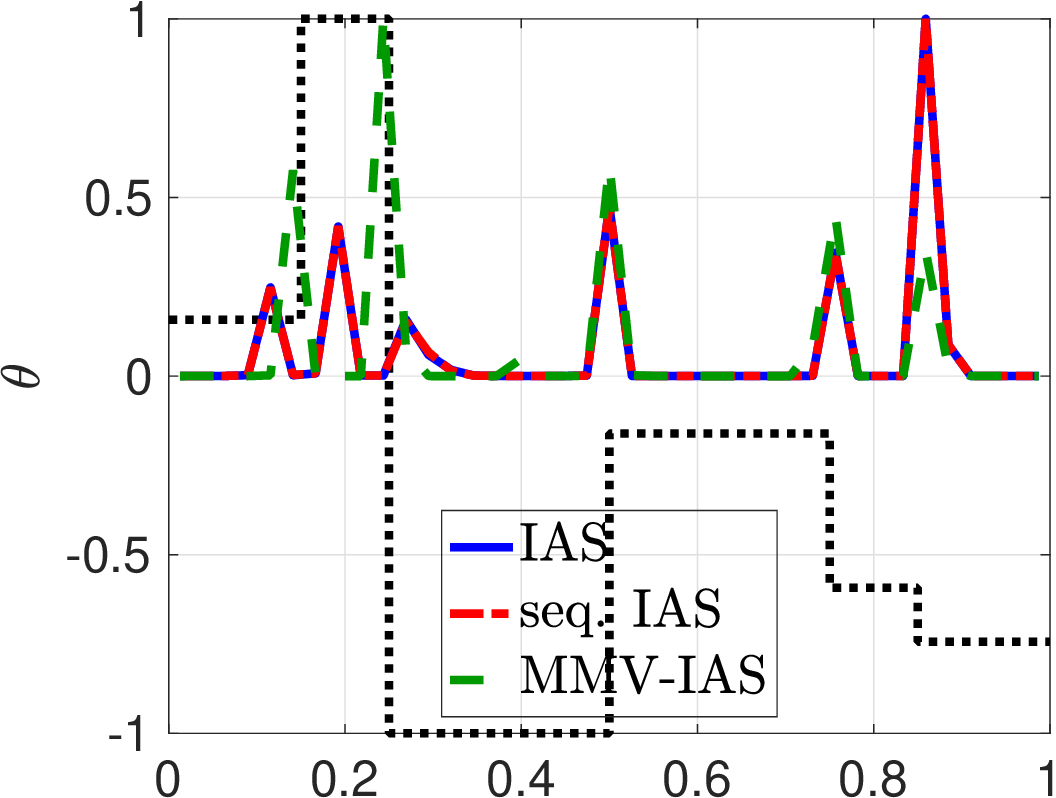} 
    	\caption{Hyper-parameters, $2$nd signal}
    	\label{fig:comp_signal2_IAS_rp1_theta}
  	\end{subfigure}%
	\begin{subfigure}[b]{0.33\textwidth}
		\includegraphics[width=\textwidth]{%
      		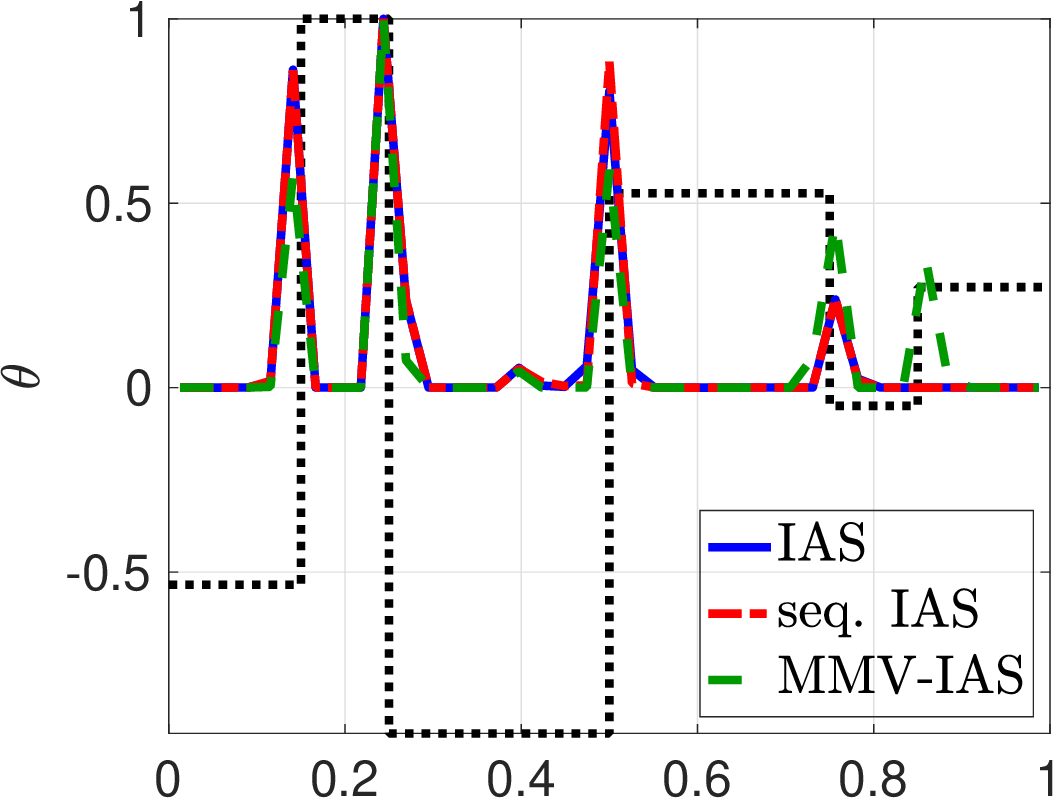} 
    	\caption{Hyper-parameters, $3$rd signal}
    	\label{fig:comp_signal3_IAS_rp1_theta}
  	\end{subfigure}%
	\begin{subfigure}[b]{0.33\textwidth}
		\includegraphics[width=\textwidth]{%
      		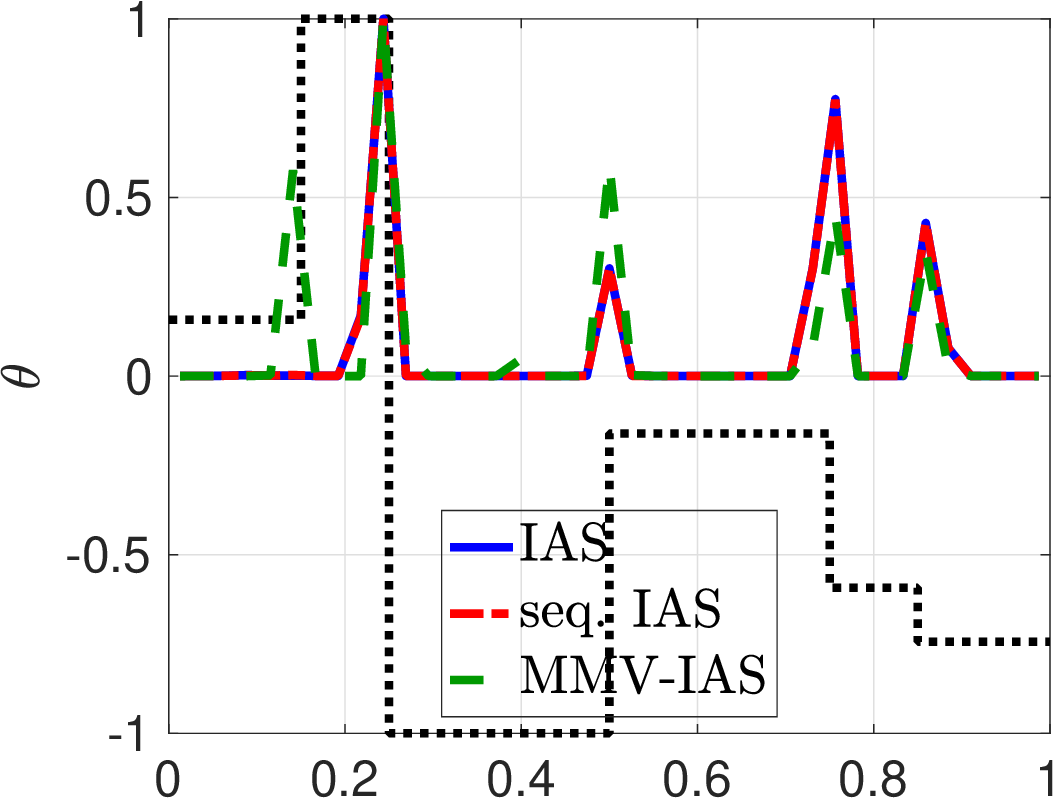} 
    	\caption{Hyper-parameters, $4$th signal}
    	\label{fig:comp_signal4_IAS_rp1_theta}
  	\end{subfigure}%
  	\caption{ 
	Reconstructions (top row) and normalized hyper-parameter estimates $\theta$ (bottom row) for the last three of four piecewise constant signals with a common edge profile. 
	We compare the IAS algorithm (blue triangles), the sequential IAS algorithm (red stars), and the MMV-IAS algorithm (green squares). 
	All methods use a generalized gamma hyper-prior with $r=1$, ensuring globally convex objective functions. 
  	}
  	\label{fig:comp_IAS_rp1}
\end{figure}

\cref{fig:comp_IAS_rp1} compares the reconstructions and normalized hyper-parameter estimates for the last three of four piecewise constant signals with a common edge profile for the IAS, sequential IAS, and MMV-IAS algorithms. 
All methods use a generalized gamma hyper-prior with $r=1$, ensuring globally convex objective functions. 
The results reveal only minor differences between the IAS and sequential IAS algorithms in this particular scenario. 
The primary reason for this similarity is the global convexity of the objective function common to both algorithms. 
Despite their differing hyper-parameter initialization, they are both theoretically guaranteed to converge to the same unique minimum as the iteration count approaches infinity. 
The slight discrepancies observed between the IAS and sequential IAS algorithms, such as those noted in \cref{fig:comp_signal3_IAS_rp1}, can be attributed to the early termination of the block-coordinate descent method. 
In comparison, when juxtaposed with the IAS and sequential IAS algorithms, the MMV-IAS algorithm demonstrates improved performance, particularly in terms of more precise edge detection and overall signal recovery.

\begin{figure}[tb]
	\centering
  	\begin{subfigure}[b]{0.33\textwidth}
		\includegraphics[width=\textwidth]{%
      		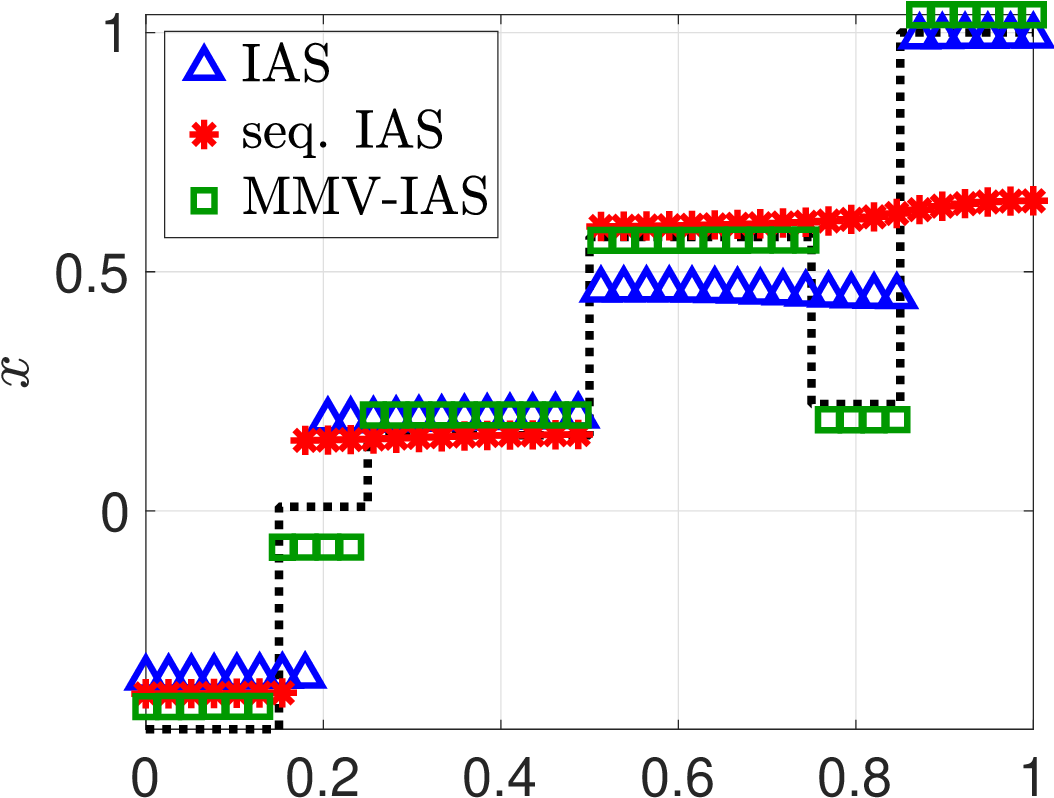} 
    	\caption{Reconstructions, $2$nd signal}
    	\label{fig:comp_signal2_IAS_rm1}
  	\end{subfigure}%
	\begin{subfigure}[b]{0.33\textwidth}
		\includegraphics[width=\textwidth]{%
      		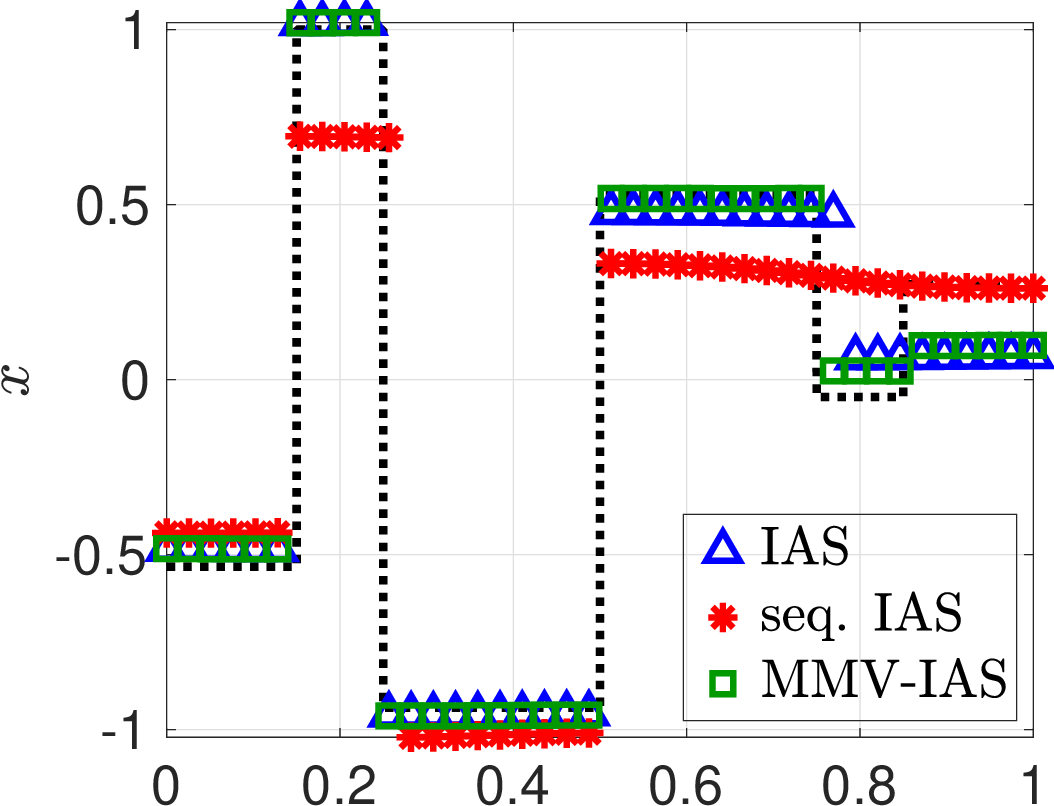} 
    	\caption{Reconstructions, $3$rd signal}
    	\label{fig:comp_signal3_IAS_rm1}
  	\end{subfigure}%
	\begin{subfigure}[b]{0.33\textwidth}
		\includegraphics[width=\textwidth]{%
      		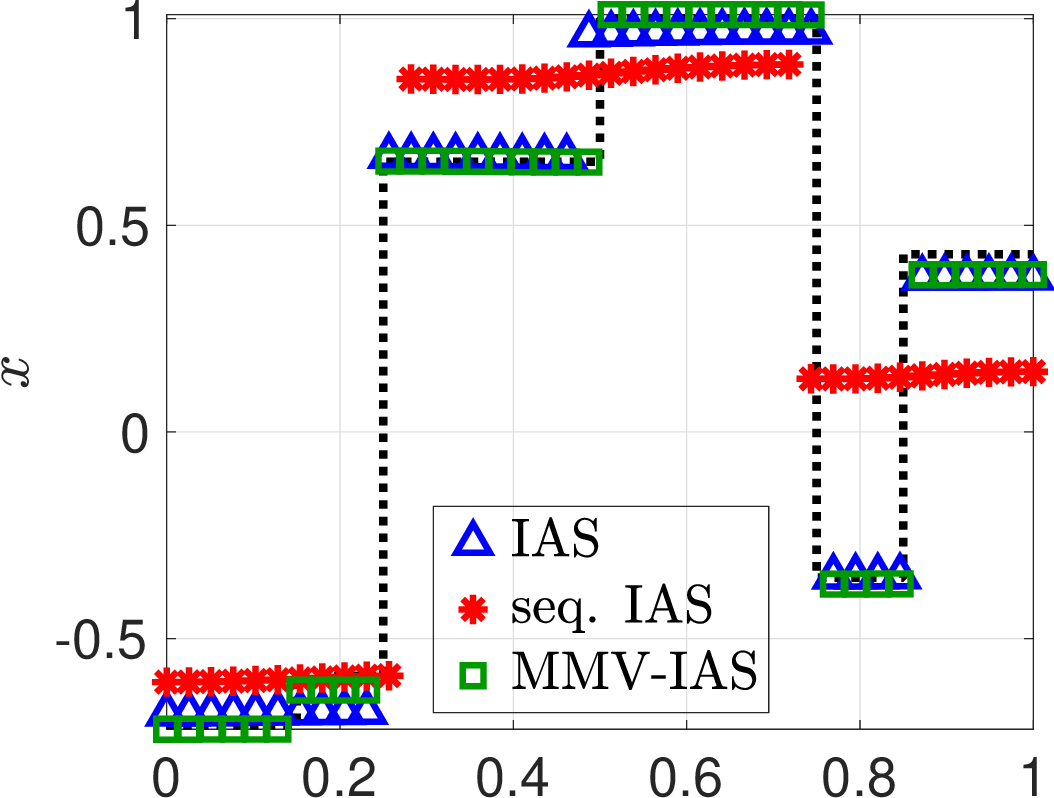} 
    	\caption{Reconstructions, $4$th signal}
    	\label{fig:comp_signal4_IAS_rm1}
  	\end{subfigure}%
	\\
  	\begin{subfigure}[b]{0.33\textwidth}
		\includegraphics[width=\textwidth]{%
      		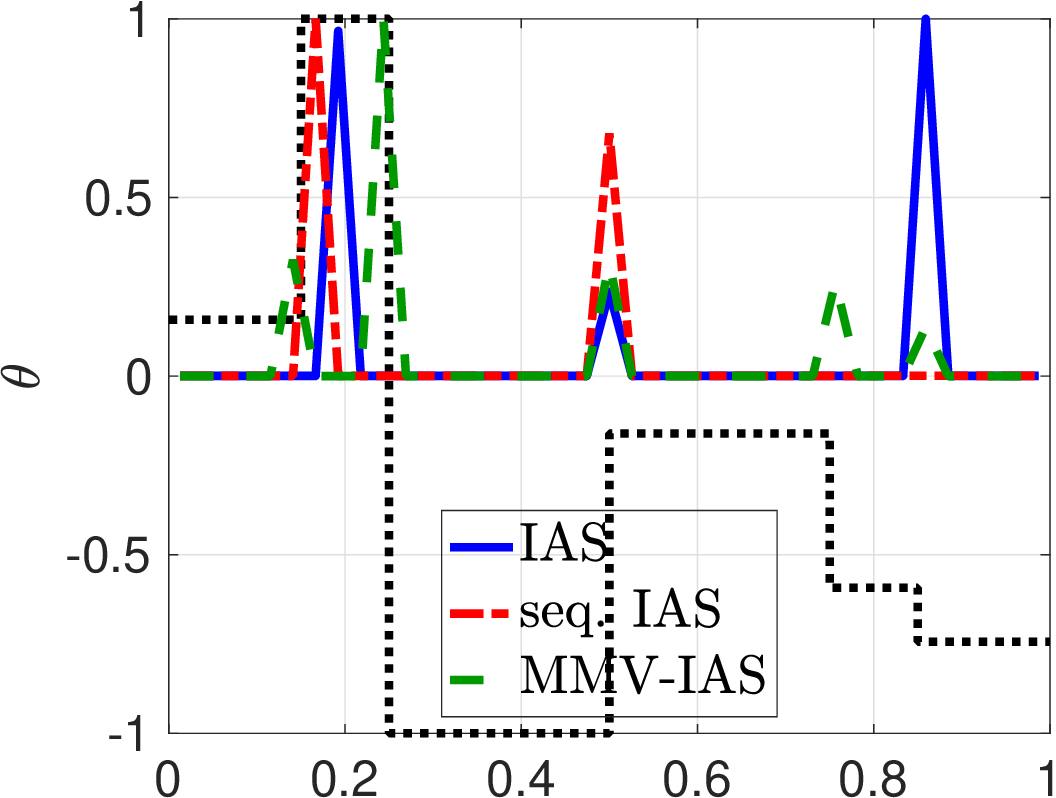} 
    	\caption{Hyper-parameters, $2$nd signal}
    	\label{fig:comp_signal2_IAS_rm1_theta}
  	\end{subfigure}%
	\begin{subfigure}[b]{0.33\textwidth}
		\includegraphics[width=\textwidth]{%
      		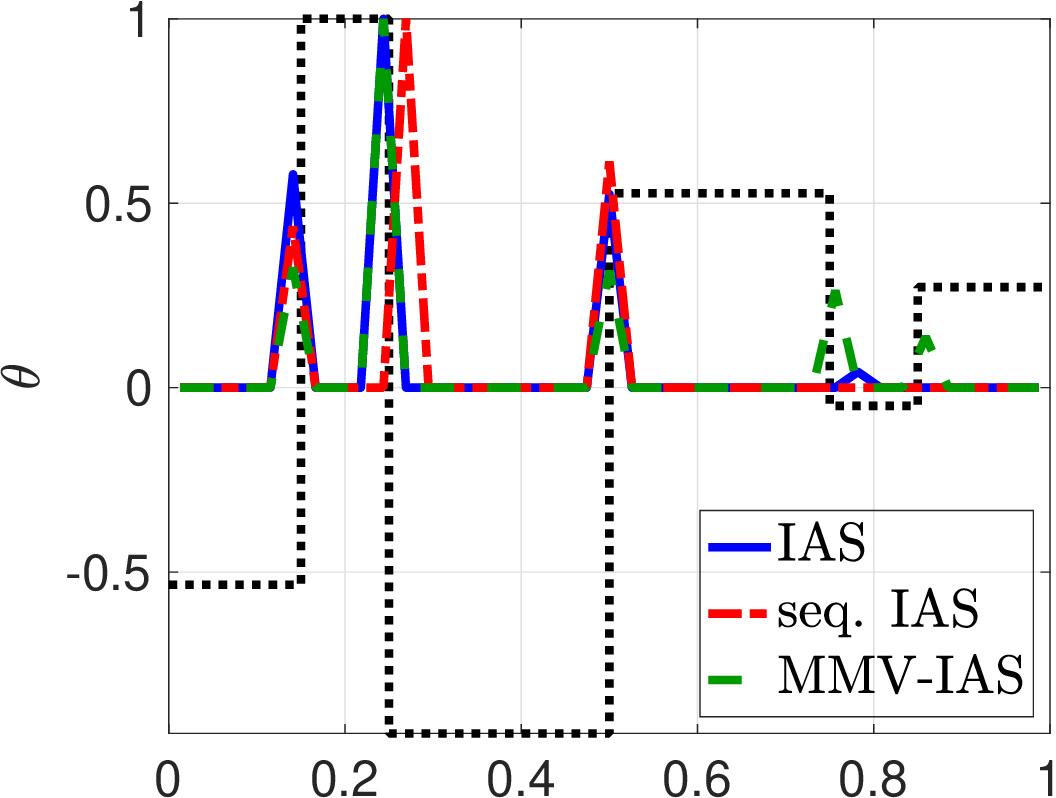} 
    	\caption{Hyper-parameters, $3$rd signal}
    	\label{fig:comp_signal3_IAS_rm1_theta}
  	\end{subfigure}%
	\begin{subfigure}[b]{0.33\textwidth}
		\includegraphics[width=\textwidth]{%
      		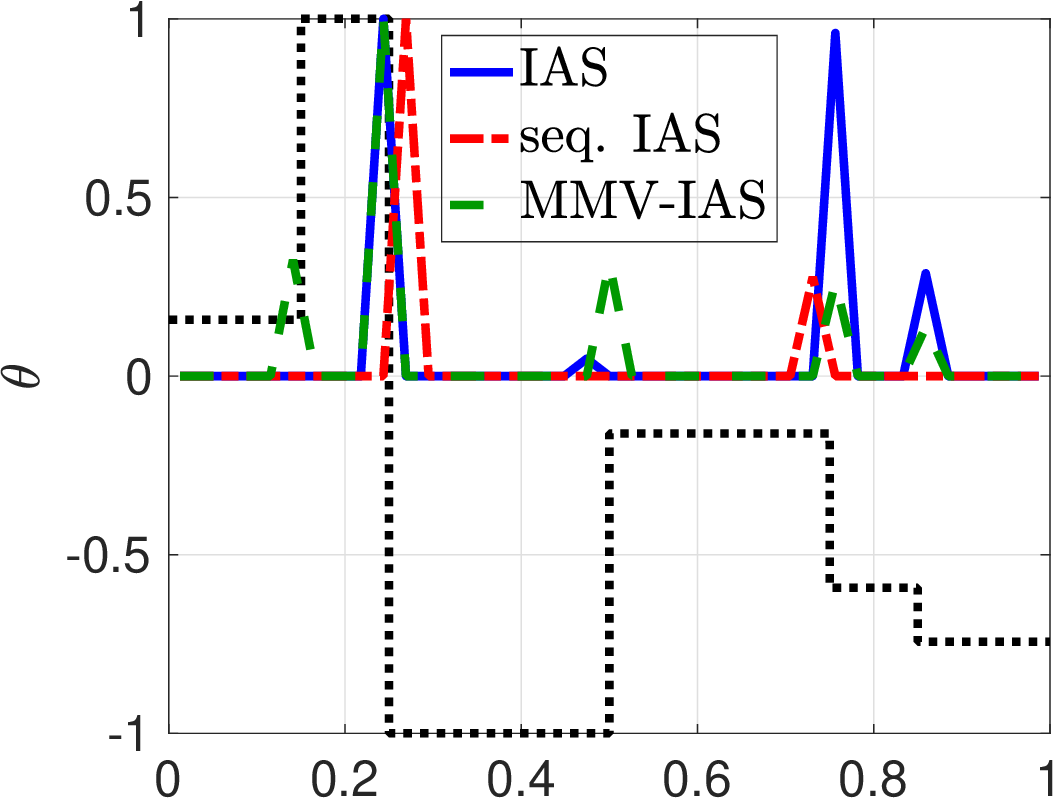} 
    	\caption{Hyper-parameters, $4$th signal}
    	\label{fig:comp_signal4_IAS_rm1_theta}
  	\end{subfigure}%
  	\caption{ 
	Reconstructions (top row) and normalized hyper-parameter estimates $\theta$ (bottom row) for the last three of four piecewise constant signals with a common edge profile. 
	We compare the IAS algorithm (blue triangles), the sequential IAS algorithm (red stars), and the MMV-IAS algorithm (green squares). 
	All methods use a generalized gamma hyper-prior with $r = -1$, leading to non-convex objective functions.
  	}
  	\label{fig:comp_IAS_rm1}
\end{figure}

\cref{fig:comp_IAS_rm1} offers a comparison akin to the previous one, but under a generalized gamma hyper-prior with $r=-1$, leading to non-convex objective functions. 
In this scenario, there are significant differences between the IAS and sequential IAS algorithms, with the latter appearing to underperform. 
For example, as shown in \cref{fig:comp_signal2_IAS_rm1}, the sequential IAS algorithm fails to detect the final edge at $0.85$, whereas the traditional IAS algorithm successfully identifies it. Similar trends are evident in \cref{fig:comp_signal3_IAS_rm1,fig:comp_signal4_IAS_rm1}. 
These results suggest that initializing the hyper-parameter vector $\boldsymbol{\theta}_l$ in the block-coordinate descent method with the previous MAP estimate $\boldsymbol{\theta}_{l-1}^{\rm MAP}$ might steer the algorithm towards a less favorable local minimum compared to a fresh initialization. 
Altering the recovery sequence of the parameter vectors could improve the sequential IAS algorithm's performance. 
Nonetheless, unless guided by the problem's context, identifying an optimal order poses a challenging and non-trivial task. 
In contrast, when compared with the IAS and sequential IAS algorithms, the MMV-IAS algorithm consistently demonstrates superior performance, especially in terms of more precise edge detection and overall signal recovery.

\begin{figure}[tb]
	\centering
  	\begin{subfigure}[b]{0.33\textwidth}
		\includegraphics[width=\textwidth]{%
      		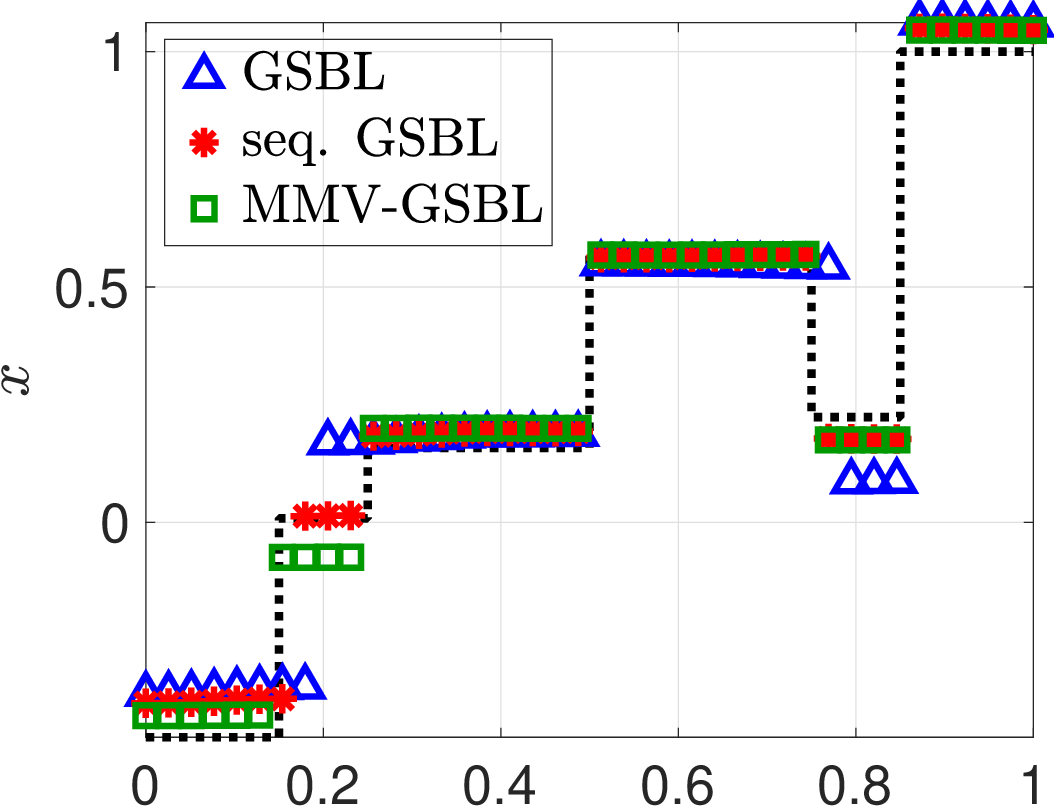} 
    	\caption{Reconstructions, $2$nd signal}
    	\label{fig:comp_signal2_GSBL}
  	\end{subfigure}%
	\begin{subfigure}[b]{0.33\textwidth}
		\includegraphics[width=\textwidth]{%
      		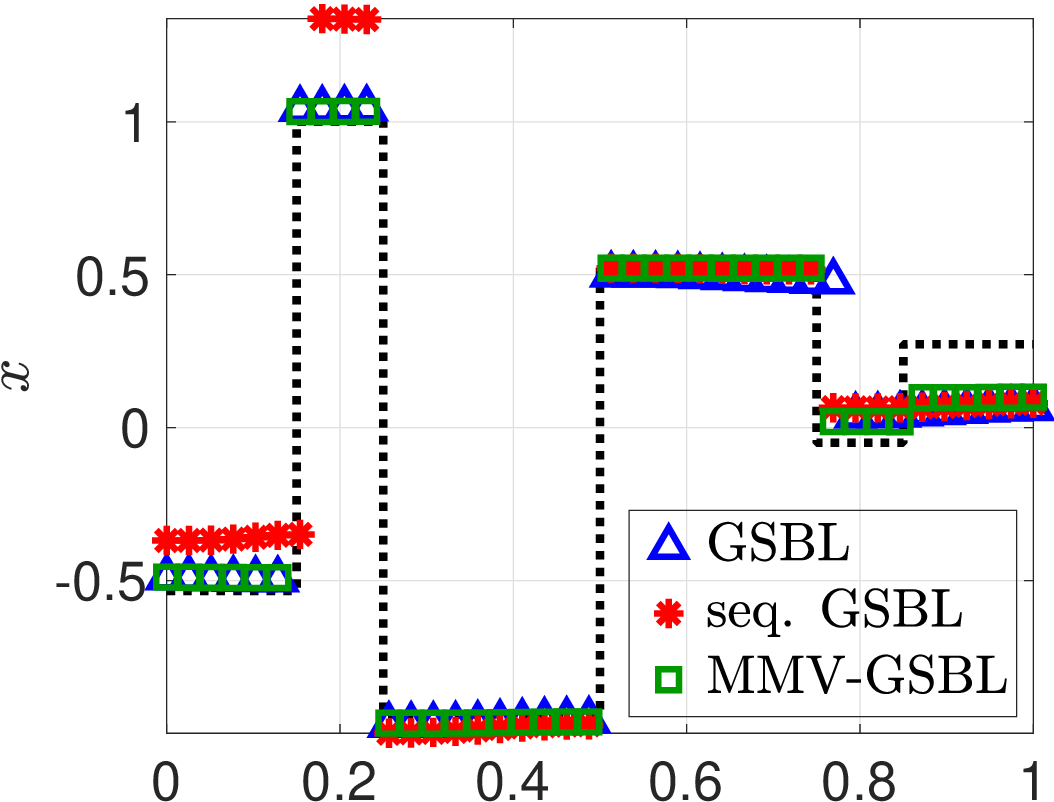} 
    	\caption{Reconstructions, $3$rd signal}
    	\label{fig:comp_signal3_GSBL}
  	\end{subfigure}%
	\begin{subfigure}[b]{0.33\textwidth}
		\includegraphics[width=\textwidth]{%
      		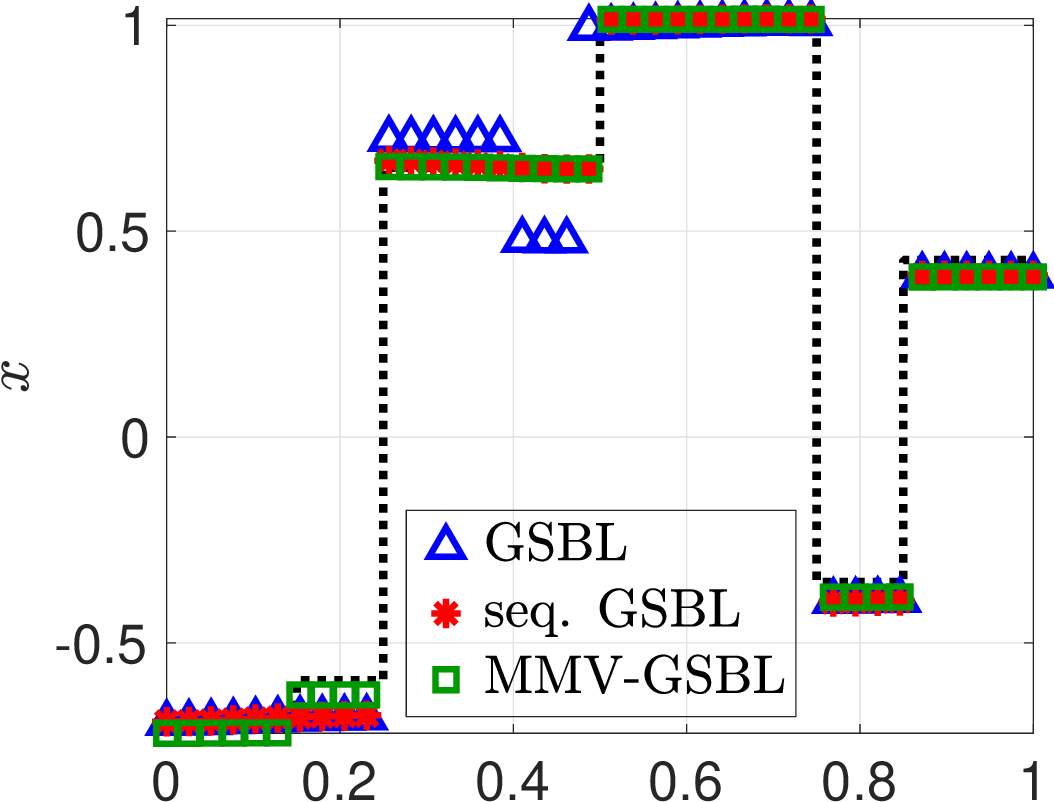} 
    	\caption{Reconstructions, $4$th signal}
    	\label{fig:comp_signal4_GSBL}
  	\end{subfigure}%
	\\
  	\begin{subfigure}[b]{0.33\textwidth}
		\includegraphics[width=\textwidth]{%
      		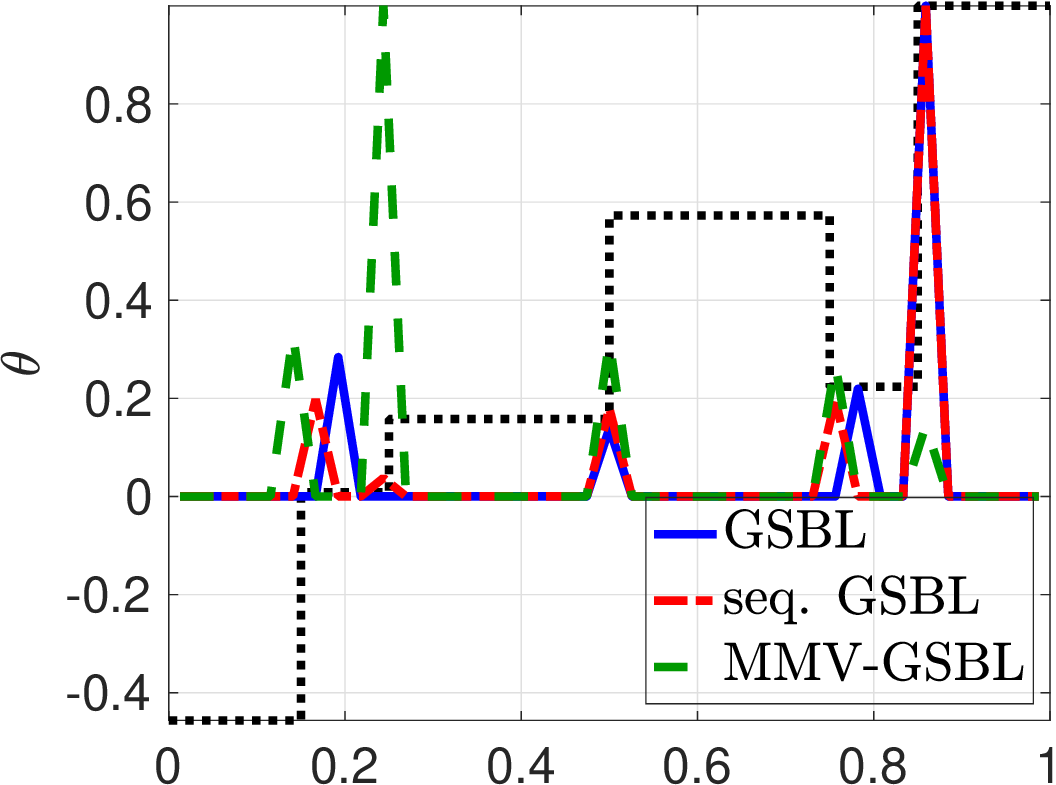} 
    	\caption{Hyper-parameters, $2$nd signal}
    	\label{fig:comp_signal2_GSBL_theta}
  	\end{subfigure}%
	\begin{subfigure}[b]{0.33\textwidth}
		\includegraphics[width=\textwidth]{%
      		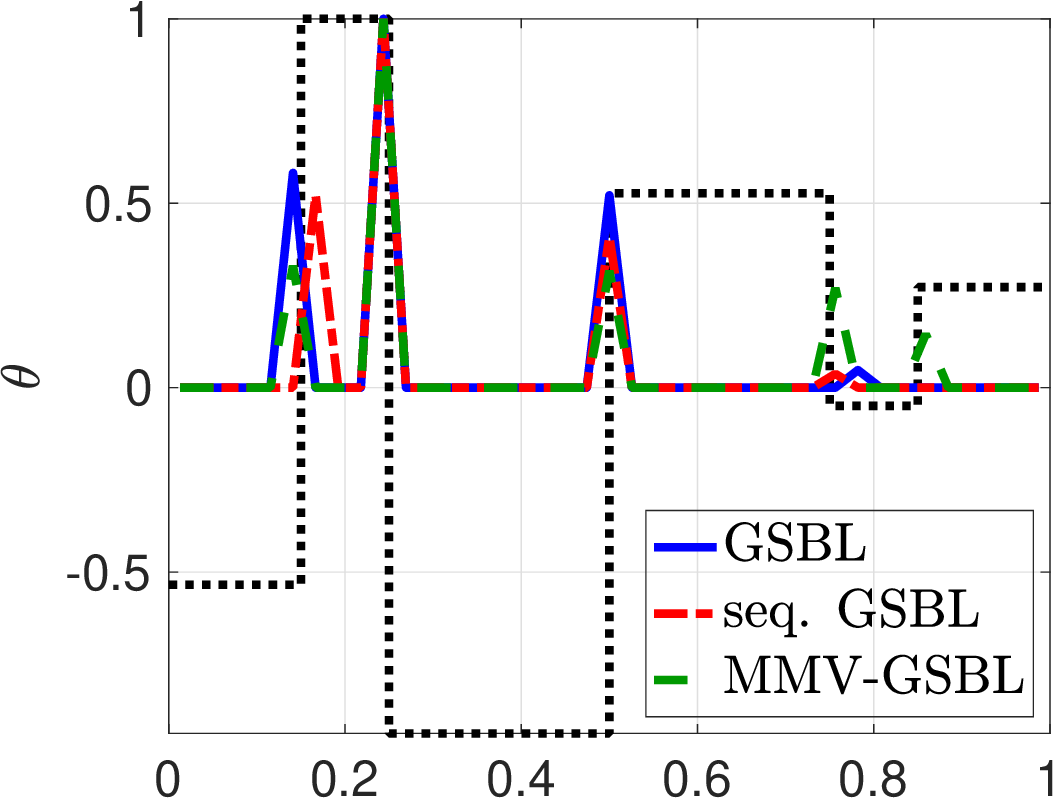} 
    	\caption{Hyper-parameters, $3$rd signal}
    	\label{fig:comp_signal3_GSBL_theta}
  	\end{subfigure}%
	\begin{subfigure}[b]{0.33\textwidth}
		\includegraphics[width=\textwidth]{%
      		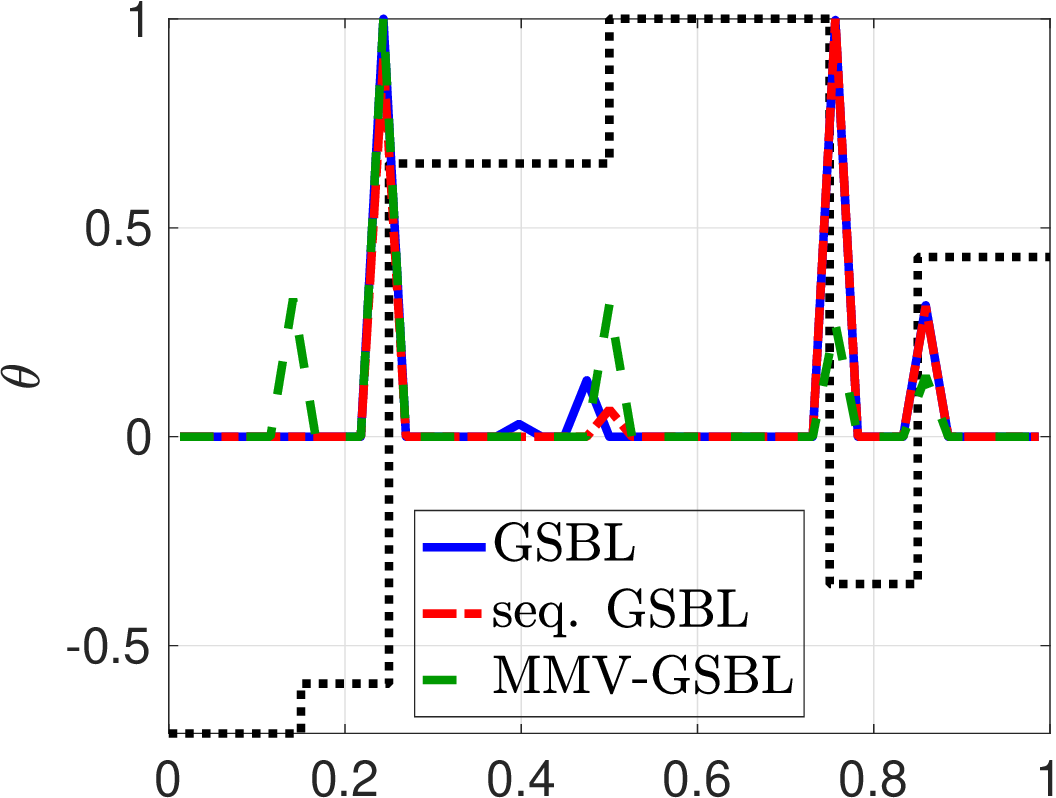} 
    	\caption{Hyper-parameters, $4$th signal}
    	\label{fig:comp_signal4_GSBL_theta}
  	\end{subfigure}%
  	\caption{ 
	Reconstructions (top row) and normalized hyper-parameter estimates $\theta$ (bottom row) for the last three of four piecewise constant signals with a common edge profile. 
	We compare the GSBL algorithm (blue triangles), the sequential GSBL algorithm (red stars), and the MMV-GSBL algorithm (green squares). 
  	}
  	\label{fig:comp_GSBL}
\end{figure}

In the concluding analysis, \cref{fig:comp_GSBL} provides a comparison analogous to the earlier ones, but this time for the GSBL, sequential GSBL, and MMV-GSBL algorithms, all of which lead to non-convex objective functions. 
In this case, moderate differences are observed between the GSBL and sequential GSBL algorithms. 
As illustrated in \cref{fig:comp_signal2_GSBL,fig:comp_signal4_GSBL}, the sequential GSBL algorithm produces improved outcomes compared to the GSBL algorithm, while in \cref{fig:comp_signal3_GSBL}, the GSBL algorithm demonstrates superior performance over its sequential counterpart. 
These findings suggest that initializing the hyper-parameter vector $\boldsymbol{\theta}_l$ in the block-coordinate descent method with the prior MAP estimate $\boldsymbol{\theta}_{l-1}^{\rm MAP}$ can variably influence the algorithm, leading it towards either a more or less favorable local minimum compared to a fresh initialization. 
In contrast, the MMV-GSBL algorithm exhibits superior overall performance. 
\section{Concluding remarks} 
\label{sec:summary} 

We presented a hierarchical Bayesian approach for inferring parameter vectors from MMVs that promotes joint sparsity. 
The method involves using separate conditionally Gaussian priors for each parameter vector and common hyper-parameters to enforce a common sparsity profile among the parameter vectors. 
Based on this joint-sparsity-promoting hierarchy, new algorithms, MMV-IAS and MMV-GSBL, were developed and demonstrated to outperform existing IAS and GSBL algorithms in several test cases.
Our findings show that incorporating joint sparsity into the current hierarchical Bayesian methodology can significantly improve its performance. 
The concept of joint-sparsity-promoting priors is flexible and can be extended beyond the conditionally Gaussian priors and (generalized) gamma hyper-priors used in the present study.

Future work will explore the potential of the proposed joint-sparsity-promoting approach by extending it to other hierarchical prior models, such as horseshoe \cite{carvalho2009handling,uribe2022horseshoe} and neural network priors \cite{neal1996priors,asim2020invertible,li2021bayesian}.
The generalization to non-linear data models, hybrid-like MAP estimation \cite{calvetti2020sparsity,si2022path}, and other inference strategies will also be addressed.
The open problem of automatic selection of the hyper-prior parameters (also noted in \cite{glaubitz2022generalized,xiao2023sequential}) will be tackled in forthcoming works. 
The promising results in our parallel MRI example suggest the proposed MMV approach can also be helpful in applications. 
Future work will consider applications such as synthetic aperture radar (SAR) and electron tomography imaging. 
In both cases, previous investigations have demonstrated that compressive sensing techniques that exploit joint sparsity of MMVS are effective in recovering point estimates, \cite{sanders2017composite,scarnati2018joint}, suggesting that the MMV-IAS or MMV-GSBL might provide an effective Bayesian approach. 
Employing various sparse transform operators $R_1,\dots,R_L$ may be appropriate in this regard.
The case of changing sparsity profiles over time will also be considered, with promising initial results already reported for sequential signaling/imaging \cite{xiao2022sequential,xiao2023sequential}.

\appendix 
\section{Proof of \cref{thm:convexity}}
\label{app:convexity_proof}

In this section, we prove \cref{thm:convexity}. 
To this end, we first present two auxiliary lemmas. 

\begin{lemma}\label{lem:derivatives} 
	Let $r \in \R\setminus\{0\}$ and $\beta, \vartheta_k > 0$ for $k=1,\dots,K$, and denote $\eta = r \beta - (L/2 + 1)$. 
	The objective function $\mathcal{G}$ in \cref{eq:G} has the following second derivatives: 
	\begin{equation}\label{eq:derivatives}
	\begin{aligned}
		\nabla_{\mathbf{x}_l} \nabla_{\mathbf{x}_m} \, \mathcal{G} 
			& = \delta_{l,m} \left( F_l^T F_l + R^T D_{\boldsymbol{\theta}}^{-1} R \right), \\ 
		\left[ \nabla_{\boldsymbol{\theta}} \nabla_{\boldsymbol{\theta}} \, \mathcal{G} \right]_{j,k} 
			& = \delta_{j,k} \left( \theta_k^{-3} \left( \sum_{l=1}^L [R \mathbf{x}_l]_k^2 \right) + \theta_k^{r-2} \left( \frac{r(r-1)}{\vartheta_k^r} \right) + \theta_k^{-2} \eta \right), \\ 
		\left[ \nabla_{\boldsymbol{\theta}} \nabla_{\mathbf{x}_l} \, \mathcal{G} \right]_{k,n}  
			& = - \theta_k^{-2} [R]_{k,n} [ R \mathbf{x}_l ]_k, 
	\end{aligned}
	\end{equation}	
	for $j,k=1,\dots,K$ and $n=1,\dots,N$. 
	Here, $\delta_{l,m}$ is the usual Kronecker delta with $\delta_{l,m} = 1$ if $l=m$ and $\delta_{l,m} = 0$ otherwise.
\end{lemma}

\begin{proof} 
	Simple calculations show that the first derivatives are 
	\begin{equation}\label{eq:derivatives_proof1}
	\begin{aligned}
		\nabla_{\mathbf{x}_l} \, \mathcal{G} 
			& = F_l^T \left( F_l \mathbf{x}_l - \mathbf{y}_l \right) + R^T D_{\boldsymbol{\theta}}^{-1} R \mathbf{x}_l, \\ 
		\left[ \nabla_{\boldsymbol{\theta}} \, \mathcal{G} \right]_k 
			& = - \theta_k^{-2} \left( \sum_{l=1}^L [ R  \mathbf{x}_l]_k^2/2 \right) + \theta_k^{r-1} \left( \frac{r}{\vartheta_k^r} \right) - \theta_k^{-1} \eta,
	\end{aligned}
	\end{equation} 
	for $l=1,\dots,L$ and $k=1,\dots,K$.
	Next, we can conclude from \cref{eq:derivatives_proof1} that 
	\begin{equation}
	\begin{aligned}
		\nabla_{\mathbf{x}_l} \nabla_{\mathbf{x}_m} \, \mathcal{G} 
			& = \delta_{l,m} \left( F_l^T F_l + R^T D_{\boldsymbol{\theta}}^{-1} R \right), \\ 
		\left[ \nabla_{\boldsymbol{\theta}} \nabla_{\boldsymbol{\theta}} \, \mathcal{G} \right]_{j,k} 
			& = \delta_{j,k} \left( \theta_k^{-3} \left( \sum_{l=1}^L [R \mathbf{x}_l]_k^2 \right) + \theta_k^{r-2} \left( \frac{r(r-1)}{\vartheta_k^r} \right) + \theta_k^{-2} \eta \right),
	\end{aligned}
	\end{equation} 
	for $l,m = 1,\dots,L$ and $j,k=1,\dots,K$. 
	To determine the mixed derivatives $\nabla_{\boldsymbol{\theta}} \nabla_{\mathbf{x}_l} \, \mathcal{G}$, note that 
	\begin{equation} 
		\left[ R^T D_{\boldsymbol{\theta}}^{-1} R \mathbf{x}_l \right]_n 
			= \sum_{j=1}^K \left[ R^T \right]_{n,j} \left[ D_{\boldsymbol{\theta}}^{-1} R \mathbf{x}_l \right]_j 
			= \sum_{j=1}^K \left[ R \right]_{j,n} \theta_k^{-1} \left[ R \mathbf{x}_l \right]_j
	\end{equation} 
	for $n=1,\dots,N$ and $l=1,\dots,L$.
	Hence we obtain
	\begin{equation} 
	\begin{aligned}
		\left[ \nabla_{\boldsymbol{\theta}} \nabla_{\mathbf{x}_l} \, \mathcal{G} \right]_{k,n} 
			= \partial_{\theta_k} \left[ R^T D_{\boldsymbol{\theta}}^{-1} R \mathbf{x}_l \right]_n 
			= - \theta_k^{-2} [R]_{k,n} [ R \mathbf{x}_l ]_k
	\end{aligned}
	\end{equation} 
	for $k=1,\dots,K$ and $n=1,\dots,N$.
\end{proof} 

The next lemma provides a lower bound in terms of the Hessian of the objective function $G$, allowing us to investigate its convexity.

\begin{lemma}\label{lem:Hessian}
	Let $r \in \R\setminus\{0\}$ and $\beta, \vartheta_k > 0$ for $k=1,\dots,K$. 
	Moreover, let 
	\begin{equation}\label{eq:Hessian}
		H = H( \mathbf{x}_{1:L}, \boldsymbol{\theta} ) = 
		\begin{bmatrix} 
			\nabla_{\mathbf{x}_{1:L}} \nabla_{\mathbf{x}_{1:L}} \, \mathcal{G} & \nabla_{\mathbf{x}_{1:L}} \nabla_{\boldsymbol{\theta}} \, \mathcal{G} \\ 
			\nabla_{\boldsymbol{\theta}} \nabla_{\mathbf{x}_{1:L}} \, \mathcal{G} & \nabla_{\boldsymbol{\theta}} \nabla_{\boldsymbol{\theta}} \, \mathcal{G}
		\end{bmatrix}
	\end{equation}
	be the Hessian of the objective function $\mathcal{G}$ in \cref{eq:G} and let $\mathbf{u} = [\mathbf{v}_{1:L}; \mathbf{w}]$ with $\mathbf{v}_l \in \R^N$, $l=1,\dots,L$, and $\mathbf{w} \in \R^K$. 
	Then, 
	\begin{equation}\label{eq:Hessian_ineq}
		\mathbf{u}^T H \mathbf{u} 
			\geq  \sum_{k=1}^K \theta_k^{-2} w_k^2 \left( \theta_k^{r} \left( \frac{r(r-1)}{\vartheta_k^r} \right) + \eta \right), 
	\end{equation} 
	where $\eta = r \beta - (L/2 + 1)$.
\end{lemma}

\begin{proof} 
	We start by noting that   
	\begin{equation}\label{eq:uTHu} 
		\mathbf{u}^T H \mathbf{u} 
			= \sum_{l=1}^L \mathbf{v}_l^T \left( \nabla_{\mathbf{x}_l} \nabla_{\mathbf{x}_l} \, \mathcal{G} \right) \mathbf{v}_l 
				+ 2 \sum_{l=1}^L \mathbf{w}^T \left( \nabla_{\boldsymbol{\theta}} \nabla_{\mathbf{x}_l} \, \mathcal{G} \right) \mathbf{v}_l 
				+ \mathbf{w}^T \left( \nabla_{\boldsymbol{\theta}} \nabla_{\boldsymbol{\theta}} \, \mathcal{G} \right) \mathbf{w}.
	\end{equation} 
	Substituting the second derivatives \cref{eq:derivatives} from \cref{lem:derivatives} yields  
	\begin{equation}\label{eq:ABC} 
	\begin{aligned}
		\mathbf{v}_l^T \left( \nabla_{\mathbf{x}_l} \nabla_{\mathbf{x}_l} \, \mathcal{G} \right) \mathbf{v}_l 
			& = \sum_{m=1}^M \left[ F_l \mathbf{v}_l \right]_m^2 + \sum_{k=1}^K \theta_k^{-1} \left[ R \mathbf{v}_l \right]_k^2, \\ 
		\mathbf{w}^T \left( \nabla_{\boldsymbol{\theta}} \nabla_{\mathbf{x}_l} \, \mathcal{G} \right) \mathbf{v}_l  
			& = - \sum_{k=1}^K \theta_k^{-2} w_k \left[ R \mathbf{x}_l \right]_k \left[ R \mathbf{v}_l \right]_k, \\ 
		\mathbf{w}^T \left( \nabla_{\boldsymbol{\theta}} \nabla_{\boldsymbol{\theta}} \, \mathcal{G} \right) \mathbf{w} 
			& = \sum_{k=1}^K w_k^2 \left( \theta_k^{-3} \left( \sum_{l=1}^L [R \mathbf{x}_l]_k^2 \right) + \theta_k^{r-2} \left( \frac{r(r-1)}{\vartheta_k^r} \right) + \theta_k^{-2} \eta \right).
	\end{aligned} 
	\end{equation} 
	Furthermore, substituting \cref{eq:ABC} into \cref{eq:uTHu} we obtain  
	\begin{equation}\label{eq:uTHu2} 
	\begin{aligned}
		\mathbf{u}^T H \mathbf{u} 
			= & \sum_{l=1}^L \sum_{m=1}^M \left[ F_l \mathbf{v}_l \right]_m^2 \\
			& + \sum_{l=1}^L \sum_{k=1}^K \left( \theta_k^{-1} \left[ R \mathbf{v}_l \right]_k^2 - 2 \theta_k^{-2} \left[ R \mathbf{v}_l \right]_k w_k \left[ R \mathbf{x}_l \right]_k + \theta_k^{-2} w_k^2 \left[ R \mathbf{x}_l \right]_k^2 \right) \\ 
			& + \sum_{k=1}^K \theta_k^{-2} w_k^2 \left( \theta_k^{r} \left( \frac{r(r-1)}{\vartheta_k^r} \right) + \eta \right). 
	\end{aligned}
	\end{equation} 
	Note that $\theta_k^{-2} \left[ R \mathbf{v}_l \right]_k w_k \left[ R \mathbf{x}_l \right]_k + \theta_k^{-2} w_k^2 \left[ R \mathbf{x}_l \right]_k^2 = \theta_k^{-3} \left( \theta_k \left[ R \mathbf{v}_l \right]_k - w_k \left[ R \mathbf{x}_l \right]_k \right)^2$. 
	Hence, we can rewrite \cref{eq:uTHu2} as 
	\begin{equation}\label{eq:uTHu3} 
	\begin{aligned}
		\mathbf{u}^T H \mathbf{u} 
			= & \sum_{l=1}^L \sum_{m=1}^M \left[ F_l \mathbf{v}_l \right]_m^2 
			+ \sum_{l=1}^L \sum_{k=1}^K \theta_k^{-3} \left( \theta_k \left[ R \mathbf{v}_l \right]_k - w_k \left[ R \mathbf{x}_l \right]_k \right)^2 \\ 
			& + \sum_{k=1}^K \theta_k^{-2} w_k^2 \left( \theta_k^{r} \left( \frac{r(r-1)}{\vartheta_k^r} \right) + \eta \right). 
	\end{aligned}
	\end{equation} 
	Finally, note that the first two sums on the right-hand side of \cref{eq:uTHu3} are non-negative, which yields the assertion.
\end{proof} 

We are now positioned to prove \cref{thm:convexity}. 

\begin{proof}[Proof of \cref{thm:convexity}]
	Recall that $\mathcal{G}$ is convex if and only if its Hessian $H$ satisfies $\mathbf{u}^T H \mathbf{u} \geq 0$ for all $\mathbf{u} = [\mathbf{v}_{1:L}; \mathbf{w}]$. 
	Let $\mathbf{u} = [\mathbf{v}_{1:L}; \mathbf{w}]$, then \cref{lem:Hessian} implies 
	\begin{equation}\label{eq:convexity_proof1}
		\mathbf{u}^T H \mathbf{u} 
			\geq  \sum_{k=1}^K \theta_k^{-2} w_k^2 \left( \theta_k^{r} \left( \frac{r(r-1)}{\vartheta_k^r} \right) + \eta \right).
	\end{equation} 
	The right-hand side of \cref{eq:convexity_proof1} is positive if 
	\begin{equation}\label{eq:convexity_proof2}
		\theta_k^{r} \left( \frac{r(r-1)}{\vartheta_k^r} \right) > -\eta, \quad k=1,\dots,K.
	\end{equation} 
	The proof for the different cases follows by enforcing condition \cref{eq:convexity_proof2}.
\end{proof}

\section*{Acknowledgements}
This work was partially supported by  AFOSR \#F9550-22-1-0411, DOD (ONR MURI) \#N00014-20-1-2595, DOE ASCR \#DE-ACO5-000R22725, and NSF DMS \#1912685.

\bibliographystyle{siamplain}
\bibliography{literature}

\end{document}